%% file: main.tex
\documentclass[letterpaper,11pt]{article}
\usepackage[margin=1in]{geometry}  

\input{defs}

\begin{document}

\title{Statistical Complexity and Optimal Algorithms \\
for Non-linear Ridge Bandits}

\author{Nived Rajaraman, Yanjun Han, Jiantao Jiao, Kannan Ramchandran\thanks{N.~Rajaraman, J.~Jiao, and K.~Ramchandran are with the Department of Electrical Engineering and Computer Sciences, University of California, Berkeley, email: \url{nived.rajaraman@berkeley.edu}, \url{{jiantao, kannanr}@eecs.berkeley.edu}. Yanjun Han is with the Courant Institute of Mathematical Sciences and the Center for Data Science, New York University, email: \url{yanjunhan@nyu.edu}.}}

\maketitle

\begin{abstract}
We consider the sequential decision-making problem where the mean outcome is a non-linear function of the chosen action. Compared with the linear model, two curious phenomena arise in non-linear models: first, in addition to the ``learning phase'' with a standard parametric rate for estimation or regret, there is an ``burn-in period'' with a fixed cost determined by the non-linear function; second, achieving the smallest burn-in cost requires new exploration algorithms. For a special family of non-linear functions named ridge functions in the literature, we derive upper and lower bounds on the optimal burn-in cost, and in addition, on the entire learning trajectory during the burn-in period via differential equations. In particular, a two-stage algorithm that first finds a good initial action and then treats the problem as locally linear is statistically optimal. In contrast, several classical algorithms, such as UCB and algorithms relying on regression oracles, are provably suboptimal.
\end{abstract}

\tableofcontents

\input{intro.tex}

\input{lowerbound.tex}

\input{upperbound.tex}
\input{discussion.tex}

\appendix
\input{auxiliary_lemma.tex}

\input{appendix_section2.tex}
\input{appendix_section3.tex}
\input{appendix_section4.tex}
\input{suboptimality.tex}

\bibliographystyle{alpha}
\bibliography{ref.bib}
\end{document}

%% file: defs.tex
\usepackage{bbm}
\usepackage{graphicx}
\usepackage{amsmath,amssymb,amsthm,amsfonts}

\usepackage{paralist}
\usepackage{bm}
\usepackage{xspace}
\usepackage{url}
\usepackage{prettyref}
\usepackage{boxedminipage}
\usepackage{wrapfig}
\usepackage{ifthen}
\usepackage{color}
\usepackage{xspace}

\usepackage{amsmath,amsthm,amsfonts,amssymb}
\usepackage{mathtools}
\usepackage{graphicx}

\usepackage{nicefrac}

\newtheorem*{definition*}{Definition}

\DeclareMathOperator*{\argmax}{arg\,max}
\DeclareMathOperator*{\argmin}{arg\,min}
\usepackage{subcaption}

\usepackage[utf8]{inputenc}

\usepackage{xcolor}
\definecolor{expert}{HTML}{008000}
\definecolor{error}{HTML}{f96565}

\usepackage{color-edits}
\addauthor{sw}{blue}

\usepackage{thmtools}
\usepackage{thm-restate}

\usepackage{tikz}
\usetikzlibrary{arrows,calc} 
\newcommand{\tikzAngleOfLine}{\tikz@AngleOfLine}
\def\tikz@AngleOfLine(#1)(#2)#3{%
\pgfmathanglebetweenpoints{%
\pgfpointanchor{#1}{center}}{%
\pgfpointanchor{#2}{center}}
\pgfmathsetmacro{#3}{\pgfmathresult}%
}

\declaretheoremstyle[
    headfont=\normalfont\bfseries, 
    bodyfont = \normalfont\itshape]{mystyle}

\usepackage[linesnumbered,algoruled,boxed,lined,noend]{algorithm2e}

\usepackage{listings}
\usepackage{amsmath}
\usepackage{amsthm}
\usepackage{tikz}
\usepackage{caption}
\usepackage{mdwmath}
\usepackage{multirow}
\usepackage{mdwtab}
\usepackage{eqparbox}
\usepackage{multicol}
\usepackage{amsfonts}
\usepackage{tikz}
\usepackage{multirow,bigstrut,threeparttable}
\usepackage{amsthm}
\usepackage{bbm}
\usepackage{epstopdf}
\usepackage{mdwmath}
\usepackage{mdwtab}
\usepackage{eqparbox}
\usetikzlibrary{topaths,calc}
\usepackage{latexsym}
\usepackage{cite}
\usepackage{amssymb}
\usepackage{bm}
\usepackage{amssymb}
\usepackage{graphicx}
\usepackage{mathrsfs}
\usepackage{epsfig}
\usepackage{psfrag}
\usepackage{setspace}
\usepackage[
            CJKbookmarks=true,
            bookmarksnumbered=true,
            bookmarksopen=true,
            colorlinks=true,
            citecolor=red,
            linkcolor=blue,
            anchorcolor=red,
            urlcolor=blue
            ]{hyperref}
\usepackage[linesnumbered,algoruled,boxed,lined]{algorithm2e}
\usepackage{algpseudocode}
\usepackage{stfloats}

\usepackage{comment}
\usepackage{mathtools}
\usepackage{blkarray}
\usepackage{multirow,bigdelim,dcolumn,booktabs}

\usepackage{xparse}
\usepackage{tikz}
\usetikzlibrary{calc}
\usetikzlibrary{decorations.pathreplacing,matrix,positioning}

\usepackage[T1]{fontenc}
\usepackage[utf8]{inputenc}
\usepackage{mathtools}
\usepackage{blkarray, bigstrut}
\usepackage{gauss}

\newcommand*{\BraceAmplitude}{0.4em}%
\newcommand*{\VerticalOffset}{0.5ex}%
\newcommand*{\HorizontalOffset}{0.0em}%
\newcommand*{\blocktextwid}{3.0cm}%
\NewDocumentCommand{\InsertLeftBrace}{%
	O{} 
	O{\HorizontalOffset,\VerticalOffset} 
	O{\blocktextwid} 
	m   
	m   
	m   
}{%
	\begin{tikzpicture}[overlay,remember picture]
	\coordinate (Brace Top)    at ($(#4.north) + (#2)$);
	\coordinate (Brace Bottom) at ($(#5.south) + (#2)$);
	\draw [decoration={brace, amplitude=\BraceAmplitude}, decorate, thick, draw=black, #1]
	(Brace Bottom) -- (Brace Top) 
	node [pos=0.5, anchor=east, align=left, text width=#3, color=black, xshift=\BraceAmplitude] {#6};
	\end{tikzpicture}%
}%
\NewDocumentCommand{\InsertRightBrace}{%
	O{} 
	O{\HorizontalOffset,\VerticalOffset} 
	O{\blocktextwid} 
	m   
	m   
	m   
}{%
	\begin{tikzpicture}[overlay,remember picture]
	\coordinate (Brace Top)    at ($(#4.north) + (#2)$);
	\coordinate (Brace Bottom) at ($(#5.south) + (#2)$);
	\draw [decoration={brace, amplitude=\BraceAmplitude}, decorate, thick, draw=black, #1]
	(Brace Top) -- (Brace Bottom) 
	node [pos=0.5, anchor=west, align=left, text width=#3, color=black, xshift=\BraceAmplitude] {#6};
	\end{tikzpicture}%
}%
\NewDocumentCommand{\InsertTopBrace}{%
	O{} 
	O{\HorizontalOffset,\VerticalOffset} 
	O{\blocktextwid} 
	m   
	m   
	m   
}{%
	\begin{tikzpicture}[overlay,remember picture]
	\coordinate (Brace Top)    at ($(#4.west) + (#2)$);
	\coordinate (Brace Bottom) at ($(#5.east) + (#2)$);
	\draw [decoration={brace, amplitude=\BraceAmplitude}, decorate, thick, draw=black, #1]
	(Brace Top) -- (Brace Bottom) 
	node [pos=0.5, anchor=south, align=left, text width=#3, color=black, xshift=\BraceAmplitude] {#6};
	\end{tikzpicture}%
}%

\usetikzlibrary{patterns}

\definecolor{cof}{RGB}{219,144,71}
\definecolor{pur}{RGB}{186,146,162}
\definecolor{greeo}{RGB}{91,173,69}
\definecolor{greet}{RGB}{52,111,72}

\theoremstyle{plain}
\newtheorem{theorem}{Theorem}
\newtheorem{Example}{Example}
\newtheorem{lemma}{Lemma}
\newtheorem{remark}{Remark}
\newtheorem{corollary}{Corollary}
\newtheorem{definition}{Definition}

\newtheorem{assumption}{Assumption}

\def \bP {\mathbb{P}}
\def \bE {\mathbb{E}}
\def \bR {\mathbb{R}}

\def\1{\mathbbm{1}}

\usepackage{xspace}

\newcommand{\stepa}[1]{\overset{\rm (a)}{#1}}
\newcommand{\stepb}[1]{\overset{\rm (b)}{#1}}
\newcommand{\stepc}[1]{\overset{\rm (c)}{#1}}
\newcommand{\stepd}[1]{\overset{\rm (d)}{#1}}
\newcommand{\stepe}[1]{\overset{\rm (e)}{#1}}
\newcommand{\stepf}[1]{\overset{\rm (f)}{#1}}

\newcommand{\prob}[1]{\mathbb{P}\left[#1\right]}

\newcommand{\regoff}[1]{\textsf{Reg}_{\mathcal{F}}^{\text{off}} (#1)}
\newcommand{\regon}[1]{\textsf{Reg}_{\mathcal{F}}^{\text{on}} (#1)}

\definecolor{myblue}{rgb}{.8, .8, 1}
\definecolor{mathblue}{rgb}{0.2472, 0.24, 0.6} 
\definecolor{mathred}{rgb}{0.6, 0.24, 0.442893}
\definecolor{mathyellow}{rgb}{0.6, 0.547014, 0.24}

\newcommand{\calA}{{\mathcal{A}}}

\newcommand{\calF}{{\mathcal{F}}}

\newcommand{\calH}{{\mathcal{H}}}

\newcommand{\calN}{{\mathcal{N}}}
\newcommand{\calO}{{\mathcal{O}}}

\newcommand{\calX}{{\mathcal{X}}}
\newcommand{\calY}{{\mathcal{Y}}}

\newcommand{\thetaLS}{{\widehat{\theta}^{\text{\rm LS}}}}

\newcommand{\jiao}[1]{\langle{#1}\rangle}
\newcommand{\gaht}{\textsc{GoodActionHypTest}\;}
\newcommand{\iaht}{\textsc{InitialActionHypTest}\;}
\newcommand{\true}{\textsf{True}\;}
\newcommand{\false}{\textsf{False}\;}

\usepackage{cleveref}
\crefname{lemma}{Lemma}{Lemmas}
\Crefname{lemma}{Lemma}{Lemmas}
\crefname{thm}{Theorem}{Theorems}
\Crefname{thm}{Theorem}{Theorems}
\Crefname{assumption}{Assumption}{Assumptions}
\Crefname{inftheorem}{Informal Theorem}{Informal Theorems}
\crefformat{equation}{(#2#1#3)}

%% file: intro.tex
\section{Introduction}

A vast majority of statistical modeling studies data analysis in a setting where the underlying data-generating process is assumed to be stationary. In contrast, sequential data analysis assumes an iterative model of interaction, where the predictions of the learner can influence the data-generating distribution. An example of this observation model is clinical trials, which require designing causal experiments to answer questions about treatment efficacy under the presence of spurious and unobserved counterfactuals \cite{bartroff2012sequential,villar2015multi}. Sequential data analysis presents novel challenges in comparison with data analysis with i.i.d. observations. One, in particular, is the ``credit assignment problem'', where value must be assigned to different actions when the effect of only a chosen action was observed \cite{minsky1961steps,sutton1984temporal}. This is closely related to the problem of designing good ``exploration'' strategies and the necessity to choose diverse actions in the learning process \cite{auer2002using,sutton2018reinforcement}.

Another observation model involving sequential data is manipulation with object interaction, which represents one of the largest open problems in robotics \cite{billard2019trends}. Intelligently interacting with previously unseen objects in open-world environments requires generalizable perception, closed-loop vision-based control, and dexterous manipulation \cite{kober2013reinforcement,kalashnikov2018qt,zhu2019dexterous}. This requires designing good sequential decision rules that continuously collect informative data, and can deal with sparse and non-linear reward functions and continuous action spaces. 



In this paper, we study a sequential estimation problem as follows. At each time $t = 1 ,2,\cdots,T$, the learner chooses an action $a_t$ in a generic action set $\calA$, based on the observed history $\calH_{t-1} = \{(a_s, r_s)\}_{s\le t-1}$. Upon choosing $a_t$, the learner obtains a noisy observation of $f_{\theta^*}(a_t)$, denoted as $r_t = f_{\theta^\star}(a_t) + z_t$, where $\{f_\theta: \calA \to \mathbb{R}\}_{\theta\in \Theta}$ is a given function class, and the noise $z_t$ follows a standard normal distribution. Here $\theta^\star \in \Theta \subseteq \mathbb{R}^d$ is an unknown parameter fixed across time, and the learner's target is to estimate the parameter $\theta^\star$ in the high dimensional regime where $d$ could be comparable to $T$. Here the learner needs to both design the sequential experiment (i.e. actions $a_1,\cdots,a_T$) adapted to the history $\{\calH_{t-1}\}_{t=1}^T$, and output a final estimator $\widehat{\theta}_T = \widehat{\theta}_T(\calH_T)$ which is close to $\theta^\star$. 

In the bandit literature, the observation $r_t$ is often interpreted as the \emph{reward} obtained for picking the action $a_t$. In addition to estimating parameter $\theta^\star$, another common target of the learner is to maximize the expected cumulative reward $\bE[ \sum_{t=1}^T r_t]$, or equivalently, to minimize the \emph{regret} defined as
\begin{align*}
    \mathfrak{R}_T(\Theta,\calA) = \bE_{\theta^\star}\left[T\cdot \max_{a^\star\in \calA} f_{\theta^\star}(a) - \sum_{t=1}^T f_{\theta^\star}(a_t)\right]. 
\end{align*}
Compared with the estimation problem, the regret minimization problem essentially requires that every action $a_t$ is close to the maximizer of $f_{\theta^\star}(\cdot)$. 

Throughout this paper, we are interested in both the estimation and regret minimization problems for the class of \emph{ridge functions} \cite{logan1975optimal}. More specifically, we assume that: 
\begin{enumerate}
    \item The parameter set $\Theta = \mathbb{S}^{d-1} = \{\theta \in \mathbb{R}^{d}: \|\theta\|_2=1 \}$ is the unit sphere in $\mathbb{R}^d$; 
    \item The action set $\calA = \mathbb{B}^{d} = \{ a\in \mathbb{R}^d: \|a\|_2 \le 1\}$ is the unit ball in $\mathbb{R}^d$; 
    \item The mean reward is given by $f_{\theta^\star}(a) = f(\jiao{\theta^\star, a})$, where $f: [-1,1]\to [-1,1]$ is a known link function. 
\end{enumerate}
The form of ridge functions also corresponds to the single index model \cite{hardle2004nonparametric} in statistics. We will be interested in characterizing the following two complexity measures. 
\begin{definition}[Sample Complexity for Estimation]\label{defn:sample_complexity}
For a given link function $f$, dimensionality $d$, and $\varepsilon \in (0,1/2]$, the sample complexity of estimating $\theta^\star$ within accuracy $\varepsilon$ is defined as
\begin{align}
T^\star(f,d,\varepsilon) = \min\left\{ T: \inf_{\widehat{\theta}_T\in \mathbb{S}^{d-1}}\sup_{\theta^\star\in \mathbb{S}^{d-1}} \bE_{\theta^\star}[1-\jiao{\widehat{\theta}_T, \theta^\star}] \le \varepsilon \right\}, 
\end{align}
where the infimum is taken over all possible actions $a^T$ adapted to $\{\calH_{t-1}\}_{t=1}^T$ and all possible estimators $\widehat{\theta}_T = \widehat{\theta}_T(\calH_T)$.
\end{definition}
\begin{definition}[Minimax Regret]\label{defn:minimax_regret}
For a given link function $f$, dimensionality $d$, and time horizon $T$, the minimax regret is defined as
\begin{align}
\mathfrak{R}_T^\star(f,d) = \inf_{a^T}\sup_{\theta^\star \in \mathbb{S}^{d-1}} \mathbb{E}_{\theta^\star} \left[ T \cdot \max_{a^\star \in \mathcal{A}} f ( \jiao{\theta^\star, a^\star}) - \sum_{t=1}^T  f ( \jiao{\theta^\star, a_t}) \right],
\end{align}
where the infimum is taken over all possible actions $a^T$ adapted to $\{\calH_{t-1}\}_{t=1}^T$. 
\end{definition}

\begin{figure}[t]
\begin{center}
	\begin{tikzpicture}[thick, scale=2]
		\draw [<->] (0,2) -- (0,0) -- (4,0); 
		\node [above] at (0,2) {minimax regret};
		\node [below] at (3.6,0) {time horizon $T$}; 
		
		\draw [dashed] (0,0) -- (2,2); 
		\node [below] at (2,1.9) {$T$};
		
		\draw [dashed, domain=0:4, smooth, variable=\x, blue] plot ({\x}, {sqrt(\x/2.5)});
		\node [above, blue] at (3.8,1.2) {$d\sqrt{T}$};
		
        \draw [red] (0,0) -- (1,1); 
		\draw [red, dashed] (1,0) -- (1,1) -- (0,1); 
		\node [below, red] at (1,0) {$d^3$}; 
		\node [left, red] at (0,1) {$d^3$};
		
		\draw[domain=2.5:4, smooth, variable=\x, red] plot ({\x}, {sqrt(\x/2.5)});
		\draw [red, dashed] (2.5,0) -- (2.5,1); 
		\node [below, red] at (2.5,0) {$d^4$};
		
		\draw [red] (1,1) -- (2.5,1);
		
	\end{tikzpicture}
\caption{When $f(x) = x^3$ is the cubic function, the minimax regret scales as $ \min\{T, d^3 + d\sqrt{T}\}$ (ignoring constant and polylogarithmic factors).}
\label{fig:phase_transition}
\end{center}
\end{figure}
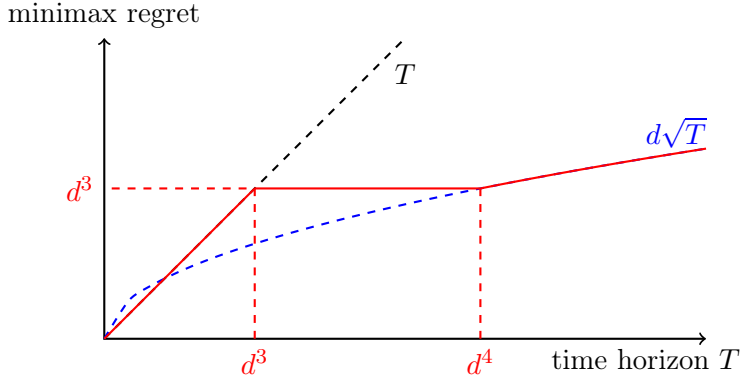

In this paper we are mainly interested in the scenario where the link function $f$ is non-linear. If $f$ is linear, i.e. $f(x)=\mathrm{id}(x)=x$, this is called the linear bandit, and it is known \cite{lattimore2020bandit,wagenmaker2022reward} that
\begin{align*}
T^\star(\text{id},d,\varepsilon)\asymp \frac{d^2}{\varepsilon}, \quad \text{ and } \quad \mathfrak{R}_T^\star(\text{id},d) \asymp \min\{ d\sqrt{T}, T\}. 
\end{align*}
Here and throughout, the symbol $\asymp$ ignores all constant and polylogarithmic factors in $(T,d,1/\varepsilon)$. However, even for many specific choices of non-linear functions $f$, much less is known about the above quantities. One of our main contributions in this paper is to identify a curious \emph{phase transition} in the learning process for non-linear link functions. Consider a toy example where $f(x)=\mathrm{cubic}(x)=x^3$. We will show that
\begin{align*}
T^\star(\mathrm{cubic},d,\varepsilon) \asymp d^3 + \frac{d^2}{\varepsilon}, \quad \text{ and } \quad \mathfrak{R}_T^\star(\mathrm{cubic},d) \asymp \min\left\{ d^3 + d\sqrt{T}, T \right\}. 
\end{align*}
A picture of the minimax regret as a function of $T$ is displayed in Figure \ref{fig:phase_transition}. We see that the minimax regret exhibits two elbows at $T\asymp d^3$ and $T\asymp d^4$: it grows linearly in $T$ until $T\asymp d^3$, stabilizes for a long time during $d^3\lesssim T\lesssim d^4$, and grows sublinearly in $T$ in the end. Similarly, the sample complexity of achieving accuracy $\varepsilon=1/2$ is already $\asymp d^3$, but improving the accuracy from $1/2$ to $\varepsilon$ only requires $\asymp d^2/\varepsilon$ additional observations. 

This curious scaling is better motivated by understanding the behavior of an optimal learner. At the beginning of the learning process, the learner has very little information about $\theta^\star$ and tries to find actions having at least a constant inner product with $\theta^\star$. Finding such actions are necessary for the learner to eventually be able to get sublinear regret. As we will discuss later, loosely speaking, finding a single such action is also \textit{sufficient} to get a sublinear regret. In other words, there is an additional \emph{burn-in period} in the learning process: 
\begin{enumerate}
    \item In the burn-in period, the learner aims to find a good \emph{initial action} $a_0$ such that $\jiao{a_0, \theta^\star}\ge \text{const}$ (say 1/2); 
    \item After the burn-in period, the learner views the problem as a linear bandit and starts learning based on the good initial action $a_0$. 
\end{enumerate}
As will be apparent later, the learning phase is relatively easy and could be solved in a similar manner to linear bandits. However, both the complexity analysis and the algorithm design in the burn-in period could be challenging and are the main focus of this paper. This burn-in period is not unique to $f$ being cubic and occurs for many choices of the link function.

Understanding the above burn-in period is important for non-linear bandits due to two reasons. First, the burn-in period results in a fixed \emph{burn-in cost} which is independent of $T$ or $\varepsilon$. This burn-in cost could be the dominating factor of our sequential problem in the high-dimensional setting - in our toy example, the burn-in cost $\Theta(d^3)$ dominates the learning cost $\Theta(d\sqrt{T})$ as long as $T=O(d^4)$, which is a reasonable range of acceptable sample sizes. Second, the long burn-in period requires new ideas of \emph{exploration} or \emph{experimental design}, which is a central problem in the current era of reinforcement learning. As a result, understanding the burn-in period provides algorithmic insights on where to explore when the learner has not gathered enough information. Similar burn-in costs were also observed in \cite{garivier2019explore,huang2021optimal}. 

The main contributions of this paper are as follows: 
\begin{itemize}
    \item We identify the existence of the burn-in period for general non-linear ridge bandit problems, and show that the two-stage algorithm which first finds a good initial action and then treats the problem as linear is near optimal for both parameter estimation and regret minimization. 
    \item We prove lower bounds for both the burn-in cost and the learning trajectory during the burn-in period, via a novel application of information-theoretic tools to establishing minimax lower bounds for sequential problems, which could be of independent interest. 
    \item We provide a new algorithm that achieves a small burn-in cost, and establish an upper bound on the learning trajectory during the burn-in period. 
    \item We show that other ideas of exploration, including the UCB and oracle-based algorithms, are provably suboptimal for non-linear ridge bandits. This is also the first failure example of UCB in a general and noisy learning environment. 
\end{itemize}

Notations: For $d\in \mathbb{N}$, let $\mathbb{B}^d$ and $\mathbb{S}^{d-1}$ denote the unit ball and sphere in $d$ dimensions, respectively. For $n\in \mathbb{N}$, let $[n]\triangleq \{1,2,\cdots,n\}$. For probability measures $P$ and $Q$ over the same probability space, let $D_{\text{KL}}(P\|Q)=\int \text{d}P\log(\text{d}P/\text{d}Q)$ and $\chi^2(P\|Q) = \int (\text{d}P)^2/\text{d}Q - 1$ be the Kullback-Leibler (KL) divergence and $\chi^2$ divergence between $P$ and $Q$, respectively. For a random vector $(X,Y)\sim P_{XY}$, let $I(X;Y)=D_{\text{KL}}(P_{XY}\|P_X\otimes P_Y)$ be the mutual information between $X$ and $Y$, where $P_X, P_Y$ are the respective marginals. For non-negative sequences $\{a_n\}$ and $\{b_n\}$, the following asymptotic notations will be used: let $a_n=O(b_n)$ denote $\limsup_{n\to\infty} a_n/b_n<\infty$, and $a_n=\widetilde{O}(b_n)$ (or $a_n\lesssim b_n$) denote $a_n = O(b_n\log^c n)$ for some $c>0$. Moreover, $a_n = \Omega(b_n)$ (resp. $a_n=\widetilde{\Omega}(b_n)$ or $a_n\gtrsim b_n$) means $b_n=O(a_n)$ (resp. $b_n\lesssim a_n$), and $a_n=\Theta(b_n)$ (resp. $a_n=\widetilde{\Theta}(b_n)$ or $a_n\asymp b_n$) means both $a_n=O(b_n)$ and $b_n=O(a_n)$ (resp. $a_n\lesssim b_n\lesssim a_n$).

\subsection{Bounds on the burn-in cost}
This section provides upper and lower bounds on the burn-in cost, which we formally define below. 

\begin{definition}[Burn-in Cost]\label{defn:burn_in_cost}
For a given link function $f$ and dimensionality $d$, the \emph{burn-in cost} is defined as
\begin{align*}
    T_{\text{\rm burn-in}}^\star(f,d) = T^\star(f,d,1/2), 
\end{align*}
where $T^\star$ is the sample complexity defined in Definition \ref{defn:sample_complexity}. 
\end{definition}

In other words, the burn-in cost is simply defined as the minimum amount of observations to achieve a constant correlation $\jiao{\widehat{\theta},\theta^\star}=\Omega(1)$. The constant $1/2$ in the definition is not essential and could be replaced by any constant bounded away from both $0$ and $1$. 





Next we specify our assumptions on the link function $f$. 

\begin{assumption}[Regularity conditions for the burn-in period]\label{assump:main}
We assume that the link function $f$ satisfies the following conditions: 
\begin{enumerate}
    \item Normalized scale: $f(0) = 0$, $f(1) = 1$, and $|f|\le 1$;
    \item Monotonicity: either (i) $f$ is increasing on $[-1,1]$; or (ii) $f$ is even and increasing on $[0,1]$. 
\end{enumerate}
\end{assumption}

We remark that Assumption \ref{assump:main} is very mild. The normalized scale is only for the scaling purpose. The monotonicity assumption ensures that $a=\theta^\star$ is a maximizer of $f(\jiao{a,\theta^\star})$, so that the task of regret minimization is aligned with the task of parameter estimation. Moreover, the additional benefit of the monotonicity condition during the burn-in period is that, by querying the noisy values of $f(\jiao{a,\theta^\star})$, the learner could decide whether or not the inner product $\jiao{a,\theta^\star}$ is improving. This turns out to be a crucial step in the algorithmic design. 


Under Assumption \ref{assump:main}, the next theorem provides an upper bound on the burn-in cost. 

\begin{theorem}[Weaker version of \Cref{thm:ub_burnincost_formal}]\label{inftheorem:integral_ub}
In a ridge bandit problem with the link function $f$ satisfying \Cref{assump:main}, for any $\kappa\in (0,1/4)$, the following upper bound holds for the burn-in cost: 
\begin{align*}
T_{\text{\rm burn-in}}^\star(f,d) \lesssim d^2\cdot \int_{1/\sqrt{d}}^{1/2} \frac{\mathrm{d}(x^2)}{\max_{1/\sqrt{d}\le y\le x}\min_{z\in [(1-\kappa)y,(1+\kappa)y]}[f'(z)]^2}, 
\end{align*}
with a hidden factor depending on $\kappa$. This upper bound is achieved by Algorithm \ref{alg:burn-in} in Section \ref{subsec:upper_bound_burnin}.
\end{theorem}


We remark that the hidden constant does not depend on $f$, so \Cref{inftheorem:integral_ub} establishes an upper bound on the burn-in cost which is \emph{pointwise} in $f$. Also, this upper bound depends on $f$ through some integral involving the derivative of $f$, suggesting that the behavior of $f$ at all points is important to determine the burn-in cost. We note that although Theorem \ref{inftheorem:integral_ub} is stated in terms of the derivative $f'$, in general we do not need to assume that $f$ is differentiable, and our general result (cf. Theorem \ref{thm:ub_burnincost_formal}) is stated in terms of finite differences of $f$. 

The integrand in Theorem \ref{inftheorem:integral_ub} looks complicated, but can be interpreted as follows: informally, suppose the link function $f$ is sufficiently smooth such that
\begin{align*}
    \min_{z\in [(1-\kappa)y, (1+\kappa)y]} [f'(z)]^2 \approx f'(y)^2 \approx d\cdot \left[ f\left(y\right) - f\left(y - \frac{1}{\sqrt{d}}\right) \right]^2, 
\end{align*}
then \Cref{inftheorem:integral_ub} says that,
\begin{align}\label{eq:burnin_informal}
T_{\text{\rm burn-in}}^\star(f,d) \lesssim d\cdot \int_{1/\sqrt{d}}^{1/2} \frac{\mathrm{d}(x^2)}{\max_{1/\sqrt{d}\le y\le x} [f(y) - f(y-1/\sqrt{d})]^2}. 
\end{align}
In the integral \eqref{eq:burnin_informal}, the variable $x$ captures the progress of the learner in terms of the inner product $\jiao{a_t, \theta^\star}$, and therefore the upper and lower limits of the integral means that the inner product grows from $\Theta(1/\sqrt{d})$ to $\Theta(1)$. In addition, the signal-to-noise ratio (SNR) of the learner is at least 
\begin{align*}
    \frac{1}{d}\max_{1/\sqrt{d}\le y\le x} [f(y) - f(y-1/\sqrt{d})]^2. 
\end{align*}
Here $f(y) - f(y-1/\sqrt{d})$ is the increment of the function value if the inner product $\jiao{a_t, \theta^\star}$ changes by $1/\sqrt{d}$, and taking the maximum over $y\le x$ corresponds to evaluating $f$ at points offering the highest SNR below the current inner product $\jiao{a_t,\theta^\star} = x$. The total burn-in cost is naturally upper bounded by the integral \eqref{eq:burnin_informal} of real-time costs using the best available SNR. This intuition will become clearer when we characterize the learning trajectory during the burn-in period in the next section. 

The next theorem shows a lower bound on the burn-in cost in terms of a different integral. 

\begin{theorem}[Weaker version of Theorem \ref{thm:lower_bound}]\label{inftheorem:integral_lb}
Suppose $f$ is even or odd. In a ridge bandit problem with the link function $f$ satisfying \Cref{assump:main}, the following lower bound holds for the burn-in cost: whenever $T_{\text{\rm burn-in}}^\star(f,d)\le T$, then
\begin{align*}
    T_{\text{\rm burn-in}}^\star(f,d) \gtrsim d\cdot \int_{\sqrt{c\log(T)/d}}^{1/2} \frac{\mathrm{d}(x^2)}{(f(x))^2}.
\end{align*}
Here $c>0$ is an absolute constant independent of $(f,d)$. 
\end{theorem}

Again, the hidden constant in Theorem \ref{inftheorem:integral_lb} also does not depend on $f$, so the above lower bound is also pointwise in $f$. In addition, the lower bound takes a similar form of the integral, highlighting again the importance of the behavior of $f$ at all points in determining the burn-in cost. However, ignoring the logarithmic factor in the lower limit, the specific form of the integrand is also different. Compared with \eqref{eq:burnin_informal}, \Cref{inftheorem:integral_lb} proves an upper bound $f(x)^2/d$ on the real-time SNR, which by the monotonicity of $f$ is no smaller than the SNR lower bound in \eqref{eq:burnin_informal}. It is an interesting question to close this gap, while we remark that even proving the above weaker SNR upper bound is highly non-trivial and possibly requires new information-theoretic ideas. We also conjecture the SNR lower bound in \eqref{eq:burnin_informal} is essentially tight, and we defer these discussions to \Cref{sec:gap}.

We also note that the assumption that $f$ is even or odd is only for the simplicity of presentation and not required in general. As will become clear in Theorem \ref{thm:lower_bound}, the general lower bound simply replaces $f(x)$ in \Cref{inftheorem:integral_lb} by $g(x):=\max\{|f(x)|, |f(-x)|\}$. 

\begin{Example}\label{example:1}
For $f(x)=|x|^p$ with $p>0$ (or $f(x)=x^p$ for $p\in \mathbb{N}$), Theorems \ref{inftheorem:integral_ub} and \ref{inftheorem:integral_lb} show that
\begin{align*}
    \max\{d, d^p\} \lesssim T_{\text{\rm burn-in}}^\star(f,d) \lesssim \max\{d^2, d^p\}. 
\end{align*}
Therefore, the upper and lower bounds match unless $p\in (1,2)$. However, this does not cause any discrepancy for the overall sample complexity $T^\star(f,d,\varepsilon)$, as the sample complexity $\Theta(d^2/\varepsilon)$ in the learning phase will dominate the burn-in cost if $p\in (1,2)$. Therefore, in many scenarios, Theorems \ref{inftheorem:integral_ub} and \ref{inftheorem:integral_lb} are sufficient to give tight results on the sample complexity within logarithmic factors. 
\end{Example}

\subsection{Learning trajectory during the burn-in period}
In addition to the burn-in cost, which is the sample complexity required to achieve a constant inner product $\jiao{\widehat{\theta}_T, \theta^\star}$, we can also provide a fine-grained analysis of the learning trajectory during the burn-in period. Specifically, we have the following definition. 

\begin{definition}[Learning trajectory]\label{defn:learning_trajectory}
For a given link function $f$, dimensionality $d$, and $\varepsilon\in (0,1/2]$, the burn-in cost for achieving $\varepsilon$ inner product is defined as
\begin{align*}
    T_{\text{\rm burn-in}}^\star(f,d,\varepsilon) = T^\star(f,d,1-\varepsilon), 
\end{align*}
where $T^\star$ is the sample complexity defined in Definition \ref{defn:sample_complexity}. We will call the function $\varepsilon\mapsto T_{\text{\rm burn-in}}^\star(f,d,\varepsilon)$ as the \emph{minimax learning trajectory} during the burn-in period. 
\end{definition}

In other words, the learning trajectory concerns the sample complexity of achieving inner products $\jiao{\widehat{\theta}_T,\theta^\star} \ge \varepsilon$, simultaneously for all $\varepsilon\in (0,1/2]$. The following theorem is a strengthening of Theorems \ref{inftheorem:integral_ub} and \ref{inftheorem:integral_lb} in terms of the learning trajectory. 

\begin{theorem}\label{thm:learning_trajectory}
Consider a ridge bandit problem with a link function $f$ satisfying Assumption \ref{assump:main}. In what follows $\kappa\in (0,1/4)$ is any fixed constant, and $c_1, c_2>0$ are absolute constants independent of $(f,d,\varepsilon)$.  
\begin{itemize}
    \item For $\varepsilon \in [c_1/\sqrt{d}, 1/2]$, the following upper bound holds on the learning trajectory: 
\begin{align*}
T_{\text{\rm burn-in}}^\star(f,d,\varepsilon) \lesssim d^2\cdot \int_{1/\sqrt{d}}^{\varepsilon} \frac{\mathrm{d}(x^2)}{\max_{1/\sqrt{d}\le y\le x}\min_{z\in [(1-\kappa)y,(1+\kappa)y]}[f'(z)]^2}. 
\end{align*}
    \item In addition assume that $f$ is even or odd. Then for $\varepsilon\in [\sqrt{c_2\log(T)/d, 1/2}]$, the following lower bound holds on the learning trajectory: if $T_{\text{\rm burn-in}}^\star(f,d,\varepsilon)\le T$, then
    \begin{align*}
    T_{\text{\rm burn-in}}^\star(f,d,\varepsilon) \gtrsim d\cdot \int_{\sqrt{c_2\log(T)/d}}^{\varepsilon} \frac{\mathrm{d}(x^2)}{(f(x))^2}.
\end{align*}
\end{itemize}
\end{theorem}

Theorem \ref{thm:learning_trajectory} shows that the integrals in Theorems \ref{inftheorem:integral_ub} and \ref{inftheorem:integral_lb} are not superfluous: when the target inner product changes from $1/2$ to $\varepsilon$, in the sample complexity we simply replace the upper limits of the integrals with $\varepsilon$ as well. Note that in the above theorem we always assume that $\varepsilon\gtrsim 1/\sqrt{d}$, as a uniformly random action $a\in \mathbb{S}^{d-1}$ achieves $\jiao{a,\theta^\star}=\Omega(1/\sqrt{d})$ with a constant probability, and thus the sample complexity for smaller $\varepsilon$ is $\Theta(1)$. This result leads to a characterization of the learning trajectory using \emph{differential equations} displayed in Figure \ref{fig:traj-ublb}. As a function of $t$, there is an algorithm where the inner product $x_t = \jiao{\theta^\star,a_t}$ can start from $\Theta(1/\sqrt{d})$ and follow the differential equation shown in the blue solid line. Moreover, for every algorithm, with high probability the start point of $x_t= \jiao{\theta^\star,a_t}$ cannot exceed $\widetilde{\Theta}(1/\sqrt{d})$, and the entire learning trajectory must lie below the differential equation shown in the red dashed line. The purple dotted line displays the performances of other exploration algorithms such as UCB and regression oracle (RO) based algorithms, showing that these algorithms make no progress until the time point $t\asymp d/[f(1/\sqrt{d})]^2$. This last part is the central theme of the next section. 
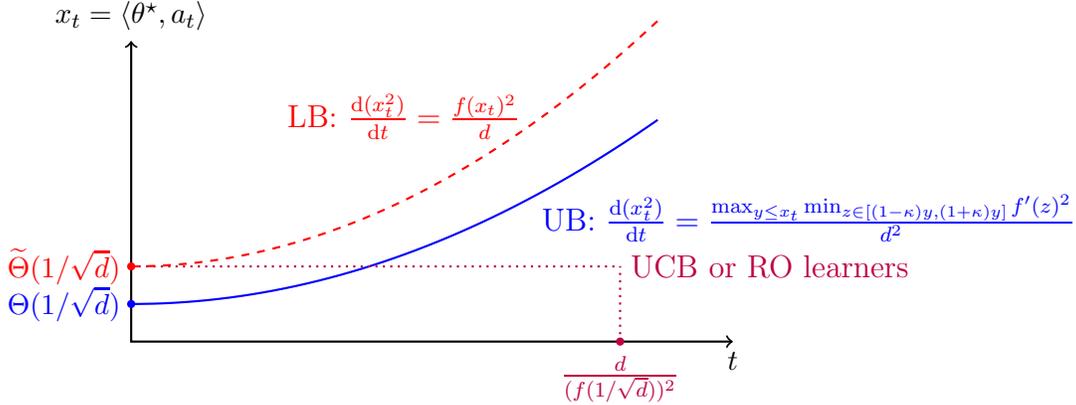
\begin{figure}[h]
\begin{center}
\begin{tikzpicture}[thick]
	\draw [<->] (0,4) -- (0,0) -- (8,0);
	\node [below] at (8,0) {$t$}; 
	\node [above] at (0,4) {$x_t = \jiao{\theta^\star, a_t}$};
	
	\filldraw [blue] (0,0.5) circle (0.04cm); 
	\node [left, blue] at (0,0.5) {$\Theta(1/\sqrt{d})$};
	
	\draw[domain=0:7, smooth, variable=\x, blue] plot ({\x}, {0.5 + \x * \x / 20});
	\node [below, blue] at (9,2.1) {\large UB: $ \frac{\mathrm{d}(x_t^2)}{\mathrm{d}t} = \frac{\max_{y\le x_t}\min_{z\in [(1-\kappa)y,(1+\kappa)y]} f'(z)^2}{d^2} $};
	
	\filldraw [red] (0,1) circle (0.04cm);
	\node [left, red] at (0,1) {$\widetilde{\Theta}(1/\sqrt{d})$};
        \draw [dotted, purple] (0,1) -- (6.5,1);
        \node [right, purple] at (6.5,1) {\large UCB or RO learners};

	
	\draw[domain=0:7, smooth, variable=\x, dashed, red] plot ({\x}, {1 + \x * \x / 15});
	\node [left, red] at (5.3,3) {\large LB: $ \frac{\mathrm{d}(x_t^2)}{\mathrm{d}t} = \frac{f(x_t)^2}{d} $};

        \filldraw [purple] (6.5, 0) circle (0.04cm); 
        \node [purple] at (6.5, -0.5) {$\frac{d}{(f(1/\sqrt{d}))^2}$};
        \draw [dotted, purple] (6.5,0) -- (6.5, 1);
\end{tikzpicture}
\end{center}
\caption{Upper and lower bounds on the minimax learning trajectory. Here UB stands for upper bound, LB stands for lower bound, and RO stands for regression oracles.}
\label{fig:traj-ublb}
\end{figure}

\subsection{Suboptimality of existing exploration algorithms}
As we discussed in the introduction, learning in the burn-in period is essentially \emph{exploration}, where the learner has not found a good action but aims to do so. In the literature of sequential decision making or bandits, several exploration ideas have been proposed and shown to work well for many problems. In this section we review two well-known exploration algorithms, i.e. algorithms based on upper confidence bounds (UCB) or regression oracles, and show that they can be strictly suboptimal for general ridge bandits. 

\subsubsection{Eluder-UCB} The UCB adopts a classical idea of ``optimism in the face of uncertainty'', i.e. the algorithm maintains for each action an optimistic upper bound on its reward, and then chooses the action with the largest optimistic upper bound. The core of the UCB algorithm is the construction of the upper confidence bound, and the Eluder-UCB algorithm \cite{russo2013eluder} proposes a general way to do so. 

In the Eluder-UCB algorithm specialized to ridge bandits, at each time $t$ the learner computes the least squares estimate of $\theta^\star$ based on the past history: 
\begin{align*}
\widehat{\theta}_t^{\text{LS}} := \argmin_{\theta\in \mathbb{S}^{d-1}} \sum_{s<t} \left(r_s - f(\jiao{\theta, a_s})\right)^2. 
\end{align*}
Then using standard theory of least squares, one can show that the true parameter $\theta^\star$ belongs to the following confidence set $\mathbb{C}_t$ with high probability: 
\begin{align}\label{eq:eluder_confidence_set}
    \mathbb{C}_t = \left\{ \theta\in \mathbb{S}^{d-1} : \sum_{s<t} \left( f (\langle a_s , \theta \rangle) - f (\langle a_s, \widehat{\theta}^{\text{LS}}_t \rangle) \right)^2 \le \textbf{Est}_{t} \right\}, 
\end{align}
where $\textbf{Est}_{t} \asymp d$ is an upper bound on the estimation error and known to the learner. Conditioned on the high probability event that $\theta^\star \in \mathbb{C}_t$, the quantity $\max_{\theta\in \mathbb{C}_t}f(\jiao{a,\theta})$ is an upper bound of $f(\jiao{a,\theta^\star})$ for every action $a$, and the Eluder-UCB algorithm chooses the action
\begin{align}\label{eq:eluderucb}
    a_t \in \argmax_{a \in \mathcal{A}} \max_{\theta \in \mathbb{C}_t} f (\langle a, \theta \rangle). \tag{El-UCB}
\end{align}
If there are ties, they can be broken in an arbitrary manner. 

The next theorem presents a lower bound on the burn-in cost for the Eluder-UCB algorithm. 

\begin{theorem} \label{thm:ucb-lower-bound}
For every Lipschitz link function $f$ satisfying Assumption \ref{assump:main}, there exists a tie-breaking rule for \cref{eq:eluderucb} such that for the Eluder-UCB algorithm, the following lower bound holds for its sample complexity $T_{\text{\rm UCB}}^\star$ of achieving inner product at least $\varepsilon$: whenever $T_{\text{\rm UCB}}^\star\le T$ and $\varepsilon\ge \sqrt{c\log(T)/d}$, it holds that
\begin{align*}
    T_{\text{\rm UCB}}^\star \gtrsim \frac{d}{g(\sqrt{c\log (T)/d})^2}. 
\end{align*}
Here $c>0$ is an absolute constant independent of $(f,d,\varepsilon)$, and $g(x):=\max\{|f(x)|,|f(-x)|\}$. 
\end{theorem}

Compared with Theorem \ref{inftheorem:integral_lb}, the lower bound for the Eluder-UCB algorithm only depends on the function value of $f$ at a single point $\widetilde{\Theta}(1/\sqrt{d})$, even for achieving an inner product $\widetilde{O}(1/\sqrt{d})$. Since $f$ is monotone (cf. Assumption \ref{assump:main}), the lower bound in Theorem \ref{thm:ucb-lower-bound} is always no smaller than the minimax lower bound in Theorem \ref{inftheorem:integral_lb}, and this gap could be arbitrarily large for carefully chosen $f$. Note that the lower bound in Theorem \ref{thm:ucb-lower-bound} is again pointwise in $f$, this means that the suboptimality of the Eluder-UCB algorithm in ridge bandits is in fact general. 

\subsubsection{Regression oracle based algorithms} Algorithms based on regression oracles follow a different idea: instead of observing the noisy observation $r_t = f(\jiao{\theta^\star, a_t})+z_t$, suppose the learner receives an estimate $\widehat{\theta}_t$ from an oracle treated as a black box. There are two types of such oracles: 
\begin{itemize}
    \item Online regression oracle: the oracle outputs $\widehat{\theta}_t$ at the beginning of time $t$ which satisfies
    \begin{align}\label{eq:online_oracle}
    \sum_{s\le t} \left( f(\jiao{\theta^\star,a_s}) - f(\jiao{\widehat{\theta}_s, a_s}) \right)^2 \le \textbf{Est}^{\textbf{On}}_t
    \end{align}
    with high probability, where $\textbf{Est}^{\textbf{On}}_t \asymp d$ is a known quantity. 
    \item Offline regression oracle: the oracle outputs $\widehat{\theta}_t$ at the end of time $t$ which satisfies
    \begin{align}\label{eq:offline_oracle}
    \sum_{s\le t} \left( f(\jiao{\theta^\star,a_s}) - f(\jiao{\widehat{\theta}_t, a_s}) \right)^2 \le \textbf{Est}^{\textbf{Off}}_t
    \end{align}
    with high probability, where $\textbf{Est}^{\textbf{Off}}_t \asymp d$ is a known quantity. 
\end{itemize}
Under the oracle model, instead of observing $(a_1,r_1,a_2,r_2,\cdots)$, the learner only observes $(a_1,\widehat{\theta}_1,a_2,\widehat{\theta}_2,\cdots)$, where the learner has no control over $\{\widehat{\theta}_t\}$ except for the error bound \eqref{eq:online_oracle} or \eqref{eq:offline_oracle}. Note that the observational model can be reduced to an oracle model with the help of certain oracles $\widehat{\theta}_t=\widehat{\theta}_t(\{a_s,r_s\}_{s=1}^t)$, but the converse may not be true. Over the recent years, an interesting line of research in the bandit literature \cite{foster2020beyond,foster2020adapting,krishnamurthy2021adapting,foster2021statistical,simchi2022bypassing} 
is the development of learning algorithms under only the oracle models.

Despite the success of oracle models, we show that for ridge bandits, the oracle models could be strictly less powerful than the original observational model. In particular, \emph{any} algorithm under the oracle model could have a suboptimal performance. The exact statement is summarized in the next theorem, where we call an oracle ``proper'' if we require that $\widehat{\theta}_t\in \mathbb{S}^{d-1}$ for every $t$, and ``improper'' otherwise. 

\begin{theorem}\label{thm:RO-lower-bound}
For every Lipschitz link function $f$ satisfying Assumption \ref{assump:main}, there exists improper online regression oracles satisfying \eqref{eq:online_oracle} or proper offline regression oracles satisfying \eqref{eq:offline_oracle} such that: for any algorithm under the oracle model, its sample complexity $T_{\text{\rm RO}}^\star$ of achieving an inner product at least $\varepsilon$ satisfies that whenever $T_{\text{\rm RO}}^\star\le T$ and $\varepsilon\ge \sqrt{c\log(T)/d}$, then
\begin{align*}
    T_{\text{\rm RO}}^\star \gtrsim \frac{d}{g(\sqrt{c\log (T)/d})^2}. 
\end{align*}
Here $c>0$ is an absolute constant independent of $(f,d,\varepsilon)$, and $g(x):=\max\{|f(x)|,|f(-x)|\}$.
\end{theorem}

The lower bound in Theorem \ref{thm:RO-lower-bound} again holds for every $f$, and is the same as the lower bound in Theorem \ref{thm:ucb-lower-bound}. Therefore, for general link function $f$, every algorithm could only achieve a strictly suboptimal performance under the oracle model. Note that this result does \emph{not} rule out the possibility that some algorithm based on a particular oracle has a smaller sample complexity than Theorem \ref{thm:RO-lower-bound}; instead, Theorem \ref{thm:RO-lower-bound} only means that even if an algorithm works, its analysis cannot treat the oracle as a black box. 

\begin{Example}
Consider again the example where $f(x)=|x|^p$ with $p>0$ (or $f(x)=x^p$ for $p\in \mathbb{N}$). Theorems \ref{thm:ucb-lower-bound} and \ref{thm:RO-lower-bound} shows that the Eluder-UCB or regression oracle based algorithms can only achieve a burn-in cost $\widetilde{\Omega}(d^{p+1})$, which is strictly suboptimal compared with Example \ref{example:1} if $p>1$. In particular, if $p\ge 2$, the suboptimality gap is as large as $\widetilde{\Omega}(d)$. 
\end{Example}

\subsection{Complexity of the learning phase}
Next we proceed to understand the learning performance after a good initial action $a_0$ is found with $\jiao{\theta^\star, a_0}\ge 1/2$. To this end we need a few additional assumptions on the link function $f$. 

\begin{assumption}[Regularity conditions for the learning phase]\label{assump:learning}
The link function $f$ is differentiable and locally linear on some interval $[1-\gamma, 1]$ around $1$:
    \begin{align}\label{eq:local_linearity}
       c_f \le \min_{x\in [1-\gamma,1]}f'(x) \le \max_{x\in [1-\gamma,1]} f'(x) \le C_f, 
    \end{align}
    where $f'$ is the derivative of $f$.
\end{assumption}

The local linearity condition may appear to be a strong condition at the first sight, as it forces $f$ to come close to being linear. The crucial feature of \eqref{eq:local_linearity} is that we only require it for $x$ bounded away from zero, thus it does not help alleviate the challenge in the burn-in period. This assumption also holds for many link functions, such as $f(x)=|x|^p$ for any fixed $p>0$.

The following theorem establishes an upper bound on the sample complexity and the regret in the learning phase. It essentially states that, every ridge bandit problem becomes a linear bandit given a good initial action, provided that \Cref{assump:learning} holds. 

\begin{theorem} \label{inftheorem:learning_ub}
Suppose the link function $f$ satisfies Assumption \ref{assump:learning}, and the learner is given an action $a_0$ with $\jiao{\theta^\star, a_0}\ge 1-3\gamma/4$. Then for every $\varepsilon<\gamma$, the output $\widehat{\theta}_T$ of \Cref{alg:learning_phase} in \Cref{subsec:upper_bound_linear} satisfies $\bE[\jiao{\widehat{\theta}_T, \theta^\star}]\ge 1-\varepsilon$ with
\begin{align*}
    T = O\left(\frac{d^2}{c_f^2\varepsilon}\right). 
\end{align*}
Here the hidden constant depends only on $\gamma$. If in addition $f$ satisfies Assumption \ref{assump:main}, \Cref{alg:learning_phase} in \Cref{subsec:upper_bound_linear} over a time horizon $T$ achieves a cumulative regret
\begin{align*}
    \mathfrak{R}_T^\star(f,d) = O\left(\min\left\{\frac{C_f}{c_f }d\sqrt{T},T\right\}\right).
\end{align*}
\end{theorem}

Ignoring the constants $(\gamma, c_f, C_f)$, the sample complexity $O(d^2/\varepsilon)$ and regret $O(d\sqrt{T})$ match the counterparts for linear bandits. Combining Theorems \ref{inftheorem:integral_ub} and \ref{inftheorem:learning_ub}, we have the following characterization for the overall sample complexity and regret of general ridge bandits. 

\begin{corollary}\label{cor:overall_upper_bound}
Suppose the link function $f$ satisfies Assumptions \ref{assump:main} and \ref{assump:learning} (with $\gamma=2/3$). Then for $\varepsilon < 1/2$ and any fixed $\kappa\in (0,1/4)$,  
\begin{align*}
    T^\star(f,d,\varepsilon) &\lesssim d^2\cdot \int_{1/\sqrt{d}}^{1/2} \frac{\mathrm{d}(x^2)}{\max_{1/\sqrt{d}\le y\le x}\min_{z\in [(1-\kappa)y,(1+\kappa)y]}[f'(z)]^2} + \frac{d^2}{c_f^2\varepsilon}, \\
    \mathfrak{R}_T^\star(f,d) &\lesssim \min\left\{ d^2\cdot \int_{1/\sqrt{d}}^{1/2} \frac{\mathrm{d}(x^2)}{\max_{1/\sqrt{d}\le y\le x}\min_{z\in [(1-\kappa)y,(1+\kappa)y]}[f'(z)]^2} + \frac{C_f}{c_f}d\sqrt{T}, T\right\}.  
\end{align*}
\end{corollary}

For the lower bounds in the learning phase, we need an additional assumption on $f$. 
\begin{assumption}[Lower bound regularity condition]\label{assump:LB}
The function $f$ is $L$-Lipschitz on $[-1,1]$, i.e. $|f(x)-f(y)|\le L|x-y|$. 
\end{assumption}

Compared with Assumption \ref{assump:learning}, Assumption \ref{assump:LB} additionally requires that $f'(x)$ is upper bounded for $x$ close to zero as well. We remark that this is not a super stringent assumption, as crucially we do \emph{not} assume that $f'(x)$ is lower bounded for small $x$. As a large derivative essentially corresponds to a large SNR, under Assumption \ref{assump:LB} it could still happen that the SNR is low during the burn-in period and the problem remains difficult. In addition, Assumption \ref{assump:LB} is also necessary for the lower bound to hold, in the sense that a super small regret might be possible without this assumption. See Example \ref{example:counterexample} at the end of this section for details. 

The lower bounds on the sample complexity and regret are summarized in the following theorem. 

\begin{theorem} \label{inftheorem:learning_lb}
Suppose the link function $f$ satisfies Assumptions \ref{assump:learning} and \ref{assump:LB}. Then for every $\varepsilon<1/2$, the following minimax lower bounds hold: 
\begin{align*}
T^\star(f,d,\varepsilon) \ge \frac{cd^2}{\varepsilon}, \quad
\mathfrak{R}_T^\star(f,d) \ge c\min\left\{ d\sqrt{T}, T\right\}, 
\end{align*}
where $c>0$ is an absolute constant depending only on $(\gamma, c_f, L)$. 
\end{theorem}






Combining Theorems \ref{inftheorem:integral_lb} and \ref{inftheorem:learning_lb}, we have the following immediate corollary on the overall lower bounds for ridge bandits. 

\begin{corollary}\label{cor:overall_lower_bound}
Suppose the link function $f$ satisfies Assumptions \ref{assump:main}, \ref{assump:learning}, and \ref{assump:LB}. Then for $\varepsilon<1/2$, 
\begin{align*}
    T^\star(f,d,\varepsilon) &\gtrsim \max_{T\ge 1} \min\left\{ d\cdot \int_{\sqrt{c\log(T)/d}}^{1/2} \frac{\mathrm{d}(x^2)}{(g(x))^2}, T\right\} + \frac{d^2}{\varepsilon}, \\
    \mathfrak{R}_T^\star(f,d) &\gtrsim \min\left\{ d\cdot \int_{\sqrt{c\log(T)/d}}^{1/2} \frac{\mathrm{d}(x^2)}{(g(x))^2} + d\sqrt{T}, T\right\},  
\end{align*}
where $g(x):=\max\{|f(x)|,|f(-x)|\}$, and the hidden constants depend only on $(c_f, L)$. 
\end{corollary}

\begin{Example}\label{example:counterexample}
This example illustrates the importance of Assumption \ref{assump:LB} for the minimax lower bound. Consider an odd function whose restriction on $[0,1]$ is
\begin{align*}
f(x) = (1-\gamma) \cdot \mathbbm{1}(\varepsilon < x \le 1-\gamma) + x \cdot \mathbbm{1}(1-\gamma < x \le 1).
\end{align*}
This function satisfies both Assumptions \ref{assump:main} and \ref{assump:learning}. However, we show that when $\varepsilon$ is very small, one can achieve an $o(d\sqrt{T})$ regret in this case. The key insight is that the function is $0$ on $[0,\epsilon]$ and $\ge 1-\gamma$ on the rest of the domain. Note that for each action $a$, by playing it $\widetilde{O}(1)$ times, the learner learns with high probability whether or not $\jiao{\theta^\star, a}\le \varepsilon$, and whether or not $\jiao{\theta^\star, a}\ge -\varepsilon$. Now choosing $a\in \{\lambda e_1,\cdots,\lambda e_d\}$ and performing bisection search over $\lambda \in [0,1]$, $\widetilde{O}(d\log(1/\varepsilon))$ observations suffice to estimate every $\theta_j$ within an additive error $\varepsilon$. Committing to this estimate then leads to a regret upper bound $\widetilde{O}(d\log(1/\varepsilon)+ \varepsilon \sqrt{d}T)$, which could be much smaller than $\Theta(d\sqrt{T})$ for small $\varepsilon$. 

Note that in this example, the lower bound on the burn-in cost in Theorem \ref{inftheorem:integral_lb} is still tight. Concretely, Theorem \ref{inftheorem:integral_lb} gives a lower bound $\Omega(d)$ for the burn-in cost, which matches (up to logarithmic factors) the above upper bound $\widetilde{O}(d\log(1/\varepsilon))$. 
\end{Example}

\subsection{Related work}

\subsubsection{Sequential estimation, testing, and experimental design} Sequential decision making has a long history in the statistics literature. In sequential estimation \cite{10.1214/aos/1176343643,10.1214/aos/1176348131} or testing \cite{wald1948optimum}, in addition to designing the estimator/test, the learner also needs to decide when to stop collecting more observations. In sequential experimental design, the goal is to decide whether and which experiment to conduct given the outcomes of the past experiments \cite{bams/1183517370,10.1214/aoms/1177706205,10.2307/2284782}. Our framework falls broadly in the class of these problems. 

\subsubsection{Stochastic bandits} The stochastic bandit problem has recieved significant research effort dating back to \cite{gittins1979bandit,lai1985asymptotically}. Under the most general scenario $f_{\theta^\star}\in \calF$ without any structural assumption on $\calF$, it is well known that the minimax regret scales as $\Theta(\sqrt{|\calA| T\log |\calF|})$ for a finite action set $\calA$ \cite{auer2002nonstochastic}. Several algorithms have then been proposed to reduce the computational complexity, either with a strong classification oracle \cite{dudik2011efficient,agarwal2014taming}, or under a realizability condition (i.e. $\bE[r(a)]=f_{\theta^\star}(a)$) with regression oracles \cite{agarwal2012contextual,foster2020beyond,simchi2022bypassing}. However, the above line of work does not cover the case with a strong realizability, i.e. the function class $\calF$ has a specific structure such as ridge bandits, and this regret upper bound is vacuous for a continuous action set $\calA$.   

Specializing to ridge bandits, the most canonical example is the linear bandit, with a link function $f(x) = x$. The minimax regret here is $\widetilde{\Theta}(d\sqrt{T})$ \cite{dani2008stochastic,chu2011contextual,abbasi2011improved}, and could be achieved by either the UCB-type \cite{chu2011contextual} or information-directed sampling algorithms \cite{russo2014learning}. The same regret bound holds for ``generalized'' linear bandits where $0<c_1 \le |f'(\cdot)| \le c_2$ everywhere \cite{filippi2010parametric,russo2014learning}. There are only a few recent work beyond generalized linear bandits. For Lipschitz and concave $f$, the same regret bound $\widetilde{\Theta}(d\sqrt{T})$ holds via a duality argument without an explicit algorithm \cite{lattimore2021minimax}. For convex $f$, the special cases of $f(x) = x^2$ and $f(x)=x^p$ with $p\ge 2$ were studied in \cite{lattimore2021bandit,huang2021optimal}, where the optimal regret scales as $\widetilde{\Theta} (\sqrt{d^p T})$. Note that this is the case where the parameter set $\Theta$ is assumed to be $\mathbb{B}^d$, a setting we discuss in \Cref{sec:ball}. 

We discuss \cite{lattimore2021bandit,huang2021optimal} in greater detail as they are closest to ours. In \cite{lattimore2021bandit}, the burn-in cost is not the dominating factor in the minimax regret, but an algorithm is designed for the burn-in period and inspires ours. In \cite{huang2021optimal}, although the authors noticed a burn-in cost in the analysis, it does not appear in the final regret bound as $\Theta$ is assumed to be the unit ball $\mathbb{B}^d$ instead of the unit sphere. In our work, we identify a fundamental role of the burn-in period in ridge bandits, by providing a lower bound on the burn-in cost and a learning trajectory during the burn-in period. In addition, we identify a sphere parameter set $\Theta$ as a more fundamental model to illustrate the phase transition, with the ball assumption being the hardest problem over spheres with different radii (cf. \Cref{sec:ball}). We also remark that \cite{huang2021optimal} proposes algorithms that works beyond ridge bandits, and proves the suboptimality of a noiseless UCB algorithm in a special example; instead, we focus on a smaller but more general problem of ridge bandits, and additionally shows that the failure of UCB is general even in the noisy scenario, answering a question of \cite{lattimore2021bandit}.

\subsubsection{Complexity measures for interactive decision making} Several structural conditions have been proposed to unify existing approaches and prove achievability results for interactive decision making, such as the Eluder dimension \cite{russo2013eluder} for bandits, and various quantities \cite{jiang2017contextual,sun2019model,wang2020reinforcement,du2021bilinear,jin2021bellman} for reinforcement learning. These quantities essentially work for generalized linear models and are not necessary in general \cite{weisz2021exponential,wang2021exponential}. 

A very recent line of research tries to characterize the statistical complexity of interactive decision making, with both upper and lower bounds, based on either the decision-estimation coefficient (DEC) and its variants \cite{foster2021statistical,foster2022note,foster2022complexity,chen2022unified,foster2023tight}, or the generalized information ratio \cite{lattimore2021mirror,lattimore2022minimax}. Although these result typically lead to the right regret dependence on $T$ for general bandit problems, the dependence on $d$ could be loose in both their upper and lower bounds. For example, the DEC lower bounds are proved via a careful two-point argument, which cannot take into account the estimation complexity, a quantity depending on $d$; this quantity is indeed the last missing piece in the state-of-the-art lower bound in \cite{foster2023tight}. The DEC upper bounds are achieved under an online regression oracle model, which by \Cref{thm:RO-lower-bound} must be suboptimal in ridge bandits. Our work complements this line of research by providing an in-depth investigation of the role of estimation complexity in interactive decision making, through the special case of ridge bandits.

\subsubsection{Information-theoretic view of sequential decision making} The sequential decision making is also related to the notion of feedback channel capacity in information theory \cite{burnavsev1980sequential,tatikonda2008capacity}, where the target is to transmit $\theta^\star$ through multiple access of some noisy channel with feedback. However, the encoding scheme in ridge bandits is restricted to the given action set $\mathcal{A}$, so the feedback channel capacity may not be achievable. In fact, the Gaussian channel capacity suggests that $\Theta(d)$ channel uses suffice to provide $d$ bits information of $\theta^\star$, which by our lower bound is generally not attainable. 

However, upper bounds of mutual information $I(\theta^\star; \calH_T)$ other than the channel capacity could be sometimes useful. A typical example is the stochastic optimization literature, where the goal is to maximize a function given access to the function and/or its gradient through some noisy oracle. The work \cite{agarwal2009information,raginsky2011information} initiated the use of the mutual information to prove the oracle complexity for stochastic optimization, while the key is the reduction to hypothesis testing problems where the classical arguments of Le Cam, Assouad, and Fano could all be applied; see, e.g. \cite{jamieson2012query,shamir2013complexity}. Instead, our problem illustrates the difficulty of applying classical hypothesis testing arguments to the sequential case, and requires the understanding of the entire trajectory $t\mapsto I(\theta^\star; \calH_t)$ for a suitable notion of ``information''. 

We also note that our problem is similar to the zeroth-order stochastic optimization. A rich line of work studied the maximization of a concave function \cite{kleinberg2004nearly,flaxman2004online,agarwal2011stochastic,bubeck2016multi,bubeck2017kernel,lattimore2020improved}, while instance-dependent bounds are also developed for general Lipschitz functions \cite{hansen1991number,bouttier2020regret,bachoc2021instance}. However, in ridge bandits, the latter bounds do not exploit the specific structure and give a complexity exponential in $d$.

%% file: lowerbound.tex
\section{Minimax Lower Bounds}\label{sec:lower_bound}
In this section we prove the minimax lower bounds for general non-linear ridge bandits. We only prove the lower bound of the burn-in cost, which is the most challenging part and requires novel information-theoretic techniques to handle interactive decision making. In particular, by recursively upper bounding a proper notion of information, we are able to prove fundamental limits for the learning trajectory of any learner at every time step. The proof of the lower bounds in \Cref{inftheorem:learning_lb} is deferred to \Cref{sec:proof_lower_bound}, where the high-level argument is similar to existing lower bounds for linear bandits, but we additionally require a delicate exploration-exploitation tradeoff to make sure that the unknown parameter $\theta^\star$ lies on the unit sphere. 

The main lower bound on the learning trajectory during the burn-in period is summarized in the following theorem. 
\begin{theorem}\label{thm:lower_bound}
Suppose the link function $f$ satisfies Assumption \ref{assump:main}, and $g(x):=\max\{|f(x)|, |f(-x)|\}$. Given $c>0, \delta \in (0,1)$, let $\{\varepsilon_t\}_{t\ge 1}$ be a sequence of positive reals defined recursively as follows: 
\begin{align}\label{eq:eps_sequence} 
	\varepsilon_1 = \sqrt{\frac{c\log(1/\delta)}{d}}, \qquad \varepsilon_{t+1}^2 = \varepsilon_t^2 + \frac{c}{d}g(\varepsilon_t)^2, \quad \forall t\ge 1. 
\end{align}
There exists a universal constant $c>0$ such that for any $\delta\in (0,1)$, if $\theta^\star$ is uniform distributed on $\mathbb{S}^{d-1}$, then for the above sequence $\{\varepsilon_t\}_{t\ge 1}$ and all $t\ge 1$, any learner satisfies that 
\begin{align*}
	\bP\left( \cap_{s\le t} \left\{|\jiao{\theta^\star, a_s}| \le \varepsilon_s \right\} \right) \ge 1 - t\delta. 
\end{align*}
\end{theorem}

Note that \Cref{thm:lower_bound} provides a pointwise Bayes lower bound of the learning trajectory for \emph{every} function $f$ and \emph{every} time step $t$. In other words, the sequence $\{\varepsilon_t\}_{t\ge 1}$ determines an upper limit on the entire learning trajectory $\{\jiao{\theta^\star, a_t} \}_{t\ge 1}$ for every possible learner. In particular, the sequence $\{\varepsilon_t\}_{t\ge 1}$ is determined in a recursive manner, which is an interesting consequence of the interactive decision making environment. 

By monotonicity of $f$, it holds that $0\le \varepsilon_t \le \varepsilon$ implies $g(\varepsilon_t)\le g(\varepsilon)$, and the following corollary on the sample complexity follows directly from Theorem \ref{thm:lower_bound}. 
\begin{corollary}\label{cor:sample_complexity}
Fix $\varepsilon<1/2$. For a large enough constant $c>0$, the sample complexity of achieving $\bP(\jiao{\theta^\star, a_T} \ge \varepsilon)\ge 1-T\delta$ is at least
\begin{align}\label{eq:lower_bound_integral}
	T = \Omega\left( d\cdot \int_{\sqrt{c\log(1/\delta)/d}}^{\varepsilon} \frac{\mathrm{d}(x^2)}{g(x)^2} \right). 
\end{align}
In particular, the above sample complexity is at least
\begin{align}\label{eq:lower_bound_max}
	T = \Omega\left(d\cdot \max_{ 2\sqrt{\frac{c\log(1/\delta)}{d}} \le x \le \varepsilon}\frac{x^2}{g(x)^2} \right). 
\end{align}
\end{corollary}

Note that \eqref{eq:lower_bound_integral} proves Theorem \ref{inftheorem:integral_lb} and the lower bound part of Theorem \ref{thm:learning_trajectory}. For $f(x) = |x|^p$, both \eqref{eq:lower_bound_integral} and \eqref{eq:lower_bound_max} give a lower bound $\widetilde{\Omega}(d^{\max\{p,1\}})$ for the burn-in cost, which is tight for $p\ge 2$. For $p=1$ which corresponds to the case of linear bandit, an improved lower bound in \Cref{thm:burnin_linear_bandit} shows that a tight lower bound $\widetilde{\Omega}(d^2)$ actually holds; we defer the discussions (including the existing results for linear bandits) to \Cref{sec:gap}.

In the remainder of this section, we will provide an information-theoretic proof of \Cref{thm:lower_bound}. In \Cref{subsec:insight}, we provide the intuition behind the update of the sequence $\{\varepsilon_t\}_{t\ge 1}$ in \eqref{eq:eps_sequence}, and discuss the failure of formalizing the above intuition using the classical mutual information. Then we introduce in \Cref{subsec:chi^2_information} the notion of $\chi^2$-informativity and how it could lower bound the probability of error. In \Cref{subsec:chi^2_information_analysis} we upper bound the $\chi^2$-informativity in a recursive way and complete the proof of \Cref{thm:lower_bound}. We also remark that although one might be tempted to apply hypothesis testing based arguments to prove \Cref{thm:lower_bound}, we find it difficult to even obtain the much weaker lower bound \eqref{eq:lower_bound_max}. We refer to the discussions below \Cref{thm:nonadaptive,thm:finite_action} for some insights. 

\subsection{Information-theoretic insights}\label{subsec:insight}
In this section we provide some intuition behind the sequence $\{\varepsilon_t\}_{t\ge 1}$ in \eqref{eq:eps_sequence}.  Let $\theta^\star\sim \mathsf{Unif}(\mathbb{S}^{d-1})$, and $I_t \triangleq I(\theta^\star; \calH_t)$ be the mutual information between $\theta^\star$ and the learner's history $\calH_t = \{(a_s,r_s)\}_{s\le t}$ up to time $t$. It holds that
\begin{align}\label{eq:information_recursion}
I_{t+1} &= I(\theta^\star; \calH_{t+1}) \nonumber \\
&\stepa{=} I(\theta^\star; \calH_t) + I(\theta^\star; a_{t+1}, r_{t+1} \mid \calH_t) \nonumber \\
&\stepb{=} I_t + I(\theta^\star; r_{t+1} \mid \calH_t, a_{t+1}) \nonumber \\
&\stepc{\le} I_t + \frac{1}{2}\log\left( 1 + \bE[f(\jiao{\theta^\star, a_{t+1}})^2] \right) \nonumber \\
&\le I_t + \frac{1}{2} \bE[f(\jiao{\theta^\star, a_{t+1}})^2], 
\end{align}
where (a) follows from the chain rule of mutual information, (b) is due to the conditional independence of $\theta^\star$ and $a_{t+1}$ conditioned on $\calH_t$ that $I(\theta^\star; a_{t+1}\mid \calH_t) = 0$, and (c) is the capacity upper bound for Gaussian channels. The above inequality shows that the mutual information increment $I_{t+1} - I_t$ is upper bounded by the second moment of $f(\jiao{\theta^\star, a_{t+1}})$, which is intuitive as larger correlation $\jiao{\theta^\star, a_{t+1}}$ should lead to a larger information gain. For the lower bound purposes, we aim to show that $\jiao{\theta^\star, a_{t+1}}$ should not be too large. The only thing we know about $a_{t+1}$ is that it is constrained in information: by the data-processing inequality of mutual information, 
\begin{align*}
	I(\theta^\star; a_{t+1}) \le I(\theta^\star; \calH_t) = I_t. 
\end{align*}
Here comes our key insight behind the update \eqref{eq:eps_sequence}: $I(\theta^\star; a) \le d\varepsilon^2$ implies that $|\jiao{\theta^\star, a}|\le \varepsilon$ with high probability. Plugging this insight back into the recursion \eqref{eq:information_recursion} of mutual information leads to (recall that $g(x)=\max_{|z|\le x} |f(z)|$)
\begin{align*}
d\varepsilon_{t+1}^2 \le d\varepsilon_t^2 + \frac{g(\varepsilon_t)^2}{2}, 
\end{align*}
which takes the same form as the update \eqref{eq:eps_sequence}. 

This insight is motivated by the following geometric calculation: if $a$ is uniformly distributed on the spherical cap $\{a\in \mathbb{S}^{d-1}: \jiao{\theta^\star, a}\ge \varepsilon \}$, then $I(\theta^\star; a)\asymp d\varepsilon^2$. However, the classical notion of the mutual information does not guarantee that this insight holds with a sufficiently high probability: the celebrated Fano's inequality (cf. \Cref{lemma:Fano}) tells that
\begin{align}\label{eq:Fano}
    \bP(|\jiao{\theta^\star, a}|\le \varepsilon) \ge 1 - \frac{I(\theta^\star; a) + \log 2}{c_0d\varepsilon^2}, 
\end{align}
for some absolute constant $c_0>0$. In other words, the probability of failure (i.e. $\jiao{\theta^\star, a}>\varepsilon$) could be as large as $I(\theta^\star; a)/(d\varepsilon^2)$, which is insufficient as $T$ could be much larger than $d$. Fano's inequality \eqref{eq:Fano} is also tight in the worst case: conditioned on $\theta^\star\sim \mathsf{Unif}(\mathbb{S}^{d-1})$, take $a\sim \text{Unif}(\{a\in \mathbb{S}^{d-1}: \jiao{\theta^\star, a}\ge \varepsilon \})$ with probability $p\asymp I/(d\varepsilon^2)$ and $a\sim \mathsf{Unif}(\mathbb{S}^{d-1})$ with probability $1-p$. In this case, $\bP(|\jiao{\theta^\star, a}|>\varepsilon)\asymp p$ and $I(\theta^\star; a)\asymp p\cdot (d\varepsilon^2)\asymp I$, and \eqref{eq:Fano} is tight. 

As a result, although the mutual information provides the correct intuition for the recursion in \eqref{eq:eps_sequence}, the potentially large failure probability in Fano's inequality \eqref{eq:Fano} prevents us from making the intuition formal. In the subsequent sections, we will find a proper notion of information such that
\begin{enumerate}
    \item it leads to a much smaller (e.g. exponential in $d$) failure probability in \eqref{eq:Fano}; 
    \item it satisfies an approximate chain rule such that the information recursion \eqref{eq:information_recursion} still holds. 
\end{enumerate}

\subsection{$\chi^2$-informativity}\label{subsec:chi^2_information}
In this section we introduce a new notion of information which satisfies the above two properties. For a pair of random variables $(X,Y)$ with a joint distribution $P_{XY}$, the \emph{$\chi^2$-informativity} \cite{csiszar1972class} between $X$ and $Y$ is defined as
\begin{align}\label{eq:chi^2_informativity}
I_{\chi^2}(X; Y) \triangleq \inf_{Q_Y}\chi^2(P_{XY} \| P_X\times Q_Y), 
\end{align}
where $\chi^2(P\|Q) = \int (\mathrm{d}P)^2/\mathrm{d}Q - 1$ is the $\chi^2$-divergence. Note that when the $\chi^2$-divergence is replaced by the Kullback-Leibler (KL) divergence, the expression in \eqref{eq:chi^2_informativity} exactly becomes the classical mutual information. Moreover, note that $I_{\chi^2}(X;Y)\neq I_{\chi^2}(Y;X)$ in general. 

In the sequel, we shall also need the following notion of \emph{conditional $\chi^2$-informativity}: for any measurable subset $E\subseteq \calX\times \calY$, the $\chi^2$-informativity conditioned on $E$ is defined as
\begin{align}\label{eq:conditional_chi^2_informativity}
I_{\chi^2}(X; Y \mid E) \triangleq \inf_{Q_Y} \chi^2(P_{XY \mid E} \| P_X\times Q_Y). 
\end{align}

The main advantage of the (conditional) $\chi^2$-informativity lies in the following lemma, which is reminiscent of the Fano's inequality in \eqref{eq:Fano}. 
\begin{lemma}\label{lemma:Fano_chi^2}
Let the random vector $(\theta^\star, a)$ satisfy that $\theta^\star\sim \mathsf{Unif}(\mathbb{S}^{d-1})$, and $a$ is supported on $\mathbb{B}^d$. For every $\varepsilon>0$ and every event $E$ of $(\theta^\star, a)$, it holds that
\begin{align*}
\bP(|\jiao{\theta^\star, a}|\le \varepsilon \mid E) \ge 1 - c_1e^{-c_0d\varepsilon^2}\sqrt{I_{\chi^2}(\theta^\star;a\mid E)+1},
\end{align*}
where $c_0,c_1>0$ are absolute constants. 
\end{lemma}
\begin{proof}
The high level idea of the proof is similar to \cite{chen2016bayes}, with minor modifications to deal with the conditioning. Let $P$ be the conditional distribution of $P_{(\theta^\star,a)}$ conditioned on $E$, and $Q$ be the distribution $P_{\theta^\star}\times Q_a$ with a generic distribution $Q_a$. Let $T: (\theta^\star, a)\mapsto \1(|\jiao{\theta^\star, a}|\le \varepsilon)$ be a given map, and $P\circ T^{-1}$ and $Q\circ T^{-1}$ be the pushforward measure of $P$ and $Q$ by $T$. The data processing inequality of the $\chi^2$-divergence gives
\begin{align*}
    \chi^2(P_{ (\theta^\star,a) \mid E}\| P_{\theta^\star} \times Q_a) &= \chi^2(P\|Q)\\
    &\ge \chi^2(P\circ T^{-1}\| Q\circ T^{-1}) \\
    &= \frac{(P(|\jiao{\theta^\star, a}|\le \varepsilon) - Q(|\jiao{\theta^\star, a}|\le \varepsilon))^2}{Q(|\jiao{\theta^\star, a}|\le \varepsilon)(1-Q(|\jiao{\theta^\star, a}|\le \varepsilon))}. 
\end{align*}
Moreover, by the product structure of $Q$, we have,
\begin{align}
    Q(|\jiao{\theta^\star, a}| > \varepsilon) \le \sup_{a_0\in \mathbb{B}^d} P_{\theta^\star}(|\jiao{\theta^\star, a_0}| > \varepsilon) \le e^{-c_0d\varepsilon^2}, \label{eq:dev-bound}
\end{align}
where the constant $c_0>0$ is given in \Cref{lemma:nolargeinnerproduct}. Consequently, 
\begin{align*}
\bP(|\jiao{\theta^\star, a}|\le \varepsilon \mid E) &= P(|\jiao{\theta^\star, a}|\le \varepsilon) \\
&\ge Q(|\jiao{\theta^\star, a}|\le \varepsilon) - \sqrt{e^{-c_0d\varepsilon^2}\cdot \chi^2(P_{ (\theta^\star,a) \mid E}\| P_{\theta^\star} \times Q_a)} \\
&\ge 1 - e^{-c_0d\varepsilon^2} - \sqrt{e^{-c_0d\varepsilon^2}\cdot \chi^2(P_{ (\theta^\star,a) \mid E}\| P_{\theta^\star} \times Q_a)} \\
&\ge 1 - c_1e^{-c_0d\varepsilon^2/2}\sqrt{\chi^2(P_{ (\theta^\star,a) \mid E}\| P_{\theta^\star} \times Q_a) + 1}, 
\end{align*}
for $c_1 = 2$. As the above inequality holds for every $Q_a$, taking the infimum over $Q_a$ completes the proof of the lemma. 
\end{proof}

Compared with Fano's inequality in \eqref{eq:Fano}, the probability of error in \Cref{lemma:Fano_chi^2} depends exponentially in $d\varepsilon^2$ and is thus sufficiently small, which enables us to apply a union bound argument. However, the $\chi^2$-informativity does not satisfy the chain rule or subadditivity (i.e. $I_{\chi^2}(X; Y,Z)\le I_{\chi^2}(X;Y) + I_{\chi^2}(X;Z\mid Y)$ may not hold), which makes it difficult to upper bound $I_{\chi^2}(\theta^\star; \calH_t)$ in the same manner as \eqref{eq:information_recursion}. This is the place where conditioning on a suitable event $E$ helps, and is the main theme of the next section.

\subsection{Upper bounding the $\chi^2$-informativity}\label{subsec:chi^2_information_analysis}
As we have discussed in the previous section, the $\chi^2$-informativity does not satisfy the chain rule or subadditivity. In this section, we establish a key lemma which upper bounds the $\chi^2$-informativity in a recursive manner via a proper conditioning. 

Let $\calH_t = \{(a_s,r_s)\}_{s\le t}$ be the learner's history up to time $t$, and $E_t = \cap_{s\le t} \{ |\jiao{\theta^\star, a_s}|\le \varepsilon_s \}$ be the target event with $\{\varepsilon_t\}_{t\ge 1}$ defined in \eqref{eq:eps_sequence}. The following lemma establishes a recursive relationship between the conditional $\chi^2$-informativity.
\begin{lemma}\label{lemma:recursion_chi^2}
For $t\ge 1$ and any prior distribution of $\theta^\star$, it holds that
\begin{align*}
    I_{\chi^2}(\theta^\star; \calH_t\mid E_t) + 1 \le \frac{\exp(g(\varepsilon_t)^2)}{\bP(E_t\mid E_{t-1})^2}(I_{\chi^2}(\theta^\star; \calH_{t-1}\mid E_{t-1})+1). 
\end{align*}
\end{lemma}
\begin{proof}
Let $\pi_0$ be an arbitrary prior distribution
of $\theta^\star$.
It is straightforward to observe that the joint distribution of $(\theta^\star, \calH_t)$ conditioned on $E_t$ could be written as
\begin{align*}
\bP(\theta^\star, \calH_t\mid E_t) = \frac{\1(E_t)}{\bP(E_t)}\cdot \pi_0(\theta^\star)\cdot \prod_{s=1}^t\left(\bP_s(a_s\mid \calH_{s-1})\cdot \varphi(r_s-f(\jiao{\theta^\star,a_s})\right),
\end{align*}
where $\bP_s$ denotes the learner's action distribution of $a_s$ based on the history $\calH_{s-1}$, and $\varphi(x)=\exp(-x^2/2)/\sqrt{2\pi}$ is the normal pdf. Based on the above expression, we have
\begin{align*}
    I_{\chi^2}(\theta^\star; \calH_t\mid E_t) + 1 &= \inf_{Q_{\calH_t}} \int \frac{\bP(\theta^\star, \calH_t\mid E_t)^2}{\pi_0(\theta^\star)Q_{\calH_t}(\calH_t)}\mathrm{d}\theta^\star\mathrm{d}a^t\mathrm{d}r^t \\
    &\le \inf_{Q_{\calH_{t-1}}} \int \frac{\bP(\theta^\star, \calH_t\mid E_t)^2}{\pi_0(\theta^\star)Q_{\calH_{t-1}}(\calH_{t-1})\cdot \bP_t(a_t\mid \calH_{t-1})\varphi(r_t)}\mathrm{d}\theta^\star\mathrm{d}a^t\mathrm{d}r^t \\
    &= \inf_{Q_{\calH_{t-1}}} \int \frac{\left[ \frac{\1(E_t)}{\bP(E_t)}\cdot \pi_0(\theta^\star)\cdot \prod_{s=1}^{t-1}\left(\bP_s(a_s\mid \calH_{s-1})\cdot \varphi(r_s-f(\jiao{\theta^\star,a_s})\right) \right]^2}{\pi_0(\theta^\star)Q_{\calH_{t-1}}(\calH_{t-1})}\\
    &\qquad \times \bP_t(a_t\mid \calH_{t-1})\times \frac{\varphi(r_t-f(\jiao{\theta^\star,a_t})^2}{\varphi(r_t)}\mathrm{d}\theta^\star\mathrm{d}a^t\mathrm{d}r^t \\
    &\stepa{=} \inf_{Q_{\calH_{t-1}}} \int \frac{\left[ \frac{\1(E_t)}{\bP(E_t)}\cdot \pi_0(\theta^\star)\cdot \prod_{s=1}^{t-1}\left(\bP_s(a_s\mid \calH_{s-1})\cdot \varphi(r_s-f(\jiao{\theta^\star,a_s})\right) \right]^2}{\pi_0(\theta^\star)Q_{\calH_{t-1}}(\calH_{t-1})}\\
    &\qquad \times \bP_t(a_t\mid \calH_{t-1})\times \exp(f(\jiao{\theta^\star,a_t})^2)\mathrm{d}\theta^\star\mathrm{d}a^t\mathrm{d}r^{t-1} \\
    &\stepb{\le} \inf_{Q_{\calH_{t-1}}} \int \frac{\left[ \frac{\1(E_t)}{\bP(E_t)}\cdot \pi_0(\theta^\star)\cdot \prod_{s=1}^{t-1}\left(\bP_s(a_s\mid \calH_{s-1})\cdot \varphi(r_s-f(\jiao{\theta^\star,a_s})\right) \right]^2}{\pi_0(\theta^\star)Q_{\calH_{t-1}}(\calH_{t-1})}\\
    &\qquad \times \bP_t(a_t\mid \calH_{t-1})\times \exp(g(\varepsilon_t)^2)\mathrm{d}\theta^\star\mathrm{d}a^t\mathrm{d}r^{t-1} \\
    &\stepc{\le} \inf_{Q_{\calH_{t-1}}} \int \frac{\left[ \frac{\1(E_{t-1})}{\bP(E_{t-1})}\cdot \pi_0(\theta^\star)\cdot \prod_{s=1}^{t-1}\left(\bP_s(a_s\mid \calH_{s-1})\cdot \varphi(r_s-f(\jiao{\theta^\star,a_s})\right) \right]^2}{\pi_0(\theta^\star)Q_{\calH_{t-1}}(\calH_{t-1})}\\
    &\qquad \times \bP_t(a_t\mid \calH_{t-1})\times \frac{\exp(g(\varepsilon_t)^2)}{\bP(E_t\mid E_{t-1})^2}\mathrm{d}\theta^\star\mathrm{d}a^t\mathrm{d}r^{t-1} \\
    &\stepd{=} \inf_{Q_{\calH_{t-1}}} \int \frac{\left[ \frac{\1(E_{t-1})}{\bP(E_{t-1})}\cdot \pi_0(\theta^\star)\cdot \prod_{s=1}^{t-1}\left(\bP_s(a_s\mid \calH_{s-1})\cdot \varphi(r_s-f(\jiao{\theta^\star,a_s})\right) \right]^2}{\pi_0(\theta^\star)Q_{\calH_{t-1}}(\calH_{t-1})}\\
    &\qquad \times \frac{\exp(g(\varepsilon_t)^2)}{\bP(E_t\mid E_{t-1})^2}\mathrm{d}\theta^\star\mathrm{d}a^{t-1}\mathrm{d}r^{t-1} \\
    &\stepe{=} \frac{\exp(g(\varepsilon_t)^2)}{\bP(E_t\mid E_{t-1})^2}(I_{\chi^2}(\theta^\star; \calH_{t-1}\mid E_{t-1})+1),
\end{align*}
where (a) integrates out $r_t$ (note that $E_t$ does not depend on $r_t$), (b) uses the definition of $E_t$ that $|\jiao{\theta^\star, a_t}|\le \varepsilon_t$, (c) follows from $\1(E_t)\le \1(E_{t-1})$, (d) integrates out $a_t$ as $E_{t-1}$ no longer depends on $a_t$, and (e) uses the definition of $I_{\chi^2}(\theta^\star; \calH_{t-1}\mid E_{t-1})$. This completes the proof. 
\end{proof}

Next we show how \Cref{lemma:recursion_chi^2} leads to the desired lower bound in \Cref{thm:lower_bound}. A repeated application of \Cref{lemma:recursion_chi^2} leads to
\begin{align} \label{eq:callback1}
I_{\chi^2}(\theta^\star; \calH_t \mid E_t) + 1 \le \prod_{s\le t}\frac{\exp(g(\varepsilon_s)^2)}{\bP(E_s\mid E_{s-1})^2} = \frac{1}{\bP(E_t)^2}\exp\left(\sum_{s\le t} g(\varepsilon_s)^2\right). 
\end{align}
Consequently, \Cref{lemma:Fano_chi^2} leads to
\begin{align*}
\bP(E_{t+1}\mid E_t) &= \bP(|\jiao{\theta^\star, a_{t+1}}|\le \varepsilon_{t+1}\mid E_t) \\
&\ge 1 - c_1e^{-c_0d\varepsilon_{t+1}^2}\sqrt{I_{\chi^2}(\theta^\star; a_{t+1}\mid E_t) + 1} \\
&\ge 1 - c_1e^{-c_0d\varepsilon_{t+1}^2}\sqrt{I_{\chi^2}(\theta^\star; \calH_t\mid E_t) + 1} \\
&\ge 1 - \frac{c_1}{\bP(E_t)}\exp\left(-c_0d\varepsilon_{t+1}^2 + \frac{1}{2}\sum_{s\le t} g(\varepsilon_s)^2\right),
\end{align*}
which further implies that
\begin{align}
\bP(E_{t+1}) = \bP(E_t)\cdot \bP(E_{t+1}\mid E_t) \ge \bP(E_t) - c_1\exp\left(-c_0d\varepsilon_{t+1}^2 + \frac{1}{2}\sum_{s\le t} g(\varepsilon_s)^2\right).  \label{eq:callback2}
\end{align}
By choosing $c\ge 1/(2c_0)$ in the recursion \eqref{eq:eps_sequence}, as well as $c>0$ large enough such that $c_1\exp(-cc_0\log(1/\delta))\le \delta$, the above inequality results in
\begin{align*}
\bP(E_{t+1}) \ge \bP(E_t) - c_1\exp(-c_0d\varepsilon_1^2) = \bP(E_t) - c_1\exp(-cc_0\log(1/\delta)) \ge \bP(E_t) - \delta,
\end{align*}
and therefore $\bP(E_t)\ge 1 - t\delta$. The proof of \Cref{thm:lower_bound} is now complete.

%% file: upperbound.tex
\section{Algorithm design}\label{sec:upper_bound}

In this section we propose an algorithm for the ridge bandit problem and prove the upper bounds in Theorems \ref{inftheorem:integral_ub} and \ref{inftheorem:learning_ub}. The algorithm consists of two stages. First, in \Cref{subsec:upper_bound_burnin}, we introduce an algorithm based on iterative direction search which finds a good initial action $a_0$ with $\jiao{\theta^\star, a_0}\ge x_0$ for a given target level $x_0\in (0,1)$; this algorithm could even be made agnostic to the knowledge of $f$. Based on this action, we proceed with a different regression-based algorithm in \Cref{subsec:upper_bound_linear} for the learning phase. For the ease of presentation, we assume that $f$ in monotone on $[-1,1]$ in the above sections, and leave the case where $f$ is even to \Cref{sec:proof_upper_bound} with slight algorithmic changes. 

\subsection{Algorithm for the burn-in period} \label{subsec:upper_bound_burnin}
Recall that the ultimate target in the burn-in period is to find an action $a_0$ which satisfies $\jiao{\theta^\star, a_0}\ge x_0$ with high probability. Our algorithmic idea is simple: if we could find $m:=\lceil x_0^2 d \rceil$ orthonormal vectors $v_1, v_2,\cdots, v_m$ with $\jiao{\theta^\star, v_i} \ge 1/\sqrt{d}$ for all $i\in [m]$, then $a_0 := m^{-1/2}\sum_{i=1}^m v_i$ is a unit vector with $\jiao{\theta^\star, a_0}\ge \sqrt{m/d}=x_0$. Finding these actions are not hard: \Cref{lemma:nontrivialcorr} shows that if $v_i\sim \mathsf{Unif}(\mathbb{S}^{d-1}\cap \text{span}(v_1,\cdots,v_{i-1})^\perp)$, then with a constant probability it holds that $\jiao{\theta^\star,v_i}\ge 1/\sqrt{d}$. The main difficulty lies in the \emph{certification} of $\jiao{\theta^\star,v_i}\ge 1/\sqrt{d}$, where we aim to make both Type I and Type II errors negligible for the following hypothesis testing problem: 
\begin{align*}
    H_0: \jiao{\theta^\star, v_i} < \frac{1}{\sqrt{d}} \quad \text{ v.s. } \quad H_1: \jiao{\theta^\star, v_i} \ge \frac{1+\kappa_1}{\sqrt{d}}.  
\end{align*}
Here $\kappa_1>0$ is a small constant to be chosen later in Theorem \ref{thm:ub_burnincost_formal}. The key ingredient of our algorithm is to find such a test which makes good use of the historic progress $(v_1,\cdots,v_{i-1})$. 

\begin{algorithm}[htb]
\caption{Iterative direction search algorithm}\label{alg:burn-in}
\textbf{Input:} link function $f$, dimensionality $d$, a noisy oracle $\calO: a\in \mathbb{B}^d\mapsto \calN(f(\jiao{\theta^\star,a}),1)$, error probability $\delta$, target inner product $x_0\in (0,1)$. \\
\textbf{Output:} an action $a_0$ such that $\jiao{\theta^\star, a_0}\ge x_0$ with probability at least $1-\delta$.  

Let numerical constants $(\kappa_1,\kappa_2,c_0,d_0)$ be given in Theorem \ref{thm:ub_burnincost_formal}. \\
Let $m\gets \lceil x_0^2 d \rceil, V\gets \{\textbf{0}_d\}, L\gets 2m\log(2m/\delta)/c_0$.  

\tcp{Find initial few directions}
\For{epoch $i = 1,\cdots, d_0$}{
\While{{\upshape \textsf{True} }}{
	 Sample $v \sim \mathrm{Unif} (V^\perp \cap \mathbb{S}^{d-1} )$\;
	    \If{$\iaht \left(v;f,d,\calO,\delta/L, \kappa_1/4 \right) = \true$}{
     
    $v_i \gets v$;
    
    $V \gets \text{span}(V\cup \{v_i\})$;

    \textbf{break}\;
}}}

\tcp{Find subsequent directions}
\For{epoch $i = d_0+1,\cdots,m$}{
$v_{\text{pre}} \gets \frac{1}{\sqrt{i-1}} \sum_{j=1}^{i-1} v_j$;

$x_{\text{\rm pre}} \gets \sqrt{(i-1)/d}$; 
            
\While{{\upshape \textsf{True}}}{
Sample $v \sim \mathrm{Unif} (V^\perp \cap \mathbb{S}^{d-1} )$;
     
\If{$\gaht \left( v; f,d,\calO,\delta/L,\kappa_1/4,\kappa_2,v_{\text{\rm pre}},x_{\text{\rm pre}} \right) = \true$}{
     
    $v_i \gets v$;

    $V \gets \text{span}(V\cup \{v_i\})$;
     
     \textbf{break}\;}
     }
}
\tcp{Final action}
Output $a_0\gets \frac{1}{\sqrt{m}}\sum_{i=1}^m v_i$. 
\end{algorithm}

The detailed algorithm is summarized in Algorithm \ref{alg:burn-in}. The algorithm runs in two stages, and calls two certification algorithms \iaht and \gaht as subroutines for the respective stages. In each of the $m$ epochs at both stages, a uniformly random direction $v_i \sim \mathsf{Unif}(\mathbb{S}^{d-1}\cap \text{span}(v_1,\cdots,v_{i-1})^\perp)$ orthogonal to the past directions is chosen, and the difference lies in how we decide whether to accept $v_i$ or not, i.e. the certification of $v_i$. Concretely, each epoch aims to achieve the following two targets: 
\begin{itemize}
    \item With a constant probability, the certification algorithm accepts a random $v_i$. This leads to a small number of trials in each loop and a small overall sample complexity. 
    \item Whenever the certification algorithm accepts $v_i$, then with high probability we have $\jiao{\theta^\star, v_i}\in [1/\sqrt{d}, (1+\kappa_1)/\sqrt{d}]$. This leads to the correctness of the algorithm. (In principle we only need the lower bound, and the upper bound is mainly for technical convenience.)
\end{itemize}

The initial stage consists of the first $d_0$ epochs, and uses a simple certification algorithm \iaht displayed in \Cref{alg:IAHT}. The recursive stage consists of the rest of the epochs, and the certification algorithm \gaht in \Cref{alg:GAHT} exploits the current progress, including a good direction $v_{\text{pre}}$ and an estimate $x_{\text{pre}}$ of the inner product: we will show that $\jiao{\theta^\star, v_{\text{pre}}}\in [x_{\text{pre}}, (1+\kappa_1)x_{\text{pre}}]$ in every epoch. The certification algorithms will be detailed in the next few subsections, and they aim to collect as few as samples to reliably solve the hypothesis testing problem. 

The performance of \Cref{alg:burn-in} is summarized in the following theorem. 

\begin{theorem}\label{thm:ub_burnincost_formal}
Let $\delta \in (0,1/2)$. Suppose that $\kappa_1 \in (0,(x_0^{-1}-1)/2)$, $\kappa_2\in (0,1/4)$, and
\begin{align*}
    d_0 = \left\lceil \frac{(2\kappa_1+4)^2\kappa_2(2-\kappa_2)}{\kappa_1^2(1-\kappa_2)^2} \right\rceil + 1, \quad c_0 = c\left( 1+\frac{\kappa_1}{4}, 1 + \frac{\kappa_1}{2}, 1 - x_0^2, \frac{1-x_0}{2} \right), 
\end{align*}
where the function $c(\cdot)$ appears in \Cref{lemma:nontrivialcorr}. Let $\{ \varepsilon_i \}_{i \ge 0}$ be a set of positive reals defined by
\begin{align*}
\varepsilon_i = \begin{cases}
\frac{1}{2}\min_{z \in \left[1/\sqrt{d}, (1+\kappa_1/2)/\sqrt{d}\right]} \left| f \left(z + \frac{c_1}{\sqrt{d}} \right) - f \left(z \right) \right| &\text{if } 1\le i\le d_0, \\
\frac{1}{2} \max_{c_2/\sqrt{d} \le y\le (1-\kappa_2)\sqrt{(i-1)/d}} \min_{z\in [(1-\kappa_1)y, (1+\kappa_1)y]} \left| f \left( z + \frac{c_1}{\sqrt{d}} \right) - f \left( z \right) \right| &\text{if } d_0+1\le i\le m,  
\end{cases} 
\end{align*}
where $c_1 = \kappa_1\sqrt{1-(1-\kappa_2)^2}/4$, $c_2=(2\kappa_1+4)\sqrt{1-(1-\kappa_2)^2}/\kappa_1$ are numerical constants determined by $(\kappa_1,\kappa_2)$, and $m = \lceil x_0^2 d \rceil$. 

If $f$ is monotone on $[-1,1]$, then with probability at least $1-\delta$, Algorithm \ref{alg:burn-in} outputs an action $a_0$ with $\jiao{\theta^\star, a_0}\ge x_0$ using at most
\begin{align*}
O\left(\log^2\left(\frac{d}{\delta}\right) \sum_{i=1}^{m} \frac{1}{\varepsilon_i^2} \right)
\end{align*}
queries, where the hidden constant depends only on $(x_0, \kappa_1, \kappa_2)$. 
\end{theorem}

By \Cref{lemma:difference_to_derivative} in the appendix, \Cref{thm:ub_burnincost_formal} implies the integral form of \Cref{inftheorem:integral_ub}. Moreover, an inspection of the proof reveals that using $\widetilde{O}(\sum_{i=1}^{k} \varepsilon_i^{-2} )$ samples gives an action $a_k$ with $\jiao{\theta^\star, a_k}\ge \sqrt{k/d}$, and this implies the learning trajectory upper bound in \Cref{thm:learning_trajectory}. 

The remainder of this section is organized as follows. In \Cref{subsec:iaht} and \ref{subsec:gaht}, we detail the certification algorithms \iaht and \gaht\!\!, and analyze their performances. \Cref{subsec:agnostic_alg} modifies Algorithm \ref{alg:burn-in} to make it \emph{agnostic} to the knowledge of $f$, such that the same upper bound in Theorem \ref{thm:ub_burnincost_formal} could be achieved by an algorithm without the knowledge of $f$. The proofs of the correctness and the sample complexity upper bounds for these algorithms are deferred to the appendix.

\subsubsection{Certifying initial directions}\label{subsec:iaht}
The \iaht algorithm certifying the quality of the initial directions $v$ is displayed in Algorithm \ref{alg:IAHT}. The idea is simple and requires nothing from the past: we query the test action $v$ multiple times to obtain an accurate estimate of $f(\jiao{\theta^\star,v})$, and apply a projection based test to see if the inner product $\jiao{\theta^\star, v}$ lies in the target interval. Here the parameter $\kappa_1$ represents the target accuracy for certification. The performance of this test is summarized in the following lemma. 

\begin{algorithm}[htb]
\caption{$\iaht(v;f,d,\calO,\delta,\kappa_1)$}
	\label{alg:IAHT}
\textbf{Input:} link function $f$, dimensionality $d$, a noisy oracle $\calO: a\in \mathbb{B}^d\mapsto \calN(f(\jiao{\theta^\star,a}),1)$, error probability $\delta$, accuracy parameter $\kappa_1$, test direction $v$. 

\textbf{Output:} with probablity $\ge 1-\delta$, $\mathsf{True}$ if $\jiao{\theta^\star,v} \in [(1+\kappa_1)/\sqrt{d}, (1+2\kappa_1)/\sqrt{d}]$, $\mathsf{False}$ if $\jiao{\theta^\star,v} \notin [1/\sqrt{d}, (1+3\kappa_1)/\sqrt{d}]$. 

Define
\begin{align} \label{eq:eps0-def}
    \varepsilon := \frac{1}{2}\min_{z \in \left[1, 1+2\kappa_1\right]} \left| f \left(\frac{z+\kappa_1}{\sqrt{d}} \right) - f \left(\frac{z}{\sqrt{d}}\right) \right|.
\end{align}

Query the test action $2\log(2/\delta)/\varepsilon^2$ times and compute the sample average $\overline{r}$\;

\eIf{$\exists x \in [(1+\kappa_1)/\sqrt{d}, (1+2\kappa_1)/\sqrt{d}]$ such that
    $|\overline{r} - f(x) |\le \varepsilon$}{\textbf{Return} \true}{\textbf{Return} \false}
\end{algorithm}

\begin{lemma}\label{lemma:iaht}
Suppose $f$ is monotone on $[-1,1]$. Then with probability at least $1-\delta$, the \iaht algorithm outputs:
\begin{itemize}
    \item \textsf{True} if $\jiao{\theta^\star, v}\in [(1+\kappa_1)/\sqrt{d}, (1+2\kappa_1)/\sqrt{d}]$; 
    \item \textsf{False} if $\jiao{\theta^\star, v}\notin [1/\sqrt{d}, (1+3\kappa_1)/\sqrt{d}]$. 
\end{itemize}
\end{lemma}

\subsubsection{Certifying subsequent directions}\label{subsec:gaht}

\begin{algorithm}[htb]
	\caption{$\gaht(v; f, d, \calO, \delta, \kappa_1, \kappa_2, v_{\text{pre}}, x_{\text{pre}})$}
	\label{alg:GAHT}
\textbf{Input:} link function $f$, dimensionality $d$, a noisy oracle $\calO: a\in \mathbb{B}^d\mapsto \calN(f(\jiao{\theta^\star,a}),1)$, error probability $\delta$, accuracy parameters $(\kappa_1, \kappa_2)$, test direction $v$, previous action $v_{\text{pre}}$, previous inner product $x_{\text{pre}}$. 

\textbf{Output:} with probablity $\ge 1-\delta$, $\mathsf{True}$ if $\jiao{\theta^\star,v} \in [(1+\kappa_1)/\sqrt{d}, (1+2\kappa_1)/\sqrt{d}]$, $\mathsf{False}$ if $\jiao{\theta^\star,v} \notin [1/\sqrt{d}, (1+3\kappa_1)/\sqrt{d}]$. 

Define $\kappa_2^\perp := \sqrt{1-(1-\kappa_2)^2}$, $\kappa_3 := \kappa_1 \kappa_2^\perp$, $\kappa_4 := (\kappa_1^{-1}+2)\kappa_2^\perp$, and 
\begin{align}
    \varepsilon := \frac{1}{2}  \max_{\kappa_4/\sqrt{d} \le y\le (1-\kappa_2)x_{\text{pre}}} \min_{z\in [(1-4\kappa_1)y, (1+4\kappa_1)y]} \left| f \left( z + \frac{\kappa_3}{\sqrt{d}} \right) - f \left( z \right) \right|. \label{eq:obj}
\end{align}

Let $y^\star$ be the maximizer of \eqref{eq:obj}, and define $\lambda := y^\star/[(1-\kappa_2)x_{\text{pre}}]\in [0,1]$. 
     
Query both actions $a_- = \lambda (1-\kappa_2) v_{\text{pre}} - \kappa_2^\perp v$ and $a_+ = \lambda (1-\kappa_2) v_{\text{pre}} + \kappa_2^\perp v$ for $2\log(4/\delta)/\varepsilon^2$ times, and compute the sample averages $\overline{r}_-$ and $\overline{r}_+$. 

\eIf{$\exists z \in \left[ y^\star , (1+3\kappa_1) y^\star \right]$ and $x \in [(1+\kappa_1)\kappa_2^\perp/\sqrt{d}, (1+2\kappa_1)\kappa_2^\perp/\sqrt{d}]$ such that
\begin{align} \label{eq:test}
    | \overline{r}_- - f( z - x )| \le \varepsilon \text{ and } | \overline{r}_+ - f( z + x )| \le \varepsilon.
\end{align}}{\textbf{Return} \true}{\textbf{Return} \false}
\end{algorithm}

In principle, certifying subsequent directions can also use the \iaht algorithm, but this may lead to a suboptimal sample complexity. The central question we answer in this section is as follows: \emph{Given an action $v_{\text{\rm pre}}$ with a known estimate $x_{\text{\rm pre}}$ for the inner product $\jiao{\theta^\star, v_{\text{\rm pre}}}\approx x_{\text{\rm pre}}$, can we certify the test direction $v$ with a smaller sample complexity?}

Recall the simple idea of the \iaht algorithm: by querying the action $v$, we estimate the value of $f(\jiao{\theta^\star, v})$, and then certify the value of the inner product $\jiao{\theta^\star,v}$. Our new observation is that, if we could estimate the value of $f(x + \jiao{\theta^\star,v})$ for a known $x$, then we could certify the value of  $\jiao{\theta^\star,v}$ as well. Since the propagation from the estimation error of $f(\jiao{\theta^\star, v})$ to that of $\jiao{\theta^\star, v}$ depends on the derivative of $f$, such a translation by $x$ could lead to a better derivative and benefit the certification step. This intuition leads to the following \gaht algorithm displayed in Algorithm \ref{alg:GAHT}. 

In \Cref{alg:GAHT}, instead of directly querying the test direction $v$, we query two actions based on the current progress: for some $\lambda\in [0,1]$ to be chosen later, pick
\begin{align*}
    a_- = \lambda (1-\kappa_2)v_{\text{pre}} - \kappa_2^\perp v, \quad a_+ = \lambda (1-\kappa_2)v_{\text{pre}} + \kappa_2^\perp v. 
\end{align*}
Here $\kappa_2\in (0,1)$ is a parameter controlling the range of the center $x$, and $\kappa_2^\perp := \sqrt{1-(1-\kappa_2)^2}$. Since $v\perp v_{\text{pre}}$ and $\lambda\in [0,1]$, both actions lie in $\mathbb{B}^d$. By querying these actions for multiple times, we obtain accurate estimates of $f(\lambda(1-\kappa_2)\jiao{\theta^\star, v_{\text{pre}}} \pm \kappa_2^\perp \jiao{\theta^\star, v})$. As $\jiao{\theta^\star, v_{\text{pre}}} \approx x_{\text{pre}}$, this corresponds to the shift of the center as outlined above, and the tuning of $\lambda\in [0,1]$ gives us the flexibility of centering anywhere below the current progress $x_{\text{pre}}$. Roughly speaking, we will choose $\lambda\in [0,1]$ to maximize the derivative $f'(\lambda x_{\text{pre}})$. Finally, as the relation $\jiao{\theta^\star, v_{\text{pre}}} \approx x_{\text{pre}}$ is only approximate, we use a more complicated projection based test \eqref{eq:test} for the certification of $\jiao{\theta^\star, v}$. 

The performance of \Cref{alg:GAHT} is summarized in the following lemma. 
\begin{lemma}\label{lemma:gaht}
Suppose $f$ is monotone on $[-1,1]$, $\jiao{\theta^\star, v_{\text{\rm pre}}} \in [x_{\text{\rm pre}}, (1+3\kappa_1)x_{\text{\rm pre}}]$, and $x_{\text{\rm pre}}\ge \kappa_4/[(1-\kappa_2)\sqrt{d}]$. Then with probability at least $1-\delta$, the \gaht algorithm outputs:
\begin{itemize}
    \item \textsf{True} if $\jiao{\theta^\star, v}\in [(1+\kappa_1)/\sqrt{d}, (1+2\kappa_1)/\sqrt{d}]$; 
    \item \textsf{False} if $\jiao{\theta^\star, v}\notin [1/\sqrt{d}, (1+3\kappa_1)/\sqrt{d}]$. 
\end{itemize}
\end{lemma}

\subsubsection{An algorithm without the knowledge of $f$}\label{subsec:agnostic_alg}

Recall that Algorithm \ref{alg:burn-in} crucially relies on the knowledge of $f$, as it is used in both projection based tests in the \iaht and \gaht algorithms. The main result of this section is summarized in the following theorem, showing that the knowledge of $f$ is not required for the burn-in period. 

\begin{theorem}\label{thm:burnin_agnostic}
Consider the same setting of \Cref{thm:ub_burnincost_formal}, and assume that $f$ is continuous and strictly increasing on $[-1,1]$. Then there is an algorithm without the knowledge of $f$ such that, with probability at least $1-\delta$, it outputs an action $a_0$ with $\jiao{\theta^\star,a_0}\ge x_0$ using
\begin{align*}
O\left(\log^3\left(\frac{d}{\delta}\right)\sum_{i=1}^m \frac{1}{\varepsilon_i^2}\right)
\end{align*}
queries, where the hidden constant depends only on $(x_0,\kappa_1,\kappa_2)$. 
\end{theorem}
Up to logarithmic factors in $(d,1/\delta)$, the sample complexity in \Cref{thm:burnin_agnostic} matches the result in \Cref{thm:ub_burnincost_formal}. The main algorithmic idea is to solve the following \emph{infinite-armed bandit problem}. Let $F$ be an unknown, continuous, and strictly increasing CDF, so that $F^{-1}(t)$ is well-defined for every $t\in (0,1)$. Let $X_1, X_2, \cdots \sim F$ be an (unobserved) infinite i.i.d. sequence (treat the index set $\mathbb{N}$ as \emph{arms}). At each time $t$, the learner chooses an arm $i_t\in\mathbb{N}$ and observes $Y_t\sim \calN(X_{i_t},1)$; the learner could either pull a new arm for exploration, or pull an existing arm to refine the knowledge of $X$. We assume that the noises at different rounds are independent. Given two values $p, q\in [0,1]$ with $p<q$, the learner's target is to find some $i\in \mathbb{N}$ such that $F(X_i)\in [p,q]$. A line of work \cite{berry1997bandit,wang2008algorithms,bonald2013two,wang2022beyond} considered similar settings, but typically focused on different targets such as best arm identification or functional estimation. 

The following lemma presents a simple algorithm based on upper and lower confidence bounds, together with a high-probability guarantee on the sample complexity. 
\begin{lemma}\label{lemma:infinite-armed-bandit}
Fix any $\varepsilon\in (0,(q-p)/4)$, and a failure probability $\delta\in (0,1/2)$. There is a learning algorithm such that with probability at least $1-\delta$, it outputs some $i\in \mathbb{N}$ with $F(X_i)\in [p,q]$ using
\begin{align*}
O_{q-p,\varepsilon}\left(\frac{\log^2(1/\delta)}{(F^{-1}(p+2\varepsilon) - F^{-1}(p+\varepsilon))^2} + \frac{\log^2(1/\delta)}{(F^{-1}(q-\varepsilon) - F^{-1}(q-2\varepsilon))^2} \right)
\end{align*}
queries, where both the algorithm and the hidden constant are independent of $F$. 
\end{lemma}

To see how \Cref{lemma:infinite-armed-bandit} is related to our problem, consider the initial certification steps in \Cref{alg:burn-in}. Let $F$ be the CDF of $f(\jiao{\theta^\star,v})$ for $v\sim \mathsf{Unif}(\mathbb{S}^{d-1})$, which is unknown due to the unknown $f$. If we sample $v_1, v_2, \cdots \sim \mathsf{Unif}(\mathbb{S}^{d-1})$, then each direction $v_i$ is an arm in the infinite-armed bandit problem, with corresponding $X_i = f(\jiao{\theta^\star,v_i})$, and the reward $r_t\sim \calN(X_{i_t},1)$ is the observation $Y_t$ when the direction $v_{i_t}$ is chosen. The crucial observation here is that both
\begin{align*}
    p &= F^{-1}(f(1/\sqrt{d})) = \bP(\jiao{\theta^\star,v}\le 1/\sqrt{d}),\\
    q &= F^{-1}(f((1+\kappa_1)/\sqrt{d})) = \bP(\jiao{\theta^\star,v}\le (1+\kappa_1)/\sqrt{d})
\end{align*}
are known thanks to the strict monotonicity of $f$, and \Cref{lemma:nontrivialcorr} tells that $q-p = \Omega(1)$. Therefore, we can apply \Cref{lemma:infinite-armed-bandit} to find a direction (arm) $v_i$ such that $\jiao{\theta^\star, v_i}\in [1/\sqrt{d}, (1+\kappa_1)/\sqrt{d}]$, with sample complexity essentially $\widetilde{O}(1/\varepsilon_1^2)$ in \Cref{thm:ub_burnincost_formal}. In summary, instead of certifying each direction one after one using the knowledge of $f$ in \Cref{alg:burn-in}, the agnostic algorithm makes use of the empirical CDF based on the comparisons between different actions.  

The same idea could also be applied to recursive certification steps, with two additional caveats. First, the CDF of $\jiao{\theta^\star, v}$ with $v\sim \mathsf{Unif}(\mathbb{S}^{d-1}\cap V^\perp)$ involves an unknown magnitude $\|\text{Proj}_{V^\perp}(\theta^\star)\|_2$; in the algorithm we estimate it and apply an induction in the analysis. Second, the optimal value of $\lambda$ in \Cref{alg:GAHT} is unknown; we overcome it by searching over a geometric grid on $\lambda$. The detailed algorithms, as well as the proofs of \Cref{lemma:infinite-armed-bandit} and \Cref{thm:burnin_agnostic}, are postponed to the appendix.

\subsection{Algorithm for the learning phase} \label{subsec:upper_bound_linear}

In this section we design an algorithm after a good action $a_0$ with $\jiao{\theta^\star,a_0}\ge 1 - 3\gamma/4$ is found, and prove the upper bound in \Cref{inftheorem:learning_ub}. The algorithm is based on a simple idea of explore-then-commit (ETC) shown in \Cref{alg:learning_phase}. In the first $m$ rounds, we cyclically explore all directions \emph{around $a_0$} in a non-adaptive manner: 
\begin{align*}
    a_t \in \left\{ \left(1 - \frac{\gamma}{8} \right) a_0 \pm \frac{\gamma}{8} e_i: i\in [d] \right\} \subseteq \mathbb{B}^d. 
\end{align*}
Here $e_i$ is the $i$-th canonical vector of $\mathbb{R}^d$. The reason why we center these actions around $a_0$ is that
\begin{align*}
\left\langle \theta^\star, \left( 1 - \frac{\gamma}{8} \right) a_0 \pm \frac{\gamma}{8} e_i \right\rangle \ge \left( 1 - \frac{3\gamma}{4} \right)\left( 1 - \frac{\gamma}{8} \right) - \frac{\gamma}{8} > 1 - \frac{3\gamma}{4} -  \frac{\gamma}{8} - \frac{\gamma}{8} = 1 - \gamma
\end{align*}
for all $i\in [d]$ and therefore we are operating in the locally linear regime in \Cref{assump:learning}. After the exploration rounds, we compute the constrained least squares estimator $\thetaLS$ for $\theta^\star$ in \eqref{eq:constrained_LS}. If our target is the estimation of $\theta^\star$, we just set $m=T$ and use $\thetaLS$ as the final estimator. If our target is to minimize the regret, we commit to the action $a_t = \thetaLS$ after $t>m$, and choose $m$ appropriately to balance the errors in the exploration and commit rounds.

\begin{algorithm}[htb]
\caption{Regression-based explore-then-commit algorithm}
	\label{alg:learning_phase}
\textbf{Input:} link function $f$, dimensionality $d$, time horizon $T$, action $a_0$ with $\langle a_0, \theta^\star \rangle \ge 1-\gamma$. 

\textbf{Output:} final estimator $\widehat{\theta}_T$, or a sequence of actions $(a_1,\cdots,a_T)$.  

Set 
\begin{align}\label{eq:m}
    m \gets \begin{cases}
        T & \text{for estimation}, \\
        \min\{T, d\sqrt{T}/c_f \} & \text{for regret minimization}.
    \end{cases}
\end{align}

\For{$t=1,2,\cdots,m$}{
Play action $a_t = \left( 1 - \frac{\gamma}{8} \right) a_0 + (-1)^{\lceil t/d\rceil}\cdot \frac{\gamma}{8} e_{((t-1)\bmod d) + 1}$; 

Receive reward $r_t\sim \calN(f(\jiao{\theta^\star,a_t}),1)$. 
}

Compute the constrained least squares estimator: 
\begin{align}\label{eq:constrained_LS}
  \thetaLS = \argmin_{\theta\in \mathbb{S}^{d-1}: \jiao{\theta, a_0}\ge 1-\gamma} \sum_{t=1}^m \left( f (\langle \theta, a_t \rangle) - r_t \right)^2.
\end{align}
            
\For{$t=m+1,\cdots,T$}{Commit to the action $a_t = \thetaLS$.}

Return $\widehat{\theta}_T=\thetaLS$ or $(a_1,\cdots,a_T)$. 
\end{algorithm}

Next we show that \Cref{alg:learning_phase} attains the upper bounds in \Cref{inftheorem:learning_ub}. To this end, we analyze the statistical performance of the constrained least squares estimator in \eqref{eq:constrained_LS}. Applying Lemma \ref{lemma:least_squares} to the function class $\calF:= \{a\mapsto f(\jiao{\theta, a}): \theta\in \mathbb{S}^{d-1}, \jiao{\theta, a_0}\ge 1-\gamma \}$, we have
\begin{align*}
\log N(u, \calF_m(\delta), L_2(P_m)) &\le \log N(u, \calF, L_2(P_m)) \\
&\stepa{\le} \log N(u/C_f, \mathbb{S}^{d-1}, L_2(\mathbb{R}^d)) = O\left(d \log\frac{C_f}{u}\right), 
\end{align*}
where (a) follows from the Lipschitz contraction of metric entropy, in conjunction with the observation that for every $t\in [m]$ and $\theta\in \mathbb{S}^{d-1}$ with $\jiao{\theta,a_0}\ge 1 - \gamma$, 
\begin{align*}
    \| \nabla_{\theta} f(\jiao{\theta,a_t}) \|_2 \le f'(\jiao{\theta, a_t}) \le C_f, 
\end{align*}
where the last step is due to \Cref{assump:learning} and $\jiao{\theta,a_t}>1-\gamma$. Consequently, \Cref{lemma:least_squares} gives that $\delta_m \asymp (1+d\log(mC_f))/\sqrt{m}$, and therefore
\begin{align}\label{eq:least_square_performance}
\bE\left[\sum_{t=1}^m (f(\jiao{\theta^\star, a_t}) - f(\jiao{\thetaLS, a_t}))^2\right] = \widetilde{O}(d). 
\end{align}
Based on \eqref{eq:least_square_performance}, we conclude that
\begin{align*}
    \bE[1 - \jiao{\theta^\star, \thetaLS}] &= \frac{1}{2}\bE[\|\theta^\star - \thetaLS\|_2^2] \stepb{\le} \frac{32d}{m\gamma^2} \bE\left[\sum_{t=1}^m \jiao{\theta^\star-\thetaLS, a_t}^2\right] \\
    &\stepc{\le} \frac{32d}{m \gamma^2 c_f^2} \bE\left[\sum_{t=1}^m (f(\jiao{\theta^\star, a_t}) - f(\jiao{\thetaLS, a_t}))^2\right] = \widetilde{O}\left(\frac{d^2}{m c_f^2}\right). 
\end{align*}
Here (b) is due to the definition of $\{a_t\}_{t=1}^m$ and simple algebra, (c) follows from $\jiao{\theta^\star, a_t}, \jiao{\thetaLS, a_t} > 1-\gamma$ and \Cref{assump:learning}. Choosing $m = T \asymp d^2/(c_f^2\varepsilon)$ proves the sample complexity upper bound of \Cref{inftheorem:learning_ub}. 

As for the regret, each round during the exploration phase incurs a regret at most 
\begin{align*}
    f(1) - f(\jiao{\theta^\star, a_t}) \le f(1) - f(1-\gamma) \le \min\{ \gamma C_f, 1\}. 
\end{align*}
Moreover, using the high probability upper bound in \Cref{lemma:least_squares}, with an overwhelming probability each round in the commit phase incurs a regret 
\begin{align*}
    f(1) - f(\jiao{\theta^\star, \thetaLS}) \le f(1) - f\left(1-\widetilde{O}\left(\frac{d^2}{m c_f^2}\right)\right) = \widetilde{O}\left(\frac{C_fd^2}{m c_f^2}\right). 
\end{align*}
Consequently, the total expected regret of \Cref{alg:learning_phase} is
\begin{align*}
\mathfrak{R}_T^\star(f,d) = \widetilde{O}\left(m\min\{ C_f, 1\} + (T-m)\cdot \frac{C_fd^2}{m c_f^2} \right) = \widetilde{O}\left(\min\left\{T, \frac{C_f d\sqrt{T}}{c_f}\right\} \right), 
\end{align*}
where the last step follows from the choice of $m$ in \eqref{eq:m}. The proof of \Cref{inftheorem:learning_ub} is now complete. 

%% file: discussion.tex
\section{Additional discussions}\label{sec:discussions}

\subsection{Nonadaptive sampling}
In this section, we show that the upper bound on the burn-in cost in \Cref{inftheorem:integral_ub} cannot be attained by any nonadaptive sampling approaches in general. Here under nonadaptive sampling, the actions $a_1,\cdots,a_T\in \mathbb{B}^d$ are chosen in advance without knowing the history. This result reveals an avoidable gap between adaptive and nonadaptive samplings, and emphasizes the importance of the sequential nature in our decision making problem. 

\begin{theorem}\label{thm:nonadaptive}
Let the link function $f$ satisfy \Cref{assump:main} in the ridge bandit problem, and $\theta^\star\sim \mathsf{Unif}(\mathbb{S}^{d-1})$. Then any nonadaptive learner cannot find $\widehat{\theta}_T$ with $\bE[\jiao{\theta^\star, \widehat{\theta}_T}] > 1/2$ if
\begin{align*}
    T < \max_{K\ge 1} \frac{cd}{g(\sqrt{(\log K)/d})^2 + K^{-1}}, 
\end{align*}
where $c>0$ is an absolute constant, and $g(x):=\max\{|f(x)|,|f(-x)|\}$. 
\end{theorem}

If $f(x)=|x|^p$ with $p>0$, \Cref{thm:nonadaptive} shows that the burn-in cost for all nonadaptive algorithms is at least $\widetilde{\Omega}(d^{p+1})$, which is suboptimal compared with Example \ref{example:1} when $p>1$. Thanks to the nonadaptive nature where $a_t$ is independent of $\theta^\star$, \Cref{thm:nonadaptive} could be proven via the classical Fano's inequality. Without this independence in the adaptive setting, we need a recursive relationship for the mutual information $I(\theta^\star; a_t)$ in the proof of \Cref{thm:lower_bound}. 

\subsection{Finitely many actions}
In this section we consider the case where the action space $\mathcal{A}$ is not continuous and is a finite subset of $\mathbb{B}^d$, with $|\calA|=K$. For linear bandits, a finite set of actions helps reduce the minimax regret from $\Theta(d\sqrt{T})$ to $\Theta(\sqrt{dT\log K})$, essentially due to the reason that it becomes less expensive to maintain a confidence bound for each action (see, e.g. \cite[Chapter 22]{lattimore2020bandit}). However, to achieve the optimal burn-in cost for general ridge bandits, we already know that it is necessary to go beyond confidence bounds. In this case, does a finite number of actions help to reduce the burn-in cost as well? 

The next theorem shows that for many link functions, a smaller set of actions does not essentially help. 

\begin{theorem}\label{thm:finite_action}
Let the link function $f$ satisfy \Cref{assump:main} in the ridge bandit problem. For every $K=\exp(o(d))$, there exists a finite action set $\calA$ with $|\calA|=K$ such that any learner cannot find $\widehat{\theta}_T$ with $\inf_{\theta^\star\in \mathbb{S}^{d-1}}\bE_{\theta^\star}[\jiao{\theta^\star, \widehat{\theta}_T}] \ge 4/5$ if
\begin{align*}
    T < \frac{c}{g(\sqrt{(c'\log K)/d})^2 + K^{-1}}, 
\end{align*}
where $c,c'>0$ are absolute constants, and $g(x):=\max\{|f(x)|,|f(-x)|\}$. 
\end{theorem}

For $f(x)=|x|^p$ with $p>0$, \Cref{thm:finite_action} shows that the burn-in cost with appropriately chosen $K$ actions is at least $\widetilde{\Omega}(d^p)$ as long as $K\gtrsim d^p$. If $p\ge 2$, this is no smaller than the optimal burn-in cost with a continuous set of actions, showing that a smaller action set is essentially not beneficial. From the algorithmic perspective, this is because that \Cref{alg:burn-in} for the burn-in period crucially requires that every direction, and in particular every convex combination of actions, could be explored - a structure that may break down for finitely many actions. Under a given discrete action set, it is an interesting future direction to understand both the burn-in cost and the appropriate algorithm for the burn-in period. 

We also point out some technical aspects in the proof of \Cref{thm:finite_action}. First, a proof based on the $\chi^2$-informativity argument in \Cref{sec:lower_bound} still works, but the proof we present in \Cref{subsec:proof_finite_action} uses the classical two-point method with an additional change-of-measure trick to a common distribution. Second, this trick does not suffice to give \Cref{cor:sample_complexity}: when passing through the common distribution to exchange the order of expectations, the inner product is always of the scale $\widetilde{\Theta}(1/\sqrt{d})$ but no other intermediate scales $1/\sqrt{d} \ll \varepsilon \ll 1$ as in \Cref{cor:sample_complexity}. See \Cref{sec:proof_discussion} for details. 

\subsection{Unit sphere vs unit ball} \label{sec:ball}
In this section, we relax the assumption $\theta^\star\in \mathbb{S}^{d-1}$ and investigate the statistical complexity of ridge bandits when $\theta^\star \in \mathbb{B}^d$. The following theorem shows that it is equivalent to think of the unit ball as a union of spheres with different radii. 

\begin{theorem}\label{thm:unit_ball}
Suppose the link function $f$ satisfies the monotonicity condition in \Cref{assump:main}, and $f'(x)/f'(y)\le C$ as long as $1/c\le x/y\le c$ for some constants $c,C>1$. Then the following upper and lower bounds hold for the minimax regret over $\theta^\star \in \mathbb{B}^d$: 
\begin{align*}
\mathfrak{R}_T^\star(f,d) &\lesssim \max_{r\in [0,1]}\min\left\{\frac{f(r)}{r^4}d^2\int_{r/\sqrt{d}}^{r/2} \frac{\mathrm{d}(x^2)}{\max_{r/\sqrt{d}\le y\le x}\min_{z\in [(1-\kappa)y,(1+\kappa)y]} [f'(z)]^2 } + d\sqrt{T}, Tf(r)\right\}, \\
\mathfrak{R}_T^\star(f,d) &\gtrsim \max_{r\in [0,1]}\min\left\{\frac{f(r)}{r^2}d\int_{r/\sqrt{d}}^{r/2} \frac{\mathrm{d}(x^2)}{\max\{f(x)^2,f(-x)^2\}} + d\sqrt{T}, Tf(r)\right\}, 
\end{align*}
where $\kappa\in (0,1/4)$ is any fixed parameter, and the hidden factors depend only on $(c,C,\kappa)$. 
\end{theorem}

The sample complexity for estimation could be obtained in a similar manner, and we omit the details. For $f(x)=|x|^p$ with $p>0$, the above theorem shows that $\mathfrak{R}_T^\star(f,d)\asymp \min\{\sqrt{d^{\max\{2,p\}}T}, T\}$, matching the result in \cite{huang2021optimal}. Note that because of an additional maximum over $r\in [0,1]$, the minimax regret over the unit ball only exhibits one elbow at $T\asymp d^{\max\{2,p\}}$, in contrast to two elbows in \Cref{fig:phase_transition} over the unit sphere. The assumption in \Cref{thm:unit_ball} is also stronger than \Cref{assump:learning}, for we need \Cref{assump:learning} to hold for every function $x\in [0,1]\mapsto f(rx)$ with $r>0$. 

If $r := \|\theta\|_2 \in [0,1]$ is known, the proof of \Cref{thm:unit_ball} adapts from our upper and lower bounds for the unit sphere after proper scaling, and we simply take the worst case radius $r\in [0,1]$. It then remains to find an estimate $\widehat{r}$ of $r$ such that $r\in [\widehat{r}/4, \widehat{r}]$ with high probability. This step is deferred to \Cref{sec:proof_discussion}, with an additional sample complexity which is negligible compared to \Cref{thm:unit_ball}.


\subsection{Closing the gap between upper and lower bounds}\label{sec:gap} There is a gap in Theorems \ref{inftheorem:integral_ub} and \ref{inftheorem:integral_lb}: the upper bound is in terms of the derivative of $f$, but the lower bound is only in terms of the function value of $f$. We conjecture that the lower bound could be strengthened, due to the following intuition. In the proof of \Cref{lemma:recursion_chi^2}, the distribution $Q_{\calH_t}$ is constructed so that $r_t\sim \calN(0,1)$. In principle, the mean of $r_t$ could be any function $\mu(a_1,r_1,\cdots,a_{t-1},r_{t-1},a_t)$ of the available history, and a natural choice is $r_t \sim \calN(\bE[f(\jiao{\theta^\star, a_t}) \mid \calH_{t-1}], 1)$. Under this choice, the information gain in the recursion becomes $
\mathsf{Var}(f(\jiao{\theta^\star, a_t})\mid \calH_{t-1})$, with expected value
\begin{align*}
    \bE[\mathsf{Var}(f(\jiao{\theta^\star, a_t})\mid \calH_{t-1})] \lesssim \max_{y\le \varepsilon_t} [f'(y)]^2\cdot \bE[\mathsf{Var}(\jiao{\theta^\star, a_t}\mid \calH_{t-1})] \le \frac{1}{d}\max_{y\le \varepsilon_t} [f'(y)]^2. 
\end{align*}
Proceeding with this intuition will give a lower bound of a similar form to \Cref{inftheorem:integral_ub}. However, a formal argument will require that the above upper bound holds with high probability rather than in expectation, a challenging claim that involves a complicated posterior distribution of $\theta^\star$. We leave it as an open direction, but give a special example where the high probability argument is feasible using the Brascamp--Lieb inequality on manifolds \cite{kolesnikov2016riemannian}. 

\begin{theorem}\label{thm:burnin_linear_bandit}
For the linear bandit $f(x) = \mathrm{id}(x) = x$ with dimension $d$, it holds that
\begin{align*}
    T_{\text{\rm burn-in}}^\star(\mathrm{id}, d) \gtrsim d^2. 
\end{align*}
\end{theorem}
Note that the lower bound $T^\star(\mathrm{id},d,\varepsilon)\gtrsim d^2$ shown in \cite{wagenmaker2022reward} only works for a small error $\varepsilon$ (say $\varepsilon\le 0.1$), due to an intrinsic limitation of the hypercube structure used in Assouad's lemma. In contrast, \Cref{thm:burnin_linear_bandit} shows the same $\widetilde{\Omega}(d^2)$ lower bound for $\varepsilon=1/2$ (or any fixed $\varepsilon<1$), which improves over the lower bound $\widetilde{\Omega}(d)$ in \Cref{inftheorem:integral_lb} for linear bandits. 




\paragraph{Information-directed sampling.} We have shown the suboptimality for two types of algorithms, but we are not able to understand the performance of the information-directed sampling algorithm: given $\theta^\star\sim \mathsf{Unif}(\mathbb{S}^{d-1})$, this algorithm chooses $a_t = \argmax_{a \in \mathbb{B}^d} I(\theta^\star; r_t(a) \mid \calH_{t-1})$ (an approximate maximizer also suffices). In other words, this algorithm always chooses the action that provides the most information about $\theta^\star$. Here the core of the challenge is also the understanding of the posterior distribution of $\theta^\star$. 

\paragraph{Upper bound with finitely many actions.} \Cref{thm:finite_action} only shows that for \emph{some} finite action set, the burn-in cost does not benefit from a smaller action set. It is unknown that whether a class of finite action sets satisfying certain conditions will make the burn-in cost significantly smaller. 

\paragraph{More general class of reward functions.} Our arguments for both the upper and lower bounds depend crucially on the form of ridge functions, and in particular, a single inner product that fully characterizes the current progress of learning. It is interesting to generalize our results to other reward functions, such as a linear combination of ridge functions $\bE[r_t]=\sum_{i=1}^m f_i(\jiao{\theta_i^\star, a_t})$. 

%% file: auxiliary_lemma.tex
\section{Auxiliary lemmas}
The first two lemmas establish concentration and anti-concentration properties of vectors sampled uniformly on the surface of a high dimensional sphere.

\begin{lemma} \label{lemma:nontrivialcorr}
Let $V$ be an $m$-dimensional subspace of $\mathbb{R}^d$ with $m\le (1-\delta_1)d$, and $v\sim \mathsf{Unif}(\mathbb{S}^{d-1}\cap V^\perp)$. If $\theta\in\mathbb{S}^{d-1}$ and $\|\mathrm{Proj}_{V^\perp}(\theta)\|_2 \ge \delta_2 > 0$, then for $0<\alpha<\beta$, 
\begin{align*}
\bP\left( \jiao{\theta, v} \in \left[\frac{\alpha}{\sqrt{d}}, \frac{\beta}{\sqrt{d}} \right] \right) \ge c, 
\end{align*}
where $c=c(\alpha,\beta,\delta_1,\delta_2)>0$ is an absolute constant depending only on $(\alpha,\beta,\delta_1,\delta_2)$. 
\end{lemma}
\begin{proof}
Without loss of generality assume that $V^\perp = \text{span}(e_1,\cdots,e_{d-m})$. The random vector $v$ then has the same distribution as $(g,\textbf{0}_{m})/\|g\|_2$, with $g\sim \calN(0,I_{d-m})$. Consequently, $\jiao{\theta,v}$ is distributed as $\lambda g_1/\|g\|_2$, with $\lambda = \|\mathrm{Proj}_{V^\perp}(\theta)\|_2 \in [\delta_2,1]$. Note that $g_1/\|g\|_2$ is the one-dimensional marginal of $\mathsf{Unif}(\mathbb{S}^{d-m-1})$, by \cite[Section 2]{bubeck2016testing}, the density $f_{d-m}$ of $g_1/\|g\|_2$ is given by
\begin{align}\label{eq:spherical_density}
    f_{d-m}(x) = \frac{\Gamma((d-m)/2)}{\Gamma((d-m-1)/2)\sqrt{\pi}}(1-x^2)^{(d-m-3)/2}. 
\end{align}
As $\Gamma(x+1/2)/\Gamma(x)=\Theta(\sqrt{x})$ as $x\to\infty$ and $m\le (1-\delta_1)d$, it holds that for $\alpha/\sqrt{d}\le x\le 2\beta/\sqrt{d}$, 
\begin{align*}
f_{d-m}(x) \ge \sqrt{d}\cdot c_1(\alpha,\beta,\delta_1)
\end{align*}
for some absolute constant $c_1(\alpha,\beta,\delta_1)>0$. Consequently, 
\begin{align*}
\bP\left( \jiao{\theta, v} \in \left[\frac{\alpha}{\sqrt{d}}, \frac{\beta}{\sqrt{d}} \right] \right) = \bP\left( \frac{g_1}{\|g\|_2} \in \left[\frac{\alpha}{\lambda\sqrt{d}}, \frac{\beta}{\lambda\sqrt{d}} \right] \right) \ge \frac{\beta-\alpha}{\lambda\sqrt{d}}\cdot \sqrt{d}c_1(\alpha,\beta,\delta_1) \ge c(\alpha,\beta,\delta_1,\delta_2), 
\end{align*}
where we recall that $\delta_2\le \lambda \le 1$. 
\end{proof}

\begin{lemma}\label{lemma:nolargeinnerproduct}
Let $\theta\sim \mathsf{Unif}(\mathbb{S}^{d-1})$, and $v\in \mathbb{S}^{d-1}$ be a fixed unit vector. Then for $\delta\in (0,1/2)$, there exists absolute constant $c_0(d) = 1/2 + o_d(1) > 0$ such that 
\begin{align*}
\bP\left(|\jiao{\theta, v}| > \sqrt{\frac{\log(2/\delta)}{c_0(d) d}}\right) \le \delta. 
\end{align*}
We take $c_0 = \inf_{d\in \mathbb{N}} c_0(d) >0$ to be a dimension-independent absolute constant. 
\end{lemma}
\begin{proof}
By rotational invariance we may assume that $v= e_1$. The proof of \Cref{lemma:nontrivialcorr} shows that the density of $\theta_1$ is $f_d$ in \eqref{eq:spherical_density}. By \cite[Lemma 1]{sodin2007tail}, it holds that for $t>0$, 
\begin{align*}
\bP(\theta_1 > t) \le (1+o_d(1))(1-\Phi(t\sqrt{d})) \le (1+o_d(1)) \exp\left( -\frac{dt^2}{2} \right), 
\end{align*}
where $\Phi(\cdot)$ is the CDF of the standard normal distribution, and the final step is due to the Gaussian tail bound $1-\Phi(x)\le \exp(-x^2/2)$ for $x\ge 0$. The proof is completed by plugging in the given value of $t$ and using $\bP(|\theta_1|>t) = 2\bP(\theta_1>t)$ due to the symmetry.  
\end{proof}

The next few lemmas review several statistical tools for proving minimax lower bounds. Here we assume that $\theta\in\Theta$ is an unknown parameter, the learner observes $X\sim P_\theta$, and $L: \Theta\times \calA\to \bR_+$ is a non-negative loss function with a generic action space $\mathcal{A}$. 

\begin{lemma}[Le Cam's two-point method, see \cite{Tsybakov2009}]\label{lemma:twopoint}
Suppose there are two parameters $\theta_0, \theta_1\in \Theta$ such that
\begin{align*}
    \inf_{a} (L(\theta_0, a) + L(\theta_1,a)) \ge \Delta. 
\end{align*}
Then the following Bayes risk lower bound holds: 
\begin{align*}
    \inf_{T(\cdot)} \bE_{u\sim \mathsf{Unif}(\{0,1\})}\bE_{\theta_u}[L(\theta_u, T(X))] \ge \frac{\Delta}{2}\left(1 - \|P_{\theta_1} - P_{\theta_2}\|_{\text{\rm TV}}\right). 
\end{align*}
\end{lemma}

\begin{lemma}[Assouad's lemma, see \cite{Tsybakov2009}]\label{lemma:Assouad}
Suppose a collection of parameters $\{\theta_u\}_{u\in \{\pm 1\}^d}$ satisfy
\begin{align*}
    \inf_{a} (L(\theta_u, a) + L(\theta_{u'},a)) \ge \Delta\cdot \sum_{i=1}^d \1(u(i)\neq u'(i)).
\end{align*}
Then the following Bayes risk lower bound holds: 
\begin{align*}
\inf_{T(\cdot)} \bE_{u\sim \mathsf{Unif}(\{\pm 1\}^{d})}\bE_{\theta_u}[L(\theta_u, T(X))] \ge \frac{\Delta d}{4}\cdot \bE_{u\sim \mathsf{Unif}(\{\pm 1\}^{d})}\left[\exp\left( -\frac{1}{d}\sum_{i=1}^d D_{\text{\rm KL}}(P_{\theta_u}\|P_{\theta_{u^{\oplus i}}}) \right)\right],
\end{align*}
where $u^{\oplus i}$ flips the $i$-th coordinate of $u$. 
\end{lemma}

\begin{lemma}[Generalized Fano's inequality, see \cite{duchi2013distance,chen2016bayes}]\label{lemma:Fano}
Let $\pi$ be any probability distribution over $\Theta$. For any $\Delta>0$, define
\begin{align*}
p_{\Delta} := \sup_{a} \pi\{ \theta\in \Theta: L(\theta,a)\le \Delta \}. 
\end{align*}
Then for $\theta\sim \pi$, the following Bayes risk lower bound holds: 
\begin{align*}
\inf_{T(\cdot)}\bE_{\pi} \bE_{\theta}[L(\theta, T(X))] \ge \Delta\left(1 - \frac{I(\theta;X) + \log 2}{\log(1/p_\Delta)}\right).
\end{align*}
\end{lemma}

Now consider a fixed design regression with $y_i = f^\star(x_i) + z_i$ for unknown $f^\star\in \calF$, fixed $x_1,\cdots,x_n$, and independent noises $z_1,\cdots,z_n\sim \calN(0,1)$. The least squares estimator of $f^\star$ is given by
\begin{align*}
\widehat{f}^{\text{LS}} = \argmin_{f\in \calF} \sum_{i=1}^n (y_i - f(x_i))^2. 
\end{align*}
The next lemma provides a general statistical guarantee for least squares estimators. 

\begin{lemma}[Corollary of Theorem 9.1 of \cite{van2000empirical}]\label{lemma:least_squares} There exists an absolute constant $c>0$ such that for all $\delta > \delta_n$, 
\begin{align*}
    \bP\left( \sum_{i=1}^n (\widehat{f}^{\text{\rm LS}}(x_i) - f^\star(x_i))^2 > n\delta^2 \right) \le c\exp\left(-\frac{n\delta^2}{c}\right), 
\end{align*}
where $\delta_n>0$ is the solution to 
\begin{align*}
    \sqrt{n}\delta_n^2 \asymp \delta_n + \int_0^{\delta_n} \sqrt{\log N(u, \calF_n(\delta_n), L_2(P_n))} \mathrm{d}u. 
\end{align*}
Here $\log N(u,\calF,d)$ denotes the metric entropy of function class $\calF$ under radius $u$ and metric $d$, $P_n$ is the empirical distribution on $\{x_1,\cdots,x_n\}$, and $\calF_n(\delta_n)$ is the localized function class
\begin{align*}
    \calF_n(\delta_n) := \left\{f\in \calF: \frac{1}{n}\sum_{i=1}^n (f(x_i) - f^\star(x_i))^2 \le \delta_n^2 \right\}. 
\end{align*}

In particular, there exists an absolute constant $C>0$ such that
\begin{align*}
\bE\left[\sum_{i=1}^n (\widehat{f}^{\text{\rm LS}}(x_i) - f^\star(x_i))^2\right] \le C(1 + n\delta_n^2), 
\end{align*}
\end{lemma}

Finally, we prove a simple inequality which is useful for the lower bound proof in \Cref{sec:proof_lower_bound}. 

\begin{lemma}\label{lemma:quadratic_inequality}
If $x_+^2\le Ax+B$ for some $A, B>0$, then $x\le A+\sqrt{B}$. 
\end{lemma}
\begin{proof}
If $x>A+\sqrt{B}$, then
\begin{align*}
    x_+^2 = x^2 > (A+\sqrt{B})x = Ax + \sqrt{B}x > Ax +  B,
\end{align*}
which is a contradiction. 
\end{proof}

%% file: appendix_section2.tex
\section{Deferred proofs in Section \ref{sec:lower_bound}}\label{sec:proof_lower_bound}

In this section, we prove the lower bounds in \Cref{inftheorem:learning_lb}. Both lower bounds will be proved via a hypothesis testing argument, which we detail next. 

The hypotheses are constructed as follows: for each $u\in \{\pm 1\}^{d-1}$, we associate a vector
\begin{align*}
    \theta_u^\star = \left(u_1\delta, \cdots, u_{d-1}\delta, \sqrt{1 - (d-1)\delta^2}\right) \in \mathbb{S}^{d-1}, 
\end{align*}
with parameter $\delta\in (0, 1/\sqrt{d}]$ to be specified later. In words, we embed a hypercube of dimension $d-1$ on the unit sphere in $d$ dimensions.

We apply Assouad's lemma (cf. \Cref{lemma:Assouad}). If $L(\theta^\star,a) = f(1)-f(\jiao{\theta^\star,a})$, then
\begin{align*}
    L(\theta_u^\star, a) + L(\theta_{u'}^\star, a) &= f(1) - f(\jiao{\theta_u^\star,a}) + f(1) - f(\jiao{\theta_{u'}^\star,a}) \\
    &\stepa{\ge} c_1\left(2 - |\jiao{\theta_u^\star, a}| - |\jiao{\theta_{u'}^\star, a}|\right) \\
    &\stepb{\ge} c_1\cdot 2\left(1 - \sqrt{1-\delta^2\sum_{i=1}^{d-1} \1(u(i)\neq u'(i)})\right) \\
    &\stepc{\ge} c_1\delta^2\sum_{i=1}^{d-1}\1(u(i)\neq u'(i)), 
\end{align*}
where (a) uses
\begin{align}\label{eq:diff_lower_bound}
f(1) - f(x) \ge f(1) - f(\max\{x,1-\gamma\}) \ge c_f (1- \max\{x,1-\gamma\}) \ge \frac{\gamma c_f}{2}(1-x)
\end{align}
for all $x\in [-1,1]$, with $c_1 = \gamma c_f/2$, (b) plugs in the minimizer $a^\star = (\theta_u^\star + \theta_{u'}^\star)/\|\theta_u^\star + \theta_{u'}^\star\|_2$, and (c) is due to $\sqrt{1-x}\le 1-x/2$ for $x\in [0,1]$. Consequently, the premise of \Cref{lemma:Assouad} holds for $\Delta = c_1\delta^2$.

To upper bound the KL divergence, for each $u\in \{\pm 1\}^{d-1}$ and $i\in [d-1]$, it is clear that
\begin{align*}
    D_{\text{\rm KL}}(P_{\theta_u^\star}^{T}\|P_{\theta_{u^{\oplus i}}^\star}^{T}) &= \frac{1}{2} \bE_{\theta_u^\star}\sum_{t=1}^T(f(\jiao{\theta_u^\star, a_t})- f(\jiao{\theta_{u^{\oplus i}}^\star, a_t}))^2 \\
    &\le \frac{L^2}{2} \bE_{\theta_u^\star}\sum_{t=1}^T(\jiao{\theta_u^\star, a_t}- \jiao{\theta_{u^{\oplus i}}^\star, a_t})^2 \\
    &= c_2\delta^2\cdot \bE_{\theta_u^\star}\sum_{t=1}^T a_t(i)^2, 
\end{align*}
thanks to \Cref{assump:LB}, with $c_2 = 2L^2$. 

To prove the sample complexity lower bound in \Cref{inftheorem:learning_lb}, we choose $f(x)=x$ in the definition of $L(\theta^\star,a)$. Lemma \ref{lemma:Assouad} then gives that
\begin{align}\label{eq:lowerbound_sample_compl}
    \inf_{\widehat{\theta}_T}\bE_{u\sim \mathsf{Unif}(\{\pm 1\}^{d})} \bE_{\theta_u}[1-\jiao{\theta_u, \widehat{\theta}_T}] &\ge \frac{c_1(d-1)\delta^2}{4}\cdot \bE_{u\sim \mathsf{Unif}(\{\pm 1\}^{d})}\left[\exp\left(-\frac{c_2\delta^2}{d-1}\sum_{i=1}^{d-1}\bE_{\theta_u^\star}\sum_{t=1}^T a_t(i)^2\right)\right] \nonumber\\
    &\ge \frac{c_1(d-1)\delta^2}{4}\cdot\exp\left(-\frac{c_2T\delta^2}{d-1}\right), 
\end{align}
where the second step follows from $\|a_t\|_2\le 1$ for all $t$. Consequently, choosing $\delta = \sqrt{\varepsilon/d}<1/\sqrt{d}$, \eqref{eq:lowerbound_sample_compl} shows that $\inf_{\widehat{\theta}_T}\sup_{\theta^\star\in \mathbb{S}^{d-1}} \bE_{\theta^\star}[1-\jiao{\theta^\star,\widehat{\theta}_T}]=\Omega(\varepsilon)$ for the choice of $T=d^2/\varepsilon$. This proves the sample complexity lower bound in \Cref{inftheorem:learning_lb}. 

The lower bound proof for the minimax regret in \Cref{inftheorem:learning_lb} requires several additional steps. Note that here $L(\theta^\star,a)=f(1)-f(\jiao{\theta^\star,a})$ is the per-step regret, therefore \Cref{lemma:Assouad} gives that
\begin{align}\label{eq:first_lower_bound}
    \mathfrak{R}_T^\star(f,d) \ge \frac{c_1(d-1)T\delta^2}{4}\cdot  \bE_{u\sim \mathsf{Unif}(\{\pm 1\}^{d-1})}\left[\exp\left( -\frac{c_2\delta^2}{2(d-1)} \bE_{\theta_u^\star}\sum_{t=1}^T\sum_{i=1}^{d-1}a_t(i)^2 \right)\right]. 
\end{align}
To proceed, we will prove a different lower bound of $\mathfrak{R}_T^\star(f,d)$ and combine it with \eqref{eq:first_lower_bound}. The second lower bound essentially says that, any good learner should put a large weight of the action on the last component: for every $u\in \{\pm 1\}^{d-1}$, 
\begin{align}\label{eq:second_lower_bound}
\mathfrak{R}_T^\star(f,d)&\ge \bE_{\theta_u^\star}\left[\sum_{t=1}^T (f(1) - f(\jiao{\theta_u^\star, a_t}))\right] \nonumber \\
&\stepd{\ge} c_1\cdot \bE_{\theta_u^\star}\left[\sum_{t=1}^T (1 - |\jiao{\theta_u^\star, a_t}|)\right] \nonumber\\
&\stepe{\ge} \frac{c_1}{8}\cdot \bE_{\theta_u^\star}\left[\sum_{t=1}^T \frac{\left(\sum_{i\le d-1} a_t(i)^2 - (d-1)\delta^2\right)_+^2}{\sum_{i\le d-1} a_t(i)^2} \right] \nonumber \\
&\stepf{\ge} \frac{c_1}{8}\cdot \frac{(\bE_{\theta_u^\star} \sum_{t=1}^T \sum_{i=1}^{d-1} a_t(i)^2 - (d-1)T\delta^2)_+^2}{\bE_{\theta_u^\star} \sum_{t=1}^T \sum_{i=1}^{d-1} a_t(i)^2}, 
\end{align}
where (d) is given by \eqref{eq:diff_lower_bound}, (e) follows from the fact that if $\sum_{i\le d-1}a(i)^2 = (d-1)\delta^2 + \eta$ with $\eta\ge 0$, then
    \begin{align*}
        |\jiao{\theta_u^\star, a}| &\le \sqrt{\sum_{i\le d-1}\theta_u^\star(i)^2}\sqrt{\sum_{i\le d-1}a(i)^2} + |\theta_u^\star(d)\cdot a(d)| \\
        &\le \sqrt{(d-1)\delta^2\left((d-1)\delta^2+\eta \right)} + \sqrt{\left(1 - (d-1)\delta^2\right)\left(1-(d-1)\delta^2-\eta\right)} \\
        &= 1 - \frac{1}{2}H^2\left(\mathsf{Bern}\left((d-1)\delta^2\right), \mathsf{Bern}\left((d-1)\delta^2+\eta\right)\right) \\
        &\le 1 - \frac{1}{2}(\sqrt{(d-1)\delta^2+\eta} - \sqrt{(d-1)\delta^2})^2 \\
        &\le 1 - \frac{\eta^2}{8((d-1)\delta^2+\eta)},
    \end{align*}
where the last step follows from $\sqrt{x}-\sqrt{y}\ge (x-y)/(2\sqrt{x})$ for $x\ge y\ge 0$, and (f) is due to the joint convexity of $(x,y)\in \mathbb{R}_+^2 \mapsto x^2/y$, and the convexity of $x\mapsto x_+:=\max\{x,0\}$. 

Now we combine the lower bounds \eqref{eq:first_lower_bound} and \eqref{eq:second_lower_bound}. By \eqref{eq:second_lower_bound} and \Cref{lemma:quadratic_inequality}, it holds that for every $u\in \{\pm1\}^{d-1}$, 
\begin{align*}
\bE_{\theta_u^\star}\sum_{t=1}^T\sum_{i=1}^{d-1}a_t(i)^2 &\le (d-1)T\delta^2 + \frac{8\mathfrak{R}_{T}^\star(f,d)}{c_1} + \sqrt{(d-1)T\delta^2\times \frac{8\mathfrak{R}_{T}^\star(f,d)}{c_1}} \\
&\le c_3\left(dT\delta^2 + \mathfrak{R}_{T}^\star(f,d) \right). 
\end{align*}
Plugging this upper bound into \eqref{eq:first_lower_bound} leads to
\begin{align*}
\mathfrak{R}_T^\star(f,d)\ge c_4dT\delta^2\exp\left(-c_5T\delta^4 - \frac{c_6\delta^2 \mathfrak{R}_{T}^\star(f,d)}{d}\right).
\end{align*}
Choosing $\delta=T^{-1/4}\le 1/\sqrt{d}$ if $T\ge d^2$, we conclude from the above that $\mathfrak{R}_T^\star(f,d)= \Omega(d\sqrt{T})$. 

For the remaining case $T<d^2$, note that $\mathfrak{R}_{kT}^\star(f,d)\le k\mathfrak{R}_{T}^\star(f,d)$ for every positive integer $k$, for the learner can run a learning algorithm of time horizon $T$ for $k$ times independently. Therefore, for $T<d^2$, it holds that
\begin{align*}
    \mathfrak{R}_T^\star(f,d) \ge \frac{\mathfrak{R}_{\lceil d^2/T \rceil T}^\star(f,d)}{\lceil d^2/T\rceil} = \Omega\left(\frac{d^2}{d^2/T}\right) = \Omega(T). 
\end{align*}
This completes the proof of the regret lower bound in \Cref{inftheorem:learning_lb}. 

\begin{remark}
We remark that the same lower bounds also hold for the Bayes risk with the uniform prior $\theta^\star\sim \mathsf{Unif}(\mathbb{S}^{d-1})$. Take the regret lower bound as an example: let $\mathcal{C}\subseteq \mathbb{S}^{d-1}$ be the hypercube used in the above proof, what we have shown is that
\begin{align*}
\bE_{\theta^\star\sim \mathsf{Unif}(\mathcal{C})} \left\{\mathbb{E}_{\theta^\star} \left[ T \cdot \max_{a^\star \in \mathcal{A}} f ( \jiao{\theta^\star, a^\star}) - \sum_{t=1}^T  f ( \jiao{\theta^\star, a_t}) \right] \right\} = \Omega\left( \min\left\{d\sqrt{T}, T\right\} \right). 
\end{align*}
By rotational invariance of $\mathbb{S}^{d-1}$, for every $v\in \mathbb{S}^{d-1}$ the same result holds with $\mathcal{C}$ replaced by $\mathcal{C}_v$, the new cube with each vertex of $\mathcal{C}$ rotated by $v$. Consequently, 
\begin{align*}
& \bE_{\theta^\star\sim \mathsf{Unif}(\mathbb{S}^{d-1})} \left\{\mathbb{E}_{\theta^\star} \left[ T \cdot \max_{a^\star \in \mathcal{A}} f ( \jiao{\theta^\star, a^\star}) - \sum_{t=1}^T  f ( \jiao{\theta^\star, a_t}) \right] \right\} \\
&= \bE_{v\sim \mathsf{Unif}(\mathbb{S}^{d-1})}\bE_{\theta^\star\sim \mathsf{Unif}(\mathcal{C}_v)} \left\{\mathbb{E}_{\theta^\star} \left[ T \cdot \max_{a^\star \in \mathcal{A}} f ( \jiao{\theta^\star, a^\star}) - \sum_{t=1}^T  f ( \jiao{\theta^\star, a_t}) \right] \right\} \\
&= \Omega\left( \min\left\{d\sqrt{T}, T\right\} \right). 
\end{align*}
Therefore, the Bayes regret under a natural uniform prior remains $\Omega(d\sqrt{T})$ asymptotically for ridge bandits with a continuous set of actions, a sharp contrast to the asymptotic $O(\log T)$ regret for multi-armed bandits \cite{lai1985asymptotically}. 
\end{remark}

%% file: appendix_section3.tex
\section{Deferred proofs in Section \ref{sec:upper_bound}}\label{sec:proof_upper_bound}

\subsection{Proof of the integral-form upper bound}
The next lemma shows that finite difference form of the upper bound in \Cref{thm:ub_burnincost_formal} implies the integral form of the upper bound in \Cref{inftheorem:integral_ub}. 

\begin{lemma}\label{lemma:difference_to_derivative}
For any $\kappa>0$ and $x_0\in (0,1)$, there exists $(\kappa_1,\kappa_2)$ such that for the sequence $\{ \varepsilon_i \}_{i=1}^m$ defined in \Cref{thm:ub_burnincost_formal} under $(\kappa_1,\kappa_2)$, then
\begin{align*}
    \sum_{i=1}^{m} \frac{1}{\varepsilon_i^2} \le C(x_0, \kappa) \cdot d^2 \int_{ \sqrt{1/d}}^{x_0} \frac{\mathrm{d}(x^2)}{\max_{y \in [ 1/\sqrt{d}, x]} \min_{z \in [ (1-\kappa)y, (1+\kappa)y ]} \left( f' (z) \right)^2}. 
\end{align*}
Here $C(x_0, \kappa)$ is an absolute constant depending only on $(x_0, \kappa)$. 
\end{lemma}
\begin{proof}
We choose $\kappa_1\in (0,\kappa/2)$, and a sufficiently small $\kappa_2>0$ such that $c_2<1$ in \Cref{thm:ub_burnincost_formal}. In this case, for $1\le i\le d_0$, we have
\begin{align*}
\frac{1}{\varepsilon_i^2} &\lesssim d\int_{1/\sqrt{d}}^{1/[(1-\kappa/2)\sqrt{d}]} \frac{\mathrm{d}(x^2)}{\min_{z\in [1/\sqrt{d}, (1+\kappa_1/2)/\sqrt{d} ]}(f(z+c_1/\sqrt{d})-f(z))^2} \\
&\stepa{\le} d\int_{1/\sqrt{d}}^{1/[(1-\kappa/2)\sqrt{d}]} \frac{\mathrm{d}(x^2)}{\max_{y\in [1/\sqrt{d},x]}\min_{z\in [(1-\kappa/2)y, (1+\kappa/2)y]}(f(z+c_1/\sqrt{d})-f(z))^2} \\
&\stepb{\lesssim} d^2\int_{1/\sqrt{d}}^{1/[(1-\kappa/2)\sqrt{d}]} \frac{\mathrm{d}(x^2)}{\max_{y\in [1/\sqrt{d},x]}\min_{z\in [(1-\kappa)y, (1+\kappa)y]}(f'(z))^2}. 
\end{align*}
Here (a) follows from the observation that for all $1/\sqrt{d}\le y\le 1/[(1-\kappa/2)\sqrt{d}]$, one has $$[(1-\kappa/2)y, (1+\kappa/2)y]\supseteq [1/\sqrt{d}, (1+\kappa_1/2)/\sqrt{d}]$$
as long as $\kappa_1\le \kappa$; (b) follows from the inequality $|f(x)-f(y)|\ge |x-y|\min_{z\in [y,x]}|f'(z)|$ for $x\ge y$, and that $(1+\kappa/2)y + c_1/\sqrt{d}\le (1+\kappa/2)y + \kappa_1/\sqrt{d}\le (1+\kappa)y$ as long as $\kappa_1\le \kappa/2$.

Similarly, for $d_0+1\le i\le m$, we have
\begin{align*}
\frac{1}{\varepsilon_i^2} &\stepc{\lesssim} d\int_{\sqrt{(i-1.5)/d}}^{\sqrt{(i-1)/d}} \frac{\mathrm{d}(x^2)}{\max_{c_2/\sqrt{d}\le y\le (1-\kappa_2)x}\min_{z\in [(1-\kappa_1)y, (1+\kappa_1)y]}(f(z+c_1/\sqrt{d}) - f(z))^2} \\
&\stepd{\asymp} d\int_{(1-\kappa_2)\sqrt{(i-1.5)/d}}^{(1-\kappa_2)\sqrt{(i-1)/d}} \frac{\mathrm{d}(x^2)}{\max_{c_2/\sqrt{d}\le y\le x}\min_{z\in [(1-\kappa_1)y, (1+\kappa_1)y]}(f(z+c_1/\sqrt{d}) - f(z))^2} \\
&\stepe{\le} d\int_{(1-\kappa_2)\sqrt{(i-1.5)/d}}^{(1-\kappa_2)\sqrt{(i-1)/d}} \frac{\mathrm{d}(x^2)}{\max_{1/\sqrt{d}\le y\le x}\min_{z\in [(1-\kappa_1)y, (1+\kappa_1)y]}(f(z+c_1/\sqrt{d}) - f(z))^2} \\
&\stepf{\lesssim} d^2\int_{(1-\kappa_2)\sqrt{(i-1.5)/d}}^{(1-\kappa_2)\sqrt{(i-1)/d}} \frac{\mathrm{d}(x^2)}{\max_{1/\sqrt{d}\le y\le x}\min_{z\in [(1-\kappa)y, (1+\kappa)y]}(f'(z))^2}.
\end{align*}
Here (c) follows from $x\le \sqrt{(i-1)/d}$ and the monotonicity of the denominator in $x$; (d) is an affine change of variable $(1-\kappa_2)x\mapsto x$; (e) is due to our choice of $(\kappa_1,\kappa_2)$ that $c_2<1$; (f) uses the same reasoning as (b), with the observation that $(1+\kappa_1)y+c_1/\sqrt{d}\le (1+\kappa)y$ as long as $\kappa_1\le \kappa/2$. 

A combination of the above inequalities gives that
\begin{align*}
    \sum_{i=1}^{m} \frac{1}{\varepsilon_i^2} \lesssim d^2 \left(\int_{1/\sqrt{d}}^{1/[(1-\kappa/2)\sqrt{d}]}  + \int_{(1-\kappa_2)\sqrt{(d_0-0.5)/d}}^{(1-\kappa_2)\sqrt{(m-1)/d}} \right)\frac{\mathrm{d}(x^2)}{\max_{y \in [ 1/\sqrt{d}, x]} \min_{z \in [ (1-\kappa)y, (1+\kappa)y ]} \left( f' (z) \right)^2}. 
\end{align*}
To conclude the final result, simply note that the integral lower limits are at least $1/\sqrt{d}$ for $d_0\ge 2$ and a sufficiently small $\kappa_2>0$, and the upper limits are at most $x_0$ by the choice of $m$ in \Cref{thm:ub_burnincost_formal}. 
\end{proof}

\subsection{Proof of \Cref{lemma:iaht}}
Note that $\overline{r}\sim \calN(f(\jiao{\theta^\star, v}), \varepsilon^2/[2\log(2/\delta)])$, therefore
\begin{align*}
\bP(|\overline{r} - f(\jiao{\theta^\star, v})|> \varepsilon) \le 2\exp\left(-\frac{2\log(2/\delta)}{2}\right) = \delta. 
\end{align*}
In the sequel we condition on the good event $|\overline{r} - f(\jiao{\theta^\star, v})| \le \varepsilon$ which happens with probability at least $1-\delta$. If $\jiao{\theta^\star, v}\in [(1+\kappa_1)/\sqrt{d}, (1+2\kappa_1)/\sqrt{d}]$, then $x=\jiao{\theta^\star, v}$ is a successful witness of the projection based test. If $\jiao{\theta^\star, v}\notin [1/\sqrt{d}, (1+3\kappa_1)/\sqrt{d}]$ and the test returns \textsf{True}, then the witness $x\in [(1+\kappa_1)/\sqrt{d}, (1+2\kappa_1)/\sqrt{d}]$ satisfies
\begin{align*}
|f(\jiao{\theta^\star, v}) - f(x)|\le |f(\jiao{\theta^\star, v}) - \overline{r}| + |\overline{r} - f(x)| \le 2\varepsilon, 
\end{align*}
which is impossible by the definition of $\varepsilon$ in \eqref{eq:eps0-def} and the monotonicity of $f$. 

\subsection{Proof of \Cref{lemma:gaht}}
Similar to the proof of \Cref{lemma:iaht}, in the sequel we condition on the good event that
\begin{align*}
|\overline{r}_- - f(\jiao{\theta^\star, a_-})| \le \varepsilon, \quad \text{ and }\quad |\overline{r}_+ - f(\jiao{\theta^\star, a_+})| \le \varepsilon, 
\end{align*}
which happens with probability at least $1-\delta$. 

If $\jiao{\theta^\star, v}\in [(1+\kappa_1)/\sqrt{d}, (1+2\kappa_1)/\sqrt{d}]$, we choose the witnesses
\begin{align*}
    z &= \lambda(1-\kappa_2)\jiao{v_{\text{pre}}, \theta^\star} \in \frac{y^\star}{x_{\text{\rm pre}}}\left[x_{\text{\rm pre}}, (1+3\kappa_1)x_{\text{\rm pre}} \right] = \left[y^\star, (1+3\kappa_1)y^\star\right], \\
    x &= \kappa_2^\perp \jiao{\theta^\star, v} \in \left[\frac{(1+\kappa_1)\kappa_2^\perp}{\sqrt{d}}, \frac{(1+2\kappa_1)\kappa_2^\perp}{\sqrt{d}}\right]
\end{align*}
in the projection based test. Note that $\jiao{\theta^\star, a_-} = z - x$ and $\jiao{\theta^\star, a_+} = z + x$, these witnesses pass the test \eqref{eq:test} thanks to the good event. 

If $\jiao{\theta^\star, v}\notin [1/\sqrt{d}, (1+3\kappa_1)/\sqrt{d}]$, again we use the notation $(z,x)$ in the above equation. We also assume by contradiction that the witnesses $(z',x')$ exist, then
\begin{align*}
|f(z+x) - f(z'+x')| &\le |f(\jiao{\theta^\star, a_+}) - \overline{r}_+| + |\overline{r}_+ - f(z'+x')| \le 2\varepsilon, \\
|f(z-x) - f(z'-x')| &\le |f(\jiao{\theta^\star, a_-}) - \overline{r}_-| + |\overline{r}_- - f(z'-x')| \le 2\varepsilon.  
\end{align*}
Since $y^\star \ge \kappa_4/\sqrt{d} = (\kappa_1^{-1}+2)\kappa_2^\perp / \sqrt{d}$, we have
\begin{align*}
z' + x', z' - x' \in \left[y^\star - \frac{(1+2\kappa_1)\kappa_2^\perp}{\sqrt{d}}, (1+3\kappa_1)y^\star + \frac{(1+2\kappa_1)\kappa_2^\perp}{\sqrt{d}} \right] \subseteq \left[(1-4\kappa_1)y^\star, (1+4\kappa_1)y^\star\right]. 
\end{align*}
Then by definition of $\varepsilon$ in \eqref{eq:obj} and the monotonicity of $f$,
\begin{align*}
|(z+x) - (z'+x')| \le \frac{\kappa_3}{\sqrt{d}}, \quad |(z-x) - (z'-x')| \le \frac{\kappa_3}{\sqrt{d}}. 
\end{align*}
Finally, by triangle inequality, 
\begin{align*}
|x-x'| \le \frac{|(z+x) - (z'+x')| + |(z-x) - (z'-x')|}{2} \le \frac{\kappa_3}{\sqrt{d}} = \frac{\kappa_1\kappa_2^\perp}{\sqrt{d}}, 
\end{align*}
which is a contradiction to $x = \kappa_2^\perp\jiao{\theta^\star,v}\notin [\kappa_2^\perp/\sqrt{d}, (1+3\kappa_1)\kappa_2^\perp/\sqrt{d}]$ and $x'\in [(1+\kappa_1)\kappa_2^\perp/\sqrt{d}, (1+2\kappa_1)\kappa_2^\perp/\sqrt{d}]$. 

\subsection{Proof of \Cref{thm:ub_burnincost_formal}}
We prove \Cref{thm:ub_burnincost_formal} based on \Cref{lemma:iaht} and \Cref{lemma:gaht}. We first specify a collection of good events, and show that with probability at least $1-\delta$, these good events simultaneously occur (note that $\kappa_1$ is replaced by $\kappa_1/4$ in \Cref{alg:burn-in}): 
\begin{enumerate}
    \item In each epoch $i\in [m]$, among the first $\log(2m/\delta)/c_0$ while loops, there is at least one sampled $v_i$ with $\jiao{\theta^\star, v_i}\in [(1+\kappa_1/4)/\sqrt{d}, (1+\kappa_1/2)/\sqrt{d}]$. Here $c_0>0$ is given in \Cref{thm:ub_burnincost_formal}. 
    \item The success event in \Cref{lemma:iaht} or \Cref{lemma:gaht} holds for all calls to the \iaht or \gaht algorithm. 
    \item The prerequisites for \Cref{lemma:gaht} (i.e. $\jiao{\theta^\star, v_{\text{\rm pre}}} \in [x_{\text{\rm pre}}, (1+3\kappa_1/4)x_{\text{\rm pre}}]$ and $x_{\text{\rm pre}}\ge \kappa_4/[(1-\kappa_2)\sqrt{d}]$) hold for all epochs at the recursive stage. 
\end{enumerate}
Let $E_{j,i}, j\in [3], i\in [m]$ denote the above $j$-th good event happening in the $i$-th epoch. We analyze the success probabilities separately: 
\begin{enumerate}
    \item For $E_{1,i}$, note that $\cap_{i'<i}\cap_{j=1}^3 E_{j,i'}$ implies that $\jiao{\theta^\star, v_{i'}}\in [1/\sqrt{d}, (1+\kappa_1)/\sqrt{d}]$ for all $i'<i$. Therefore, by the choice of $\kappa_1$ in \Cref{thm:ub_burnincost_formal}, 
    \begin{align*}
    \|\text{Proj}_{\text{span}(v_1,\cdots,v_{i-1})^\perp}(\theta^\star)\|_2^2 &= 1 - \sum_{i'<i}\jiao{\theta^\star, v_{i'}}^2 \ge 1 - (m-1)\left(\frac{1+\kappa_1}{\sqrt{d}}\right)^2 \\
    & > 1 - x_0^2d\cdot \frac{(1+\kappa_1)^2}{d} > 1 - \left(\frac{1+x_0}{2}\right)^2 > \left(\frac{1-x_0}{2}\right)^2, 
    \end{align*}
    so that the condition of \Cref{lemma:nontrivialcorr} holds with $\delta_1 = 1-x_0^2$ and $\delta_2 = (1-x_0)/2$. By the definition of $c_0$ in \Cref{thm:ub_burnincost_formal}, each while loop samples a direction $v_i$ with $\jiao{\theta^\star, v_i}\in [(1+\kappa_1/4)/\sqrt{d}, (1+\kappa_1/2)/\sqrt{d}]$ with probability at least $c_0$, and
    \begin{align*}
        \bP(E_{1,i} \mid \cap_{i'<i}\cap_{j=1}^3 E_{j,i}) \ge 1 - (1-c_0)^{\log(2m/\delta)/c_0} \ge 1 - \frac{\delta}{2m}. 
    \end{align*}
    \item For $E_{2,i}$, each call fails with probability at most $\delta/L$, by the target failure probability set in Algorithm \ref{alg:burn-in}. By $E_{1,i}$, if the first $\log(2m/\delta)/c_0$ calls all succeed, then the $i$-th epoch will break before $\log(2m/\delta)/c_0$ calls. By $E_{3,i}$, the prerequisites of \Cref{lemma:gaht} also hold. Consequently, 
    \begin{align*}
    \bP(E_{2,i} \mid E_{1,i}, E_{3,i}) \ge 1 - \frac{\log(2m/\delta)}{c}\cdot \frac{\delta}{L} = 1 - \frac{\delta}{2m}. 
    \end{align*}
    \item For $E_{3,i}$, the first event $\jiao{\theta^\star, v_{\text{pre}}}\in [x_{\text{\rm pre}}, (1+3\kappa_1/4)x_{\text{\rm pre}}]$ is contained in $E_{2,i-1}$: 
    \begin{align*}
        \jiao{\theta^\star, v_{\text{pre}}} = \frac{1}{\sqrt{i-1}}\sum_{i'<i} \jiao{\theta^\star, v_{i'}} \in \left[\sqrt{\frac{i-1}{d}}, \left(1+\frac{3\kappa_1}{4}\right)\sqrt{\frac{i-1}{d}} \right] = \left[x_{\text{\rm pre}}, \left(1+\frac{3\kappa_1}{4}\right)x_{\text{\rm pre}}\right]. 
    \end{align*}
    The second event simply follows from 
    \begin{align*}
        x_{\text{\rm pre}} \ge \sqrt{\frac{d_0}{d}} \ge \frac{(2+(\kappa_1/4)^{-1})\sqrt{1-(1-\kappa_2)^2}}{(1-\kappa_2)\sqrt{d}} = \frac{\kappa_4}{(1-\kappa_2)\sqrt{d}}
    \end{align*}
    by the definition of $d_0$ and $\kappa_4$. So $\bP(E_{3,i}\mid E_{2,i-1})=1$. 
\end{enumerate}
Consequently, by the union bound, all good events simultaneously occur with probability at least $1-\delta$. In the sequel we condition on all the above good events. 

Next we prove the correctness of the algorithm. In fact, event $E_{2,i}$ implies $\jiao{\theta^\star, v_i}\in [1/\sqrt{d}, (1+\kappa_1)/\sqrt{d}]$, and 
\begin{align*}
    \jiao{\theta^\star, a_0} = \frac{1}{\sqrt{m}}\sum_{i=1}^m \jiao{\theta^\star, v_i} \ge \sqrt{\frac{m}{d}} \ge \sqrt{\frac{x_0^2 d}{d}} = x_0. 
\end{align*}

Finally we analyze the sample complexity of \Cref{alg:burn-in}. By the event $E_{1,i}$, in the $i$-th epoch the certification algorithm is called for at most $\log(2m/\delta)/c_0$ times, and thus the total sample complexity is at most
\begin{align*}
    \sum_{i=1}^m \frac{\log(2m/\delta)}{c_0}\cdot \frac{4\log(4/\delta)}{\varepsilon_i^2} = O\left(\log^2\left(\frac{d}{\delta}\right)\cdot\sum_{i=1}^m \frac{1}{\varepsilon_i^2}\right), 
\end{align*}
which completes the proof of \Cref{thm:ub_burnincost_formal}. 

\subsection{Proof of \Cref{lemma:infinite-armed-bandit}}
We present the algorithm first, followed with the proof of \Cref{lemma:infinite-armed-bandit} by establishing the correctness and the sample complexity upper bound, respectively. 

\subsubsection{Algorithm description}
\begin{algorithm}[htb]
\caption{Adaptive arm identification algorithm}\label{alg:UCBLCB}
\textbf{Input:} parameters $(p,q,\varepsilon)$, error probability $\delta$. 

\textbf{Output:} with probablity $\ge 1-\delta$, an arm $i\in \mathbb{N}$ with $F(X_i)\in [p,q]$. 

Define
\begin{align*}
    L = \left\lceil \frac{2\log(16/\delta)}{\varepsilon^2} + \frac{2\log(8/\delta)}{(q-p-4\varepsilon)^2} \right\rceil.
\end{align*}

Initialize the count $n_i\gets 0$ and empirical mean $\bar{y}_i\gets 0$ for the first $L$ arms $i=1,\cdots,L$. 

\While{{\upshape \textsf{True} }}{
Pull each of the first $L$ arms once and observe $Y_1, \cdots, Y_L$. 

Update the counts and empirical means:
\begin{align*}
\bar{y}_i \gets \frac{n_i\bar{y}_i + Y_i}{n_i+1}, \quad n_i \gets n_i + 1, \quad i=1,\cdots,L. 
\end{align*}

Compute the upper and lower confidence bounds for each arm: 
\begin{align*}
\text{UCB}_i = \bar{y}_i + \sqrt{\frac{2\log(2\pi^2 n_i^2 L/3\delta)}{n_i}}, \quad \text{LCB}_i = \bar{y}_i - \sqrt{\frac{2\log(2\pi^2 n_i^2 L/3\delta)}{n_i}}, \quad 
i=1,\cdots,L. 
\end{align*}

Let $\text{UCB}^{(1)}\le \cdots\le \text{UCB}^{(L)}$ and $\text{LCB}^{(1)}\le \cdots\le \text{LCB}^{(L)}$ be the order statistics. 

\If{there exists $i\in [L]$ with $\text{\rm LCB}_i \ge \text{\rm UCB}^{(\lceil (p+\varepsilon/2)L \rceil)}$ and $\text{\rm UCB}_i \le \text{\rm LCB}^{(\lfloor (q-\varepsilon/2)L \rfloor)}$}{
break the loop and return the arm $i$. 
}
}
\end{algorithm}

The algorithm is displayed in \Cref{alg:UCBLCB}, with simple algorithmic ideas based on upper and lower confidence bounds. The confidence bounds are constructed in the standard way so that each $X_i$ is sandwiched between $\text{LCB}_i$ and $\text{UCB}_i$ with high probability. Therefore, the stopping condition in \Cref{alg:UCBLCB} ensures that the arm $i$ has an empirical CDF in $[p+\varepsilon/2, q-\varepsilon/2]$ based on $L$ i.i.d. arms. By the convergence of the empirical CDF to the true CDF, for $L$ large enough the arm $i$ has a CDF in $[p,q]$, as desired. The following sections make the above intuition rigorous. 

\subsubsection{Proof of correctness}
Let $X^{(1)} \le \cdots\le X^{(L)}$ be the order statistics for the true arm means. We define a good event $E$, which is an intersection of three good events: 
\begin{enumerate}
    \item event $E_1$: $F(X^{(\lceil (p+\varepsilon/2)L \rceil)})\in [p,p+\varepsilon]$, and $F(X^{(\lfloor (q-\varepsilon/2)L \rfloor)})\in [q-\varepsilon,q]$; 
    \item event $E_2$: there is an arm $i\in [L]$ such that $F(X_i)\in [p+2\varepsilon, q-2\varepsilon]$; 
    \item event $E_3$: the confidence bounds are always correct, i.e. $X_i\in [\text{LCB}_i, \text{UCB}_i]$ for all $i\in [L]$ and all iterations. 
\end{enumerate}
We first show that $\bP(E)\ge 1-\delta$. For event $E_1$, note that $F(X_i)$ is uniformly distributed on $[0,1]$ for each $i$, so $\sum_{i=1}^L \mathbbm{1}(F(X_i) < p)$ follows a Binomial distribution $\mathsf{B}(L,p)$. Consequently, 
\begin{align*}
\bP(F(X^{(\lceil (p+\varepsilon/2)L \rceil)}) < p) &= \bP\left( \sum_{i=1}^L \mathbbm{1}(F(X_i) < p) \ge \lceil(p+\varepsilon/2)L\rceil \right) \\
&\le \exp(-2L(\varepsilon/2)^2) \le \frac{\delta}{16}
\end{align*}
by Hoeffding's concentration inequality and the choice of $L$. The other claims are proved similarly, and the union bound gives that $\bP(E_1)\ge 1 - \delta/4$. The analysis of the event $E_2$ is similar: note that $E_2$ contains the event $X^{\lceil (p+q)/2 \rceil} \in [p+2\varepsilon, q-2\varepsilon]$, and the same analysis gives that this event holds with probability at least $1-\delta/4$. 

Finally we look at the event $E_3$. Since $\bar{y}_i\sim \calN(X_i, 1/m)$ when the arm $i$ is pulled $m$ times, the Gaussian tail bound gives that
\begin{align*}
\bP\left( \text{LCB}_i \le X_i \le \text{UCB}_i \right) \ge 1 - 2\exp\left( - \frac{m}{2} \cdot \frac{2\log(2\pi^2 m^2 L/3\delta)}{m}\right) = 1 - \frac{3\delta}{\pi^2m^2L}. 
\end{align*}
Using the union bound over arms $i\in [L]$ and iterations $m\in \mathbb{N}$, we have
\begin{align*}
\bP(E_3) \ge 1 - \sum_{m=1}^\infty \sum_{i=1}^L \frac{3\delta}{\pi^2m^2 L} = 1 - \frac{\delta}{2}. 
\end{align*}
A final union bound then gives that $\bP(E) = \bP(E_1\cap E_2 \cap E_3) \ge 1- \delta$. 

Next we show the correctness of the algorithm given the good event $E$. Let $i$ be the final output of the algorithm, then
\begin{align*}
    X_i &\overset{E_3}{\ge} \text{LCB}_i \ge \text{\rm UCB}^{(\lceil (p+\varepsilon/2)L \rceil)} \overset{E_3}\ge X^{(\lceil (p+\varepsilon/2)L \rceil)} \overset{E_1}{\ge} F^{-1}(p), \\
    X_i &\overset{E_3}{\le} \text{UCB}_i \le \text{\rm LCB}^{(\lfloor (q-\varepsilon/2)L \rfloor)} \overset{E_3}\le X^{(\lfloor (q-\varepsilon/2)L \rfloor)} \overset{E_1}{\le} F^{-1}(q). 
\end{align*}
This shows that $F(X_i)\in [p,q]$, i.e. the algorithm outputs a correct answer with probability at least $1-\delta$. 

\subsubsection{Analysis of sample complexity}
For the sample complexity of \Cref{alg:UCBLCB}, we again condition on the good event $E$ and upper bound the number of iterations. By event $E_2$, there is some $i\in [L]$ such that $X_i \in [F^{-1}(p+2\varepsilon), F^{-1}(q-2\varepsilon)]$. Consider all arms $X_j$ with $X_j\le X^{(\lceil (p+\varepsilon/2)L \rceil)}$. By event $E_1$, all such arms satisfy $X_j\le F^{-1}(p+\varepsilon)$. Therefore, after $m$ iterations with $m$ being the solution to
\begin{align*}
    \sqrt{\frac{2\log(2\pi^2m^2 L/3\delta)}{m}} \le \frac{F^{-1}(p+2\varepsilon) - F^{-1}(p+\varepsilon)}{4}, 
\end{align*}
all such arms have 
\begin{align*}
\text{UCB}_j = \text{LCB}_j + 2\sqrt{\frac{2\log(2\pi^2m^2 L/3\delta)}{m}} &\overset{E_3}{\le} X_j + \frac{F^{-1}(p+2\varepsilon) - F^{-1}(p+\varepsilon)}{2} \\
&\le \frac{F^{-1}(p+2\varepsilon) + F^{-1}(p+\varepsilon)}{2}. 
\end{align*}
Similarly, the event $E_3$ applied to arm $i$ gives that
\begin{align*}
\text{LCB}_i = \text{UCB}_i - 2\sqrt{\frac{2\log(2\pi^2m^2 L/3\delta)}{m}} &\overset{E_3}{\ge} X_i - \frac{F^{-1}(p+2\varepsilon) - F^{-1}(p+\varepsilon)}{2} \\
&\overset{E_2}{\ge} \frac{F^{-1}(p+2\varepsilon) + F^{-1}(p+\varepsilon)}{2}. 
\end{align*} 
This means that the arm $i$ meets the first stopping condition after $m$ iterations, with
\begin{align*}
m = O\left(\frac{\log(1/\delta)}{(F^{-1}(p+2\varepsilon) - F^{-1}(p+\varepsilon))^2}\right). 
\end{align*}
The other stopping condition can be analyzed in an analogous manner. The final sample complexity is then $Lm$, giving the claimed result in \Cref{lemma:infinite-armed-bandit}. 

\subsection{Proof of \Cref{thm:burnin_agnostic}}
Similar to the proof of \Cref{lemma:infinite-armed-bandit}, we describe our algorithm first, followed by the proofs of correctness and the claimed upper bound on the sample complexity. 

\subsubsection{Algorithm description}
The algorithm is similar to the iterative direction search algorithm in \Cref{alg:burn-in}, with certification steps replaced by proper applications of \Cref{alg:UCBLCB}. Let $f_d$ be the density in \eqref{eq:spherical_density}, and $F_d$ be its CDF. For small constants $\kappa_1', \kappa_2'>0$ to be chosen later, the algorithm runs as follows: 
\begin{enumerate}
    \item the first step $i=1$: construct an infinite-armed bandit instance with $F$ being the (unknown) CDF of $f(\jiao{\theta^\star, v})$ with $v\sim \mathsf{Unif}(\mathbb{S}^{d-1})$, invoke \Cref{alg:UCBLCB} with parameters
    \begin{align*}
        p \gets F_d\left(\frac{1}{\sqrt{d}}\right), \quad q \gets F_d\left(\frac{1+\kappa_1'}{\sqrt{d}}\right), \quad \varepsilon \gets \frac{q-p}{5}, \quad \delta \gets \frac{\delta}{d}. 
    \end{align*}
    Let $v_1\in \mathbb{S}^{d-1}$ be the output of \Cref{alg:UCBLCB}.   
    \item the subsequent steps $2\le i\le m:=\lceil x_0^2 d\rceil$: let $v_{\text{pre}} = \frac{1}{\sqrt{i-1}}\sum_{j<i} v_j$, and $\Lambda = \{e^{-j\kappa_2'}: 0\le j\le J\}$ be a geometric grid of $\lambda$ with $J = \lceil \log(\sqrt{d}/\kappa_2')/\kappa_2'\rceil$. For each possible value $\lambda \in \Lambda$, construct an infinite-armed bandit instance (indexed by $\lambda$) with $F_\lambda$ being the (unknown) CDF of $f(\jiao{\theta^\star, \lambda (1-\kappa_2') v_{\text{pre}} + \sqrt{\kappa_2'(2-\kappa_2')} v})$ with $v\sim \mathsf{Unif}(\mathbb{S}^{d-1}\cap \{v_{\text{pre}}\}^\perp)$. Next we run $|\Lambda| = J+1$ instances of \Cref{alg:UCBLCB}, one for each $\lambda\in \Lambda$, with parameters
    \begin{align*}
        p \gets F_{d-1}\left(\frac{1}{\sqrt{d+1-i}}\right), \quad q \gets F_{d-1}\left(\frac{1+\kappa_1'}{\sqrt{d+1-i}}\right), \quad \varepsilon \gets \frac{q-p}{5}, \quad \delta \gets \frac{\delta}{d(J+1)}. 
    \end{align*}
    The $(J+1)$ instances of \Cref{alg:UCBLCB} are interleaved as follows: the While loops of \Cref{alg:UCBLCB} across all instances are synchronized. At the end of each iteration in the While loop, if any instance outputs a vector, we stop and move to the $(i+1)$-th step; otherwise, we perform an additional round of iteration to all instances. Let $v_i\in \mathbb{S}^{d-1}$ be the final output in this step. 
    \item final output $a_0$: same as \Cref{alg:burn-in}, we simply take $a_0 = \frac{1}{\sqrt{m}}\sum_{i=1}^m v_i$. 
\end{enumerate}

We comment on why this algorithm is well-defined and agnostic to $f$. First, the CDFs $(F_d, F_{d-1})$ are known to the learner, so the parameters $(p,q,\varepsilon)$ passed to \Cref{alg:UCBLCB} are well-defined. Second, to see why we are reduced to the infinite-armed bandit setting, for $i=1$ we pull the arm $v\sim \mathsf{Unif}(\mathbb{S}^{d-1})$ and observe $\calN(f(\jiao{\theta^\star, v}),1)$, and for $i\ge 2$ we sample $v\sim \mathsf{Unif}(\mathbb{S}^{d-1}\cap \{v_{\text{pre}}\}^\perp)$, pull the arm $v' = \lambda (1-\kappa_2') v_{\text{pre}} + \sqrt{\kappa_2'(2-\kappa_2')}v\in \mathbb{B}^d$, and observe $\calN(f(\jiao{\theta^\star, v'},1))$.

\subsubsection{Proof of correctness}
We condition on the good event that the outputs of all instances of \Cref{alg:UCBLCB} are all correct; this happens with probablity at least
\begin{align*}
    1 - \frac{\delta}{d} - \sum_{i=2}^m (J+1)\cdot \frac{\delta}{d(J+1)} \ge 1 - \delta. 
\end{align*}
Conditioned on this event, we will show that for $\kappa_1'>0$ small enough, 
\begin{align}\label{eq:induction}
\frac{1}{i}\sum_{1\le j\le i} \jiao{\theta^\star, v_j} \in \left[ \frac{1}{\sqrt{d}}, \frac{1+\kappa_1'}{\sqrt{d}} \right], \qquad \forall i = 1,\cdots,m. 
\end{align} 
We prove \eqref{eq:induction} by induction on $i$. For $i=1$, the correctness of \Cref{alg:UCBLCB} implies that $F(f(\jiao{\theta^\star,v_1}))\in [p,q]$, with $F(t) = F_d\circ f^{-1}(t)$ thanks to the monotonicity of $f$ and $F_d$ being the CDF of $\jiao{\theta^\star, v}$ with $v\sim \mathsf{Unif}(\mathbb{S}^{d-1})$. Consequently, $F_d(\jiao{\theta^\star, v_1})\in [p,q]$, and the definitions of $(p,q)$ yield $\jiao{\theta^\star, v_1}\in [1/\sqrt{d}, (1+\kappa_1')/\sqrt{d}]$. 

For the inductive step, assume $i\ge 2$ and \eqref{eq:induction} holds for $i-1$. For each $\lambda\in \Lambda$, the correctness of \Cref{alg:UCBLCB} now gives $F_\lambda(f(\jiao{\theta^\star,\lambda (1-\kappa_2') v_{\text{pre}} + \sqrt{\kappa_2'(2-\kappa_2')}v_i}))\in [p,q]$, where
\begin{align*}
F_\lambda(t) = F_{d-1}\left( \frac{f^{-1}(t) - \jiao{\theta^\star,\lambda (1-\kappa_2') v_{\text{pre}}} }{ \sqrt{\kappa_2'(2-\kappa_2')}\|\mathrm{Proj}_{\{v_{\text{pre}}\}^\perp}(\theta^\star)\|_2}\right), 
\end{align*}
for $\jiao{\theta^\star, v_i}/\|\mathrm{Proj}_{\{v_{\text{pre}}\}^\perp}(\theta^\star)\|_2 \sim f_{d-1}$ by Lemma \ref{lemma:nontrivialcorr}. By definitions of $(p,q)$, this gives that
\begin{align*}
\jiao{\theta^\star, v_i} \in \frac{\|\mathrm{Proj}_{\{v_{\text{pre}}\}^\perp}(\theta^\star)\|_2}{\sqrt{d+1-i}}\times \left[1, 1+\kappa_1' \right] = \sqrt{\frac{1 - \jiao{\theta^\star, v_{\text{pre}}}^2}{d+1-i}}\times \left[1, 1+\kappa_1' \right].
\end{align*}
Let $t = \jiao{\theta^\star, v_{\text{pre}}}^2$, by the induction hypothesis we have $t\in [(i-1)/d, (1+\kappa_1')^2(i-1)/d]$. Therefore,
\begin{align*}
\frac{1}{i}\sum_{1\le j\le i} \jiao{\theta^\star, v_i} &= \frac{\sqrt{(i-1)t} + \jiao{\theta^\star, v_i}}{i} \\
&\in \left[ \frac{\sqrt{(i-1)t} + \sqrt{(1-t)/(d+1-i)}}{i}, \frac{\sqrt{(i-1)t} + (1+\kappa_1')\sqrt{(1-t)/(d+1-i)}}{i} \right] \\
&\stepa{\subseteq} \left[ \frac{1}{\sqrt{d}}, \frac{1+\kappa_1'}{\sqrt{d}} \right], 
\end{align*}
completing the inductive step. It remains to prove (a). The upper bound is easily established using $t\le (1+\kappa_1')^2(i-1)/d$ for the first term, and $1-t\le 1-(i-1)/d$ for the second term. For the lower bound, one can easily check that
\begin{align*}
    t \le \frac{d}{d+1} \Longrightarrow \frac{\mathrm{d}}{\mathrm{d}t}\left[\frac{\sqrt{(i-1)t} + \sqrt{(1-t)/(d+1-i)}}{i} \right] \ge 0.
\end{align*}
As $t\le (1+\kappa_1')^2(i-1)/d \le (1+\kappa_1')^2(m-1)/d \le x_0^2(1+\kappa_1')^2$, by choosing $\kappa_1'>0$ small enough we may ensure that $t\le d/(d+1)$. This means that it suffices to check the value at $t = (i-1)/d$, and we arrive at the lower bound in (a). Therefore \eqref{eq:induction} holds by induction. 

Finally, choosing $i=m$ in \eqref{eq:induction} now gives 
\begin{align*}
\jiao{\theta^\star, a_0} = \frac{1}{\sqrt{m}}\sum_{1\le j\le m}\jiao{\theta^\star, v_j} \ge \sqrt{\frac{m}{d}} \ge x_0, 
\end{align*}
establishing the correctness. 

\subsubsection{Analysis of sample complexity}
For $i\in [m]$ and $k\in [5]$, define the quantity $\xi_{i,k}$ by
\begin{align*}
p_1 + k\varepsilon_1 = F_{d}\left(\frac{\xi_{1,k}}{\sqrt{d}}\right), \qquad p_i + k\varepsilon_i = F_{d-1}\left(\frac{\xi_{i,k}}{\sqrt{d}}\right), \quad 2\le i\le m.
\end{align*}
It is clear that $1 = \xi_{i,0} \le \xi_{i,1} \le \cdots \le \xi_{i,5} = 1+\kappa_1'$. In addition, \Cref{lemma:nontrivialcorr} tells that $\min_{i\in [m], k\in [5]}(\xi_{i,k} - \xi_{i,k-1}) \ge c_1'$ for some numerical constant $c_1'>0$ depending only on $(x_0, \kappa_1')$. 

The sample complexity for the first step $i=1$ is straightforward. Since $F^{-1}\circ F_d = f$, \Cref{lemma:infinite-armed-bandit} tells that the sample complexity for the first step is
\begin{align*}
    &O\left( \frac{\log^2(d/\delta)}{(f(\xi_{1,2}/\sqrt{d})-f(\xi_{1,1}/\sqrt{d}))^2} + \frac{\log^2(d/\delta)}{(f(\xi_{1,4}/\sqrt{d})-f(\xi_{1,3}/\sqrt{d}))^2} \right) \\
    &=  O\left( \frac{\log^2(d/\delta)}{\min_{z\in [1/\sqrt{d}, (1+\kappa_1')/\sqrt{d}]} |f(z+c_1'/\sqrt{d})-f(z)|^2 } \right). 
\end{align*}

For $i\ge 2$, by \Cref{lemma:infinite-armed-bandit} and the way of interleaving, the sample complexity for the $i$-th step is
\begin{align*}
O\left((J+1)\cdot \min_{\lambda\in\Lambda}\left(\frac{\log^2((J+1)d/\delta)}{(b_{2,\lambda} - b_{1,\lambda})^2} + \frac{\log^2((J+1)d/\delta)}{(b_{4,\lambda} - b_{3,\lambda})^2} \right) \right), 
\end{align*}
where by the definition of $F_\lambda$ we have
\begin{align*}
    b_{k,\lambda} := F_\lambda^{-1}\left(F_{d-1}\left(\frac{\xi_{i,k}}{\sqrt{d}}\right)\right) = \sqrt{\kappa_2'(2-\kappa_2')}\|\mathrm{Proj}_{\{v_{\text{pre}}\}^\perp} \|_2 \frac{\xi_{i,k}}{\sqrt{d}} + \lambda (1-\kappa_2') \jiao{\theta^\star, v_{\text{pre}}}. 
\end{align*}
By \eqref{eq:induction}, for $\kappa_1'>0$ small enough it holds that 
\begin{align*}
\sqrt{\kappa_2'(2-\kappa_2')}\|\mathrm{Proj}_{\{v_{\text{pre}}\}^\perp} \|_2 &\ge \sqrt{\kappa_2'(2-\kappa_2')} \sqrt{1-\frac{m(1+\kappa_1')^2}{d}} = \Omega_{x_0, \kappa_1', \kappa_2'}(1), \\
\jiao{\theta^\star, v_{\text{pre}}} &\in \left[\sqrt{\frac{i-1}{d}}, (1+\kappa_1')\sqrt{\frac{i-1}{d}}\right]. 
\end{align*}
By choosing $y_0 = (1-\kappa_2')\sqrt{(i-1)/d}$ and thinking of $z\in \{b_{1,\lambda}, b_{3,\lambda}\}$, the above sample complexity is upper bounded by
\begin{align*}
O\left( \frac{\log^3(d/\delta)}{\max_{\lambda\in \Lambda}\min_{z\in [\lambda y_0, (1+\kappa_1')\lambda y_0 + \kappa_1'/\sqrt{d}]} |f(z+c_1/\sqrt{d}) - f(z)|^2} \right), 
\end{align*}
where $c_1 = c_1(x_0, c_1', \kappa_1', \kappa_2') > 0$ is a numerical constant. Next we show that for small $\kappa_1', \kappa_2'>0$, the maximization over a discrete grid $\lambda\in \Lambda$ could be replaced by the maximization over a continuous interval $y\in [c_2/\sqrt{d}, y_0]$, for any given numerical constant $c_2>0$. To see this, for any $y\in [c_2/\sqrt{d}, y_0]$, choose $j\in [J]$ such that $e^{-\kappa_2'j} y_0 < y \le e^{-\kappa_2'(j-1)} y_0$. Such $j$ always exists, for 
\begin{align*}
e^{-\kappa_2' J} y_0 \le e^{-\kappa_2' J} < \frac{\kappa_2'}{\sqrt{d}}< \frac{c_2}{\sqrt{d}} \le y
\end{align*}
as long as $\kappa_2' < c_2$. Now for $z\in [e^{-\kappa_2'j} y_0, (1+\kappa_1')e^{-\kappa_2'j} y_0 + \kappa_1'/\sqrt{d}]$, we have
\begin{align*}
e^{-\kappa_2'j} y_0 &\ge e^{-\kappa_2'}y \ge (1-\kappa_2')y, \\
(1+\kappa_1')e^{-\kappa_2'j} y_0 + \frac{\kappa_1'}{\sqrt{d}} &\le (1+\kappa_1')y + \frac{\kappa_1'}{\sqrt{d}} \le \left(1+\kappa_1' + \frac{\kappa_1'}{c_2}\right)y, 
\end{align*}
and by choosing $\kappa_1', \kappa_2'>0$ small enough it holds that $z\in [(1-\kappa_1)y, (1+\kappa_1)y]$, for any prescribed constant $\kappa_1>0$. In other words, we have shown that for any function $\ell(\cdot)$, 
\begin{align*}
\max_{\lambda\in \Lambda}\min_{z\in [\lambda y_0, (1+\kappa_1')\lambda y_0 + \kappa_1'/\sqrt{d}]} \ell(z) \ge \max_{c_2/\sqrt{d}\le y\le y_0}\min_{z\in [(1-\kappa_1)y, (1+\kappa_1)y]} \ell(z)
\end{align*}
for small enough constants $\kappa_1' = \kappa_1'(x_0, \kappa_1, c_2)>0, \kappa_2' = \kappa_2'(x_0, \kappa_1, c_2)>0$. Therefore, the sample complexity for the $i$-th step is finally upper bounded by
\begin{align*}
O\left( \frac{\log^3(d/\delta)}{\max_{c_2/\sqrt{d}\le y\le (1-\kappa_2')\sqrt{(i-1)/d}}\min_{z\in [(1-\kappa_1) y_0, (1+\kappa_1)y_0]} |f(z+c_1/\sqrt{d}) - f(z)|^2} \right). 
\end{align*}

Finally, by summing over $i=1,\cdots,m$ and comparing the above expressions with $\varepsilon_i$ in \Cref{thm:ub_burnincost_formal}, the sample complexity upper bound in \Cref{thm:burnin_agnostic} follows. 

\subsection{The case where $f$ is even}\label{subsec:f_even}
We remark that \Cref{inftheorem:integral_ub,inftheorem:learning_ub} still hold if $f$ is even and monotone on $[0,1]$, with the following minor changes in the statement and the algorithm: 
\begin{itemize}
    \item In Theorem \ref{thm:ub_burnincost_formal}, the claim $\jiao{\theta^\star, a_0}\ge x_0$ is changed into $|\jiao{\theta^\star, a_0}|\ge x_0$. Similarly, in \Cref{inftheorem:learning_ub}, the quantity $\jiao{\widehat{\theta}_T, \theta}$ is replaced by $|\jiao{\widehat{\theta}_T, \theta}|$;  
    \item All appearances of $\jiao{\theta^\star,v}$ are replaced by $|\jiao{\theta^\star,v}|$ in the claim of \Cref{lemma:iaht}, and by $\jiao{\theta^\star,v}\text{sign}(\jiao{\theta^\star, v_{\text{pre}}})$ in the claim of \Cref{lemma:gaht}; 
    \item Line 9 in \Cref{alg:burn-in} is changed slightly, as detailed below. 
\end{itemize}

When $f$ is even, in Line 9 of \Cref{alg:burn-in}, we are free to assign $v_i\gets \pm v$ but hope that $\jiao{\theta^\star, v}$ and $\jiao{\theta^\star, v_1}$ have the same sign. This is done by querying the actions $(v+v_1)/\sqrt{2}$ and $(v-v_1)/\sqrt{2}$ for $O(\log(1/\delta)/\varepsilon^2)$ times to obtain sample averages $\overline{r}_+$ and $\overline{r}_-$. Then we assign $v_i\gets v$ if $\overline{r}_+\ge \overline{r}_-$, and $v_i\gets -v$ otherwise -- this leads to a tiny increase in the sample complexity and failure probability. To see why it works, note that if the high probability event in \Cref{lemma:iaht} happens when the while loop breaks, we have $|\jiao{\theta^\star, v}|, |\jiao{\theta^\star, v_1}|\in [1/\sqrt{d}, (1+\kappa_1)/\sqrt{d}]$. Hence if $\jiao{\theta^\star, v}$ and $\jiao{\theta^\star, v_1}$ have the same sign, we have
\begin{align*}
\left| \jiao{\theta^\star, \frac{v+v_1}{\sqrt{2}}} \right| \ge \sqrt{\frac{2}{d}}, \quad \left| \jiao{\theta^\star, \frac{v-v_1}{\sqrt{2}}} \right| \le \frac{\kappa_1}{\sqrt{2d}}. 
\end{align*}
Similar to the proof of \Cref{lemma:iaht}, by choosing a small $\kappa_1>0$, this difference could be detected with high probability. 

The correctness of the modified statements is mostly straightforward and omitted, and we only include a short note on the correctness of $|\jiao{\theta^\star, a_0}|$ in \Cref{thm:ub_burnincost_formal}. By the above modification of the algorithm, we have $\jiao{\theta^\star, v_i}\text{sign}(\jiao{\theta^\star, v_1}) \ge 1/\sqrt{d}$ for all $1\le i\le d_0$. For $i>d_0$, by induction, the modified \Cref{lemma:gaht} ensures that $\jiao{\theta^\star, v_i}\text{sign}(\jiao{\theta^\star, v_1}) \ge 1/\sqrt{d}$ as well. Consequently, 
\begin{align*}
    |\jiao{\theta^\star, a_0}|\ge \jiao{\theta^\star, a_0}\text{sign}(\jiao{\theta^\star, v_1}) = \frac{1}{\sqrt{m}}\sum_{i=1}^m \jiao{\theta^\star, v_i}\text{sign}(\jiao{\theta^\star, v_1}) \ge \sqrt{\frac{m}{d}} \ge x_0. 
\end{align*}

%% file: appendix_section4.tex
\section{Deferred proofs in Section \ref{sec:discussions}}\label{sec:proof_discussion}

\subsection{Proof of \Cref{thm:nonadaptive}}
Let $(a_1,\cdots,a_T)$ be the actions taken by the nonadaptive algorithm, and $P_\theta^T$ be the distribution of the observations $(r_1,\cdots,r_T)$ when the true parameter $\theta^\star$ is $\theta$. By the generalized Fano's inequality in \Cref{lemma:Fano} with $\Delta = 3/4$, the Bayes risk is lower bounded by
\begin{align*}
\bE[1 - \jiao{\widehat{\theta}_T, \theta^\star}] \ge \frac{3}{4}\left(1 - \frac{I(\theta^\star; \calH_T) + \log 2}{\log(1/\max_{a\in \mathbb{S}^{d-1}}\bP(\jiao{\theta^\star, a}\le 1/4))}\right) = \frac{3}{4}\left(1 - \frac{I(\theta^\star; \calH_T) + \log 2}{\Omega(d)}\right), 
\end{align*}
where the final inequality is due to \Cref{lemma:nolargeinnerproduct}. To upper bound the mutual information, we use its variational representation $I(X;Y) = \min_{Q_Y}\bE_{P_X}[D_{\text{KL}}(P_{Y\mid X} \| Q_Y)]$ to get
\begin{align*}
I(\theta^\star; \calH_T) \le \bE[D_{\text{KL}}(P_{\theta^\star}^T \| P_0^T)] = \frac{1}{2}\sum_{t=1}^T \bE[f(\jiao{\theta^\star, a_t})^2] \lesssim T \cdot \min_{K\ge 1}\left( g\left(\sqrt{\frac{\log K}{d}}\right)^2 + \frac{1}{K} \right), 
\end{align*}
where the final step crucially uses the nonadaptive nature that $a_t$ is independent of $\theta^\star$, as well as the concentration result in \Cref{lemma:nolargeinnerproduct}. A combination of the above inequalities completes the proof. 

\subsection{Proof of \Cref{thm:finite_action}}\label{subsec:proof_finite_action}
By \Cref{lemma:nolargeinnerproduct} and the union bound, there exists a subset $\calA\subseteq \mathbb{S}^{d-1}$ such that for every pair of distinct $a, a'\in \calA$, it holds that $|\jiao{a,a'}| \le \sqrt{(c'\log K)/d}$. Pick this action set, and for each pair $(\theta_i, \theta_j)$ in $\calA$, consider the hypothesis testing problem of $\theta^\star = \theta_i$ versus $\theta^\star = \theta_j$. As
\begin{align*}
2 - \jiao{\theta_i, \theta} - \jiao{\theta_j, \theta} \ge 2 - \|\theta_i + \theta_j\|_2 \ge     2 - \sqrt{2(1+\sqrt{(c'\log K)/d})} > \frac{1}{2}
\end{align*}
for all $\theta\in\mathbb{S}^{d-1}$ and $d$ large enough, Le Cam's two-point lower bound (cf. \Cref{lemma:twopoint}) shows that
\begin{align*}
\inf_{\widehat{\theta}_T}\sup_{\theta^\star \in \mathbb{S}^{d-1}} \bE_{\theta^\star}[1 - \jiao{\theta^\star, \widehat{\theta}_T}] &\ge \frac{1}{4}(1-\|P_{\theta_i}^T - P_{\theta_j}^T\|_{\text{TV}}) \\
&\ge \frac{1}{4}(1-\|P_{\theta_i}^T - P_0^T\|_{\text{TV}}-\|P_{\theta_j}^T - P_0^T\|_{\text{TV}}), 
\end{align*}
where again $P_\theta^T$ denotes the distribution of all observations up to time $T$ when $\theta^\star = \theta$. By Pinsker's inequality, 
\begin{align*}
    \|P_{\theta}^T - P_0^T\|_{\text{TV}} \le \sqrt{\frac{1}{2}D_{\text{KL}}(P_0^T \| P_\theta^T) } = \sqrt{\frac{1}{4}\sum_{t=1}^T \bE_{P_0^T}[f(\jiao{a_t, \theta})^2]}. 
\end{align*}
Consequently, by averaging over all distinct pairs $\theta_i, \theta_j \in \calA$, we conclude that
\begin{align*}
    \inf_{\widehat{\theta}_T}\sup_{\theta^\star \in \mathbb{S}^{d-1}} \bE_{\theta^\star}[1 - \jiao{\theta^\star, \widehat{\theta}_T}] &\ge \frac{1}{4}\left(1 - 2\bE_{\theta\sim \mathsf{Unif}(\calA)}\|P_\theta^T - P_0^T\|_{\text{TV}}\right) \\
    &\ge \frac{1}{4}\left(1 - \sqrt{\sum_{t=1}^T \bE_{\theta\sim \mathsf{Unif}(\calA)}\bE_{P_0^T}[f(\jiao{a_t,\theta})^2] }\right). 
\end{align*}
Since $a_t$ is restricted to lie in $\calA$, we have
\begin{align*}
    \bE_{\theta\sim \mathsf{Unif}(\calA)}[f(\jiao{a_t,\theta})^2] \le \frac{1}{K} + g\left(\sqrt{\frac{c'\log K}{d}}\right)^2, 
\end{align*}
and plugging it into the above display completes the proof of the theorem. 

\subsection{Proof of \Cref{thm:unit_ball}}
We describe an algorithm that finds an estimate $\widehat{r}$ of $r$ such that $r\in [\widehat{r}/4, \widehat{r}]$ with high probability. To this end, we randomly sample $m$ actions $a_1,\cdots,a_m\sim \mathsf{Unif}(\mathbb{S}^{d-1})$, and play each action $n$ times to obtain the empirical mean rewards $\overline{r}_1,\cdots,\overline{r}_n$ for these actions. Let $F$ be the CDF of the random variable $\jiao{a_i, \theta^\star}$. By rotational invariance, the constants $c_1 := F(2r/\sqrt{d}), c_2 := F(3r/\sqrt{d})$ are known to the learner, and \Cref{lemma:nontrivialcorr} implies that $c_2 - c_1 = \Omega(1)$. Consequently, as long as
\begin{align*}
n \gtrsim \frac{\log(1/\delta)}{(f(4r/\sqrt{d}) - f(3r/\sqrt{d}))^2} + \frac{\log(1/\delta)}{(f(2r/\sqrt{d}) - f(r/\sqrt{d}))^2}
\end{align*}
and $\delta>0$ is small enough, with high probability 
\begin{align*}
    \bP(\overline{r}_i \le f(4r/\sqrt{d})) &\ge \bP(f(\jiao{\theta^\star, a_i}) \le f(3r/\sqrt{d})) = c_2, \\
    \bP(\overline{r}_i \le f(r/\sqrt{d})) &\le \bP(f(\jiao{\theta^\star, a_i}) \le f(2r/\sqrt{d})) = c_1. 
\end{align*}
Now replacing the LHS by the empirical CDFs using $m$ samples, as $c_2 - c_1 = \Omega(1)$, we conclude that $m=\widetilde{O}(1)$ samples are sufficient to tell the above CDF difference. Consequently, this procedure only requires an additional sample size $mn$, which is no larger than the number of samples used in \Cref{alg:IAHT} for the link function $x\mapsto f(rx)$.  

Finally, as the choice of $n$ depends on the unknown quantity $r$, we perform a grid search over $r\in [0,1]$, start from the ``easiest'' candidate (i.e. with the smallest $n$), move to the next candidate and incrementally increase $n$ until the test reports success for some candidate $\widehat{r}$. This procedure only amplifies the failure probability by at most the grid size, which could be made logarithmic.

\subsection{Proof of Theorem \ref{thm:burnin_linear_bandit}}
The proof is based on a similar application of the $\chi^2$-informativity method in Section \ref{sec:lower_bound}, with a few changes in defining and handling the good events. For a fixed policy, define the good event
\begin{align*}
    E_t = \bigcap_{s\le t} \left\{ |\jiao{\theta^\star - \bE[\theta^\star\mid \calH_{s-1}], a_s} | \le \sqrt{\frac{c\log(1/\delta)}{d}}\right\},  
\end{align*}
where the constants $c>0, \delta\in (0,1)$ are chosen later. 

We first show that the good event happens with high probability under the uniform prior $\theta^\star \sim \pi := \mathsf{Unif}(\mathbb{S}^{d-1})$. By the Bayes rule, the posterior distribution of $\theta^\star$ given the history $\calH_{t-1}$ is
\begin{align*}
\pi(\theta^\star \mid \calH_{t-1}) \propto \pi(\theta^\star)\prod_{s=1}^{t-1} \varphi(r_s - \jiao{\theta^\star, a_s}) =: \pi(\theta^\star)\exp(-U(\theta^\star)). 
\end{align*}
Here $\varphi(x) := \exp(-x^2/2)/\sqrt{2\pi}$ is the normal pdf. Note that for linear bandits, $U$ is convex in $\theta^\star$. Since the Ricci curvature tensor of $\mathbb{S}^{d-1}$ is $(d-2)I$, the generalized Brascamp--Lieb inequality on manifolds (cf. \cite[Theorem 2.2]{kolesnikov2016riemannian}) implies that the posterior density $\pi(\theta^\star\mid \calH_{t-1})$ satisfies a subGaussian concentration with a variance parameter $O(1/d)$ for every $\calH_{t-1}$. Finally, since $\theta^\star$ and $a_t$ are conditionally independent given $\calH_{t-1}$, we have
\begin{align*}
\bP\left(|\jiao{\theta^\star - \bE[\theta^\star\mid \calH_{t-1}], a_t} | \le \sqrt{\frac{c\log(1/\delta)}{d}} ~ \bigg| ~ \calH_{t-1} \right) \ge 1 - \delta
\end{align*}
for every $\calH_{t-1}$ and a proper absolute constant $c>0$. Consequently, a union bound gives
\begin{align}\label{eq:high_prob}
    \bP(E_T) \ge 1 - T\delta. 
\end{align}

Next we upper bound the conditional $\chi^2$-informativity $I(\theta^\star; \calH_t\mid E_t)$, in a recursive way similar to \Cref{lemma:recursion_chi^2}. Specifically, since $E_t\in \sigma(\theta^\star, \calH_t)$, it holds that
\begin{align*}
    I_{\chi^2}(\theta^\star; \calH_t\mid E_t) + 1 &= \inf_{Q_{\calH_t}} \int \frac{\bP(\theta^\star, \calH_t\mid E_t)^2}{\pi(\theta^\star)Q_{\calH_t}(\calH_t)}\text{d}\theta^\star\text{d}a^t\text{d}r^t \\
    &\stepa{\le} \inf_{Q_{\calH_{t-1}}} \int \frac{\bP(\theta^\star, \calH_t\mid E_t)^2}{\pi(\theta^\star)Q_{\calH_{t-1}}(\calH_{t-1})\cdot \bP_t(a_t\mid \calH_{t-1})\varphi(r_t - \jiao{\bE[\theta^\star\mid \calH_{t-1}], a_t})}\text{d}\theta^\star\text{d}a^t\text{d}r^t \\
    &= \inf_{Q_{\calH_{t-1}}} \int \frac{\left[ \frac{\1(E_t)}{\bP(E_t)}\cdot \pi(\theta^\star)\cdot \prod_{s=1}^{t-1}\left(\bP_s(a_s\mid \calH_{s-1})\cdot \varphi(r_s-\jiao{\theta^\star,a_s}\right) \right]^2}{\pi(\theta^\star)Q_{\calH_{t-1}}(\calH_{t-1})}\\
    &\qquad \times \bP_t(a_t\mid \calH_{t-1})\times \frac{\varphi(r_t-\jiao{\theta^\star,a_t})^2}{\varphi(r_t- \jiao{\bE[\theta^\star\mid \calH_{t-1}], a_t})}\text{d}\theta^\star\text{d}a^t\text{d}r^t \\
    &\stepb{=} \inf_{Q_{\calH_{t-1}}} \int \frac{\left[ \frac{\1(E_t)}{\bP(E_t)}\cdot \pi(\theta^\star)\cdot \prod_{s=1}^{t-1}\left(\bP_s(a_s\mid \calH_{s-1})\cdot \varphi(r_s-\jiao{\theta^\star,a_s}\right) \right]^2}{\pi(\theta^\star)Q_{\calH_{t-1}}(\calH_{t-1})}\\
    &\qquad \times \bP_t(a_t\mid \calH_{t-1})\times \exp(\jiao{\theta^\star - \bE[\theta^\star\mid \calH_{t-1}],a_t}^2)\text{d}\theta^\star\text{d}a^t\text{d}r^{t-1} \\
    &\stepc{\le} \inf_{Q_{\calH_{t-1}}} \int \frac{\left[ \frac{\1(E_t)}{\bP(E_t)}\cdot \pi(\theta^\star)\cdot \prod_{s=1}^{t-1}\left(\bP_s(a_s\mid \calH_{s-1})\cdot \varphi(r_s-\jiao{\theta^\star,a_s}\right) \right]^2}{\pi(\theta^\star)Q_{\calH_{t-1}}(\calH_{t-1})}\\
    &\qquad \times \bP_t(a_t\mid \calH_{t-1})\times \exp\left(\frac{c\log(1/\delta)}{d}\right)\text{d}\theta^\star\text{d}a^t\text{d}r^{t-1} \\
    &\stepd{\le} \inf_{Q_{\calH_{t-1}}} \int \frac{\left[ \frac{\1(E_{t-1})}{\bP(E_{t-1})}\cdot \pi(\theta^\star)\cdot \prod_{s=1}^{t-1}\left(\bP_s(a_s\mid \calH_{s-1})\cdot \varphi(r_s-\jiao{\theta^\star,a_s}\right) \right]^2}{\pi(\theta^\star)Q_{\calH_{t-1}}(\calH_{t-1})}\\
    &\qquad \times \bP_t(a_t\mid \calH_{t-1})\times \frac{\exp(c\log(1/\delta)/d)}{\bP(E_t\mid E_{t-1})^2}\text{d}\theta^\star\text{d}a^t\text{d}r^{t-1} \\
    &\stepe{=} \inf_{Q_{\calH_{t-1}}} \int \frac{\left[ \frac{\1(E_{t-1})}{\bP(E_{t-1})}\cdot \pi(\theta^\star)\cdot \prod_{s=1}^{t-1}\left(\bP_s(a_s\mid \calH_{s-1})\cdot \varphi(r_s-\jiao{\theta^\star,a_s}\right) \right]^2}{\pi(\theta^\star)Q_{\calH_{t-1}}(\calH_{t-1})}\\
    &\qquad \times \frac{\exp(c\log(1/\delta)/d)}{\bP(E_t\mid E_{t-1})^2}\text{d}\theta^\star\text{d}a^{t-1}\text{d}r^{t-1} \\
    &\stepf{=} \frac{\exp(c\log(1/\delta)/d)}{\bP(E_t\mid E_{t-1})^2}(I_{\chi^2}(\theta^\star; \calH_{t-1}\mid E_{t-1})+1),
\end{align*}
where (a) defines a valid distribution $Q_{\calH_t}$ over $\calH_t$ as $\bE[\theta^\star\mid \calH_{t-1}]\in \sigma(\calH_t)$, (b) integrates out $r_t$ (note that $E_t$ does not depend on $r_t$), (c) uses the definition of $E_t$, (d) follows from $\1(E_t)\le \1(E_{t-1})$, (e) integrates out $a_t$ as $E_{t-1}$ no longer depends on $a_t$, and (f) uses the definition of $I_{\chi^2}(\theta^\star; \calH_{t-1}\mid E_{t-1})$. Continuing this process from $t=T$ to $t=1$ leads to
\begin{align}\label{eq:chi^2_upper_bound}
    I_{\chi^2}(\theta^\star; \calH_T \mid E_T) + 1 \le \frac{\exp(cT\log(1/\delta)/d)}{\bP(E_T)^2}. 
\end{align}

Finally we apply \Cref{lemma:Fano_chi^2} to prove the claimed lower bound in \Cref{thm:burnin_linear_bandit}. For any estimator $\widehat{\theta}_T\in \sigma(\calH_T)\cap \mathbb{S}^{d-1}$, by \eqref{eq:high_prob} and \eqref{eq:chi^2_upper_bound} it holds that
\begin{align*}
\bP\left(\jiao{\widehat{\theta}_T, \theta^\star} \le \frac{1}{2}\right) &\ge \bP(E_T)\cdot \bP\left(\jiao{\widehat{\theta}_T, \theta^\star} \le \frac{1}{2} ~ \bigg| ~ E_T \right) \\
&\ge \bP(E_T)\left(1 - c_1e^{-c_0d/4}\cdot \frac{\exp(cT\log(1/\delta)/(2d))}{\bP(E_T)}\right) \\
&\ge 1 - T\delta - c_1\exp\left(-\frac{c_0d}{4} + \frac{cT\log(1/\delta)}{2d}\right). 
\end{align*}
Choosing $\delta = 1/(4T)$ and $T = c'd^2/\log d$ for a small absolute constant $c'>0$, the above probability is at least $1/2$. This completes the proof of the claim $T_{\text{burn-in}}^\star(\mathsf{id},d)\gtrsim d^2$ in \Cref{thm:burnin_linear_bandit}. 

%% file: suboptimality.tex
\section{Suboptimality of existing algorithms}
\input{lower_bound_EluderUCB}

\input{lower_bound_online_regression}

%% file: lower_bound_EluderUCB.tex
\subsection{Suboptimality of Eluder UCB}
In this section we prove the lower bound in \Cref{thm:ucb-lower-bound} for the Eluder-UCB algorithm in \eqref{eq:eluderucb}. First, by \Cref{lemma:least_squares} and the Lipschitzness of $f$, similar arguments to \Cref{subsec:upper_bound_linear} yields that $\theta^\star \in \mathbb{C}_t$ with high probability for the confidence set $\mathbb{C}_t$ in \eqref{eq:eluder_confidence_set}, with $\textbf{Est}_t \asymp d$. Since we do not care about the constants, let us redefine $\mathbb{C}_t$ in \eqref{eq:eluder_confidence_set} with $\textbf{Est}_t$ replaced by $4\textbf{Est}_t$. Under this new definition, by triangle inequality we have
\begin{align}\label{eq:eluder_confidence_set_larger}
\mathbb{C}_t' := \left\{ \theta\in \mathbb{S}^{d-1}: \sum_{s<t} \left(f(\jiao{a_s,\theta}) - f(\jiao{a_s,\theta^\star}) \right)^2 \le \textbf{Est}_t \right\} \subseteq \mathbb{C}_t 
\end{align}
with high probability. 

To prove the lower bound, next we construct $\theta^\star$ and a valid action sequence $(a_1,\cdots,a_T)$ returned by the Eluder-UCB algorithm. Suppose that
\begin{align*}
    T \le \min\left\{ T_0, \frac{d}{g(\sqrt{(c\log T_0)/d})^2} \right\}
\end{align*}
for some $T_0\in \mathbb{N}$ and $c>0$, with $g(x) := \max\{|f(x)|,|f(-x)|\}$. By \Cref{lemma:nolargeinnerproduct} and the union bound, there exists $T_0+1$ points $\theta_0, \theta_1, \cdots, \theta_{T_0}$ such that $|\jiao{\theta_i, \theta_j}| \le \sqrt{(c\log T_0)/d}$ for all $i\neq j$. Now we claim that if $\theta^\star = \theta_0$, then with high probability $a_t = \theta_t$ for all $t\in [T]$ will be valid outputs of the Eluder-UCB algorithm. Note that this result implies the claimed lower bound in \Cref{thm:ucb-lower-bound} as $\jiao{\theta^\star, a_t} = O(\sqrt{(\log T_0)/d})$ for all $t\in [T]$. 

To prove the claim, first note that thanks to \Cref{assump:main}, any action $a_t \in \mathbb{C}_t$ is a valid output of \eqref{eq:eluderucb} at time $t$. Then thanks to \eqref{eq:eluder_confidence_set_larger}, with high probability every $a_t \in \mathbb{C}_t'$ is valid as well. Finally, for every $t\in [T]$, using the inner product upper bound, we have
\begin{align*}
\sum_{s<t} \left(f(\jiao{a_s,a_t}) - f(\jiao{a_s,\theta^\star}) \right)^2 &= \sum_{s<t} \left(f(\jiao{\theta_s,\theta_t}) - f(\jiao{\theta_s,\theta_0}) \right)^2 < T\cdot 4g(\sqrt{(c\log T_0)/d})^2 \le 4d. 
\end{align*}
This shows that $a_t\in \mathbb{C}_t'$, and completes the proof of the claim.

%% file: lower_bound_online_regression.tex
\subsection{Suboptimality of regression oracles}
In this section we prove the lower bound in Theorem \ref{thm:RO-lower-bound} for regression oracle based algorithms. First we show that for ridge bandits with Lipschitz $f$, we may choose both quantities $\textbf{Est}_t^{\textbf{On}}$ and $\textbf{Est}_t^{\textbf{Off}}$ in \eqref{eq:online_oracle} and \eqref{eq:offline_oracle} to be $\widetilde{\Theta}(d)$. In fact, $\textbf{Est}_t^{\textbf{On}} \lesssim d$ follows from \cite[Theorem 1]{rakhlin2014online}, where the sequential entropy condition \cite[Eqn. (6)]{rakhlin2014online} is ensured by the Lipschitzness of $f$. This is achieved by an improper online regression oracle; see the requirement on $\widehat{\calY}$ in \cite[Section 2.1]{rakhlin2015online}. The upper bound $\textbf{Est}_t^{\textbf{Off}} \lesssim d$ is achieved by the least squares estimator and ensured by \Cref{lemma:least_squares}, thus the offline regression oracle could be taken to be proper. 

\subsubsection{Improper online regression oracle}
We construct an improper online regression oracle that satisfies \eqref{eq:online_oracle} but provides no information for the learner. This oracle is simply chosen to be $\widehat{\theta}_t \equiv 0$, regardless of the learner's historic actions. Then the learner has no information about $\theta^\star$ and thus the actions $a_1,\cdots,a_T$ are independent of $\theta^\star$. Therefore, if $\theta^\star\sim \mathsf{Unif}(\mathbb{S}^{d-1})$, with high probability it holds that $|\jiao{\theta^\star, a_t}| = O(\sqrt{(\log T)/d})$ for every $t\in [T]$ by \Cref{lemma:nolargeinnerproduct}. Consequently, with high probability, 
\begin{align*}
    \sum_{t=1}^T \left(f(\jiao{\theta^\star,a_t}) - f(\jiao{\widehat{\theta}_t,a_t}) \right)^2 \le T\cdot g\left(O\left(\sqrt{\frac{\log T}{d}}\right)\right)^2
\end{align*}
satisfies \eqref{eq:online_oracle} due to the assumption of $T$ in \Cref{thm:RO-lower-bound}. 

\subsubsection{Proper offline regression oracle}
Similar to the online case, we can also construct a proper offline regression oracle that satisfies \eqref{eq:offline_oracle} but provides no information for the learner. We take $\theta^\star\sim \mathsf{Unif}(\mathbb{S}^{d-1})$, and the offline oracle samples independent $\widehat{\theta}_1, \cdots, \widehat{\theta}_T \sim \mathsf{Unif}(\mathbb{S}^{d-1})$. Since the learner's actions are independent of $\theta^\star$, again by \Cref{lemma:nolargeinnerproduct}, $|\jiao{\theta^\star, a_t}| = O(\sqrt{(\log T)/d})$ for every $t\in [T]$ with high probability. As for \eqref{eq:offline_oracle}, note that $\widehat{\theta}_t$ is independent of $\{a_s\}_{s\le t}$, and \Cref{lemma:nolargeinnerproduct} implies that $|\jiao{\widehat{\theta}_t, a_s}|=O(\sqrt{(\log T)/d})$ for all $s\le t$ with high probability as well. Consequently, with high probability, 
\begin{align*}
    \sum_{s=1}^t \left(f(\jiao{\theta^\star,a_s}) - f(\jiao{\widehat{\theta}_t,a_s}) \right)^2 \le T\cdot 4g\left(O\left(\sqrt{\frac{\log T}{d}}\right)\right)^2
\end{align*}
satisfies \eqref{eq:offline_oracle} due to the assumption of $T$ in \Cref{thm:RO-lower-bound}.